\PassOptionsToPackage{table}{xcolor}

\documentclass{article}

\usepackage[nonatbib, preprint]{neurips_2024}

\usepackage[
  backend=biber,       
  style=alphabetic,    
]{biblatex}
\usepackage[utf8]{inputenc} 
\usepackage[T1]{fontenc}    
\usepackage{hyperref}       
\usepackage{url}            
\usepackage{booktabs}       
\usepackage{pgfplots}      
\usepackage{multirow}       
\usepackage{amsfonts}       
\usepackage{nicefrac}       
\usepackage{microtype}      
\usepackage{amsmath}
\usepackage{amsthm}
\usepackage{amssymb}
\usepackage{comment}
\usepackage{subcaption}
\usepackage{graphicx}
\usepackage{booktabs}
\usepackage{siunitx}
\usepackage{venndiagram}
\usepackage{tikz}
\usepackage[export]{adjustbox}
\usepackage{comment}
\usepackage{algorithmic}
\usepackage{algorithm}
\usepackage{setspace}
\usepackage{enumitem}
\usepackage{tabularx}
\usepgfplotslibrary{groupplots}

\usepackage{siunitx}      
\newcolumntype{M}{>{$}l<{$}}

\newcommand{\commentPablo}[1]{\textbf{\textcolor{black}}}

\newcommand{\abs}[1]{\left\lvert#1\right\rvert}

\newcommand{\normdist}[2]{\mathcal{N}(#1, #2)}

\newcommand{\triplenorm}[1]{%
  \left\lVert\!\!\left\lVert\!\!\left\lVert #1 
  \right\rVert\!\!\right\rVert\!\!\right\rVert
}

\newtheorem{definition}{Definition}[section]
\newtheorem{theorem}{Theorem}[section]
\newtheorem{lemma}{Lemma}[section]

\newtheorem{remark}{Remark}[section]

\title{Explicit Density Approximation for Neural Implicit Samplers Using a Bernstein-Based Convex Divergence}

\author{%
  José Manuel de Frutos\thanks{Corresponding author: jofrutos@ing.uc3m.es} \\
  Universidad Carlos III\\
  \And
  Manuel A. Vázquez\\
  Universidad Carlos III\\
  \AND
  Pablo M. Olmos\\
  Universidad Carlos III\\
  \And
  Joaquín Míguez\\
  Universidad Carlos III\\
}

\makeatletter
  \let\saved@addcontentsline\addcontentsline

  \newcommand\suppressTOCentries{%
    \renewcommand{\addcontentsline}[3]{}%
  }

  \newcommand\restoreTOCentries{%
    \let\addcontentsline\saved@addcontentsline%
  }
\makeatother

\begin{document}

\maketitle

\suppressTOCentries
\begin{abstract}

Rank-based statistical metrics, such as the invariant statistical loss (ISL), have recently emerged as robust and practically effective tools for training implicit generative models. In this work, we introduce dual-ISL, a novel likelihood-free objective for training implicit generative models that interchanges the roles of the target and model distributions in the ISL framework, yielding a convex optimization problem in the space of model densities. We prove that the resulting rank-based discrepancy $d_K$ is i) continuous under weak convergence and with respect to the $L^1$ norm, and ii) convex in its first argument—properties not shared by classical divergences such as KL or Wasserstein distances. Building on this, we develop a theoretical framework that interprets $d_K$ as an $L^2$-projection of the density ratio $q = p/\tilde p$ onto a Bernstein polynomial basis, from which we derive exact bounds on the truncation error, precise convergence rates, and a closed-form expression for the truncated density approximation. We further extend our analysis to the multivariate setting via random one-dimensional projections, defining a sliced dual-ISL divergence that retains both convexity and continuity. We empirically show that these theoretical advantages translate into practical ones. Specifically, across several benchmarks dual-ISL converges more rapidly, delivers markedly smoother and more stable training, and more effectively prevents mode collapse than classical ISL and other leading implicit generative methods—while also providing an explicit density approximation. 

\end{abstract}

\section{Introduction}

Implicit generative models are a class of models that learn to generate data samples without explicitly modeling the underlying probability distribution \cite{mohamed2016learning}, enabling flexible modeling of high-dimensional data across vision (e.g., DCGAN \cite{radford2015unsupervised}), audio (e.g., WaveGAN \cite{donahue2018adversarial}), and text domains (e.g., SeqGAN \cite{yu2017sequence}). Instead of directly estimating the data distribution, these models learn a mapping from a simple input distribution (such as a multivariate Gaussian) to the data space through a deterministic or stochastic function. A prominent example is the generator in a Generative Adversarial Network (GAN) \cite{goodfellow2014generative}, which transforms random noise vectors into realistic data samples. The generator is trained in tandem with a discriminator that learns to distinguish real data from generated data, providing feedback that guides the generator to improve. Unlike explicit models, implicit models do not require tractable likelihoods, allowing them to generate high-quality samples even when the data distribution is complex or high-dimensional.

The Invariant Statistical Loss (ISL) is a rank-based loss function recently proposed in \cite{de2024training} that compares the empirical order statistics of samples from the data and from the implicit generative model. In this work, we introduce \textbf{dual-ISL}, a novel likelihood-free objective obtained by swapping the roles of the data and model distributions within the ISL framework. Remarkably, the induced discrepancy \(d_K\) admits a fully \emph{explicit closed-form density approximation}: it is exactly the \(L^2\)–projection of the density ratio
\[
q(x)\;=\;\frac{p_{\mathrm{target}}(x)}{p_{\mathrm{model}}(x)}, 
\qquad
x = F_{\mathrm{target}}^{-1}(t),\;t\in[0,1],
\]
onto the space of dual-Bernstein polynomials of degree \(K\) \cite{lorentz2012bernstein,jiittler1998dual}. Writing
\[
q_K(x)
=\sum_{n=0}^K \mathbb{Q}_K(n)\,\tilde b_{n,K}\bigl(F_{\mathrm{target}}(x)\bigr)
\]
with computable coefficients \(\{\mathbb{Q}_K(n)\}\) immediately yields
\[
p_{\mathrm{model}}(x)\;\approx\;\frac{p_{\mathrm{target}}(x)}{q_K(x)}.
\]
This explicit representation not only provides analytic error bounds via Bernstein approximation theory and ensures convexity over the space of densities, but also enables efficient density evaluation without auxiliary sampling and provable convergence rates inherited from polynomial approximation. By marrying likelihood-free training with a tractable, closed-form density, dual-ISL bridges a critical gap—offering both rigorous theory and practical stability in implicit generative modeling. Moreover, this rank-based construction directly parallels the univariate optimal transport problem: matching order statistics via the probability–integral transform recovers the Monge map and yields the \(p\)-Wasserstein distance \cite{villani2008optimal}. Dual-ISL extends this perspective by providing an explicit, closed-form polynomial approximation of the density ratio—complete with convexity guarantees and convergence rates—rather than only a transport map.

\section{The Invariant Statistical Loss (ISL)} \label{section: preliminaries}

We briefly review the invariant statistical loss (ISL) from \cite{de2024training}.  ISL is built on a simple rank statistic whose distribution is exactly uniform when two samples come from the same probability density function, and which varies continuously under $L^1$‐perturbations of the underlying densities.

\subsection{Rank statistic and uniformity}\label{sec:rank}

Let $\tilde y_1,\dots,\tilde y_K$ be i.i.d.\ samples from a univariate real distribution with pdf $\tilde{p}$, and let $y$ be a single sample independently drawn from another distribution with pdf $p$. Define the subset
\begin{align*}
\mathcal{A}_{K} 
\;:=\; 
\Bigl\{\tilde{y} \in \{\tilde{y}_{k}\}_{k=1}^{K} : \tilde{y} \leq y\Bigr\},
\end{align*}
and the \emph{rank statistic}
\begin{align}\label{eq: rank statistic}
     A_{K} 
     \;:=\; 
     \bigl|\mathcal{A}_{K}\bigr|,
 \end{align}
 i.e., $A_{K}$ counts how many samples in $\{\tilde{y}_1,\ldots,\tilde{y}_K\}$ lie at or below $y$. Then $A_{K}$ is a discrete random variable (r.v.) taking values in $\{0,1,\ldots,K\}$, and we denote its pmf by 
 \begin{align*}
     \mathbb{Q}_K:\{0,\ldots,K\}\to[0,1].
 \end{align*}
When the two pdfs $p$ and $\tilde{p}$ coincide, this pmf is exactly uniform \cite{de2024training}.
\begin{theorem}\label{thm:uniformity}
If $p = \tilde{p}$, then $A_{K}$ is uniformly distributed on $\{0,\ldots,K\}$, i.e.\ $\mathbb{Q}_{K}(n)=\tfrac{1}{K+1}$ for all $n \in \{0,\ldots,K\}$.
\end{theorem}

\subsection{ISL discrepancy}\label{sec:loss}
The ISL discrepancy quantifies the deviation of the pmf $\mathbb{Q}_{K}$ from the uniform law on $\{0,\dots,K\}$. To be specific, we define the discrepancy function
\begin{align}\label{eq:dK}
d_{K}(p,\tilde{p})
:= \dfrac{1}{K+1}\|\mathbb{Q}_{K} - \mathbb{U}_{K}\|_{\ell_{1}}= 
\frac{1}{K+1}\sum_{n=0}^{K}
\Bigl|\,
\tfrac{1}{K+1} 
\;-\; 
\mathbb{Q}_{K}(n)
\Bigr|= \dfrac{2}{K+1} \mathrm{TV}(\mathbb{Q}_{K}, \mathbb{U}).
\end{align}
where $\mathbb U_K$ is the uniform pmf on $\{0,\dots,K\}$ and $\mathrm{TV}(\cdot,\cdot)$ denotes total variation distance.
By Theorem~\ref{thm:uniformity}, $d_K(p,p)=0$ for all $K$.  Moreover, Theorem~\ref{thm:continuity} below, ensures that \(d_K(p,\tilde p)\) depends continuously on \(\tilde p\) in the \(L^1\) sense, while Theorem~\ref{thm:identifiability} guarantees that if \(d_K(p,\tilde p)=0\) for all \(K\), then \(\tilde p=p\) almost everywhere.  Hence, in the large-\(K\) limit, \(d_K\) behaves as a proper divergence, vanishing precisely when the two densities coincide.

\begin{theorem}[Continuity]\label{thm:continuity}
If $\|p - \tilde{p}\|_{L^1(\mathbb{R})}\,\le\,\epsilon,$ then for all $n\in \{0,\ldots,K\}$,
\begin{align*}
 \frac{1}{K+1}-\epsilon
 \;\;\le\;\;
 \mathbb{Q}_{K}(n)
 \;\;\le\;\;
 \frac{1}{K+1}+\epsilon.
\end{align*}
\end{theorem}

\begin{theorem}[Identifiability]\label{thm:identifiability}
Let $p,\tilde{p}$ be pdfs of univariate real r.v.s. If the rank statistic $A_{K}$ in \eqref{eq: rank statistic} is uniformly distributed on $\{0,\ldots,K\}$ for \emph{every} $K\in \mathbb{N}$, then $p=\tilde{p}$ almost everywhere.
\end{theorem}

Finally, when $\tilde p=\tilde p_\theta$ depends smoothly on a parameter vector ~$\theta$, one can show (under mild regularity assumptions) that $\theta\mapsto d_K(p,\tilde p_\theta)$ is continuous and differentiable, making it suitable for gradient‐based optimization (see Theorem 4 in \cite{de2024robust}). For full proofs and additional remarks, see \cite{de2024training,de2024robust}.

\subsection{A surrogate for ISL optimization}


Directly minimizing the divergence $d_K\bigl(p,\tilde p_\theta\bigr)$ with respect to the generator parameters $\theta$ is normally not feasible: the pmf $\mathbb{Q}_{K}$ has to be approximated empirically and its dependence on $\theta$ is unknown. To overcome this difficulty, \cite{de2024training} introduced a carefully designed surrogate loss that (i) closely tracks $d_K$, and (ii) admits gradient optimization via standard backpropagation. This surrogate is constructing by approximating the pmf of $\mathbb{Q}_{K}$ using  sigmoidal functions and a Gaussian kernel density estimator. In practice, training with the surrogate yields virtually identical performance to optimizing the true ISL, while remaining fully likelihood‐free and amenable to efficient stochastic optimization. For full details of the surrogate derivation, implementation, and bias‐variance trade‐offs, see \cite[Section 2.3]{de2024robust}.

\section{The dual-Invariant statistical loss} \label{section: new theoretical results about discrepancy measure}

By interchanging the roles of the data distribution \(p\) and the model distribution \(\tilde p\) in the ISL framework, we obtain a \emph{dual} objective that remains likelihood-free, but crucially becomes convex in the model pdf $\tilde{p}$.  

\subsection{Continuity and convexity of $d_{K}(p,\tilde{p})$}

Unlike most classical discrepancies, this rank‐based measure is \emph{weakly continuous}: if \(p_n\overset{w}{\longrightarrow} p\) weakly, then \(\lim_{n\to\infty}d_K(p_n,\tilde p)= d_K(p,\tilde p)\) (Theorem~\ref{thm:rankWeakConv} below).  In contrast, the Kullback–Leibler divergence does not enjoy weak continuity, and the Wasserstein and Energy distances require uniformly bounded moments to guarantee even this level of stability \cite[Section 5]{huster2021pareto}.  Finally, we show that \(d_K\) is convex in its first argument (Theorem~\ref{Theorem: Convexity in the First Argument}), yielding a tractable convex optimization problem in the space of densities.

A key insight in the continuity proof is that each probability mass $\mathbb{Q}_{K}(n)$, for $n=0,\ldots, K$, can be written as a continuous mixture of the binomial pmf's. Indeed, drawing \(K\) i.i.d.\ samples \(\tilde y_i\sim\tilde p\) and counting how many of them fall below \(y\) yields a \(\mathrm{Binomial}(K,\tilde F(y))\) distribution.  Since \(y\) itself is drawn from \(p\), one obtains \cite[Appendix 1]{de2024training}
\[
\begin{aligned}
\mathbb{Q}_K(n)&:=\int_{\mathbb{R}}h_{n}(y)\,p(y)\,\mathrm{d}y\,,\\
\text{where, }&\; h_{n}(y):=\binom{K}{n}\,\tilde F(y)^{n}\bigl(1-\tilde F(y)\bigr)^{K-n}\,,\quad n=0,1,\dots,K\,.
\end{aligned}
\]
and the bounded, continuous functions \(h_n\) then ensure weak continuity of \(\mathbb{Q}_K\) and hence of \(d_K\).  We formalize this argument in the following theorem.




\begin{theorem}[Continuity under weak convergence]
\label{thm:rankWeakConv}
Let $(p_n)_{n\ge1}$ be a sequence of pdfs on~$\mathbb{R}$ converging weakly to a density~$p$, and let $\tilde p$ be a fixed reference density with cdf~$\tilde F$.  For each $K\in\mathbb{N}$, define
\begin{align*}
\mathbb Q_K^{(n)}(m)
:=\int_{\mathbb{R}}\binom Km\tilde F(y)^m\bigl(1-\tilde F(y)\bigr)^{K-m}\,p_n(y)\,\mathrm{d} y,
\end{align*}
for $m=0,\dots,K$.  Then
\begin{enumerate}[label=(\roman*)]
  \item (\emph{Pointwise convergence})  
    $\;\lim_{n\to \infty}\mathbb Q_K^{(n)}(m)\;=\;\mathbb Q_K(m)\;$ for each $\;m=0,\dots,K$.
  \item (\emph{Continuity of $d_K$}) $\;\lim_{n\to \infty}d_K(p_n,\tilde p)\;=\;d_K(p,\tilde p)\,$.
\end{enumerate}
\end{theorem}

\begin{proof}
    See Appendix \ref{appendix: proofs of new theoretical results about discrepancy measure}. 
\end{proof}

Since strong convergence implies weak convergence, the previous theorem remains applicable when the sequence $\{p_n\}_{n\geq 1}$ converges to $p$ in the $L^1$ norm. We can also establish that the ISL divergence is continuous with respect to its second argument $\tilde{p}$; a detailed proof can be found in Appendix~\ref{appendix: proofs of new theoretical results about discrepancy measure}, Theorem~\ref{thm:continuity-second-arg}.

We now see that the discrepancy $d_{K}(p,\tilde{p})$ is indeed convex in its first argument.
\begin{theorem}[Convexity] \label{Theorem: Convexity in the First Argument}
    For any probability distributions $p_1, p_2$ and $\tilde{p}$ on $\mathbb{R}$, and for any $\lambda \in [0, 1]$, the discrepancy $d_{K}$ satisfies
    \begin{align*}
    d_{K}(\lambda p_1 + (1 - \lambda) p_2, \, \tilde{p}) \leq \lambda \, d_{K}(p_1, \, \tilde{p}) + (1 - \lambda)\, d_{K}(p_2, \, \tilde{p})
    \end{align*}
\end{theorem}

\begin{proof}
    See Appendix~\ref{appendix: proofs of new theoretical results about discrepancy measure}.
\end{proof}

\subsection{A dual loss function}


Because \(d_K\) is convex in its first argument, we can obtain a new training criterion by swapping the data and model distributions.  Specifically, let \(\tilde y\sim \tilde p\) be a simulated sample from our generator and let \(y_{1:K}\overset{\mathrm{i.i.d.}}{\sim} p\) be \(K\) independent real data points.  We then form the rank statistic as
\begin{align*}
\tilde{A}_K \;:=\; 
\abs{\Bigl\{y \in \{y_{k}\}_{k=1}^{K} : y \leq \tilde{y}\Bigr\}}
\end{align*}
whose pmf \(\tilde{\mathbb{Q}}_K(n)=\mathbb{P}\left(\tilde{A}_K=n\right)\) remains uniform if and only if \(\tilde p=p\).  All of our previous guarantees—continuity under small \(L^1\) perturbations (Theorem~\ref{thm:continuity}) and identifiability when \(\mathbb{Q}_K\) is exactly uniform for every \(K\) (Theorem~\ref{thm:identifiability})—carry over unchanged. The dual‐ISL discrepancy
\[
d_K(\tilde p,p)
=\frac1{K+1}\sum_{n=0}^K\left|\tilde{\mathbb{Q}}_K(n)-\dfrac{1}{K+1}\right|
\]
therefore yields a convex, likelihood‐free training objective in the space of generator densities. The pseudocode for this method is provided in the supplementary material (see Algorithm~\ref{alg:dual-isl}).

\subsection{Dual-ISL vs. ISL, GANs \& Diffusion on 1D distributions}

We start considering the same experimental setup as \cite{zaheer2017gan, de2024training}. We evaluate dual‐ISL on six benchmark targets using \(N=1000\) i.i.d.\ samples drawn from each distribution.  The first three are standard univariate pdfs, and the latter three are mixtures with equal mixing weights. Model$_1$ combines Gaussians $\normdist{5}{2}$ and $\normdist{-1}{1}$; Model$_2$ combines Gaussians $\normdist{5}{2}$, $\normdist{-1}{1}$, and $\normdist{-10}{3}$; and Model$_3$ combines a Gaussian $\normdist{-5}{2}$ with a Pareto$(5,1)$ distribution.

We train a 4-layer MLP generator (7–13–7–1 units, ELU activations) with $\epsilon\sim\mathcal{N}(0,1)$ input noise for $10^4$ epochs with Adam (learning rate $10^{-2}$), and compare Dual-ISL, ISL, GAN \cite{goodfellow2014generative}, WGAN \cite{arjovsky2017wasserstein}, MMD-GAN \cite{li2017mmd}, and a diffusion baseline using the Kolmogorov–Smirnov distance (KSD) metric (Table \ref{paper Learning1D_table}). Experimental details are provided in Supplementary Material Section \ref{appendix: 1D experiments}.

\begin{table}[h]
\centering
\small                     
\setlength{\tabcolsep}{5pt}     
\rowcolors{2}{gray!10}{white}
\resizebox{\linewidth}{!}{       
  \begin{tabular}{
    l
    c
    c
    c
    c
    c
    c
  }
  \toprule
  \textbf{Target} 
    & \textbf{Dual‐ISL} 
    & \textbf{ISL} 
    & \textbf{GAN} 
    & \textbf{WGAN} 
    & \textbf{MMD‐GAN} 
    & \textbf{Diffusion}\\
  \midrule
  \(\mathcal{N}(4,2)\)  
    & \(\mathbf{0.018 \pm 0.005}\)
    & \(0.020 \pm 0.003\) 
    & \(0.018 \pm 0.003\) 
    & \(0.024 \pm 0.017\)
    & \(0.042 \pm 0.026\) 
    & \(0.020 \pm 0.002\)\\
  \(\mathcal{U}(-2,2)\) 
    & \(0.034 \pm 0.015\) 
    & \(0.021 \pm 0.004\) 
    & \(0.049 \pm 0.032\) 
    & \(0.064 \pm 0.062\) 
    & \(0.104 \pm 0.060\) 
    & \(\mathbf{0.013 \pm 0.002}\)\\
  \(\mathrm{Cauchy}(1,2)\)    
    & \(0.016 \pm 0.003\) 
    & \(\mathbf{0.013 \pm 0.002}\) 
    & \(0.013 \pm 0.002\) 
    & \(0.052 \pm 0.055\) 
    & \(0.031 \pm 0.008\) 
    & \(0.114 \pm 0.034\)\\
  \(\mathrm{Pareto}(1,1)\)    
    & \(\mathbf{0.090 \pm 0.080}\) 
    & \(0.198 \pm 0.148\) 
    & \(0.117 \pm 0.041\) 
    & \(0.106 \pm 0.043\) 
    & \(0.158 \pm 0.168\) 
    & \(0.209 \pm 0.011\)\\
  \(\mathrm{Mixture}_1\)       
    & \(0.016 \pm 0.004\) 
    & \(\mathbf{0.016 \pm 0.002}\) 
    & \(0.017 \pm 0.004\) 
    & \(0.080 \pm 0.069\) 
    & \(0.054 \pm 0.033\) 
    & \(0.031 \pm 0.031\)\\
  \(\mathrm{Mixture}_2\)        
    & \(\mathbf{0.016 \pm 0.002}\) 
    & \(0.017 \pm 0.003\) 
    & \(0.026 \pm 0.014\) 
    & \(0.031 \pm 0.023\) 
    & \(0.042 \pm 0.061\) 
    & \(0.050 \pm 0.005\)\\
  \(\mathrm{Mixture}_3\)        
    & \(\mathbf{0.170 \pm 0.019}\) 
    & \(0.171 \pm 0.012\) 
    & \(0.190 \pm 0.094\) 
    & \(0.216 \pm 0.040\) 
    & \(0.187 \pm 0.108\) 
    & \(0.173 \pm 0.024\)\\
  \bottomrule
  \end{tabular}
}
\caption{\small KSD  over 10 runs for Dual‐ISL and baselines.  
Setup: \(K=10\), 1000 epochs, \(N=1000\).}
\label{paper Learning1D_table}
\end{table}


The convexity of the dual-ISL objective not only accelerates convergence—yielding faster, smoother, and more stable training curves compared to classical ISL (see Figure \ref{figure: convergence rate in epochs dual vs isl} in Appendix)—but also enhances mode coverage on challenging mixtures.  As shown in Figure \ref{figure:dual-isl vs isl vs mmd-gan mixture normal pareto}, both dual-ISL and classical ISL successfully avoid the mode collapse exhibited by MMD-GAN, with dual-ISL most accurately capturing the heavy tail of the Pareto component.  


\begin{figure}[htbp] 
   \centering
   \begin{subfigure}{0.32\textwidth}
     \centering
     \includegraphics[width=\linewidth]{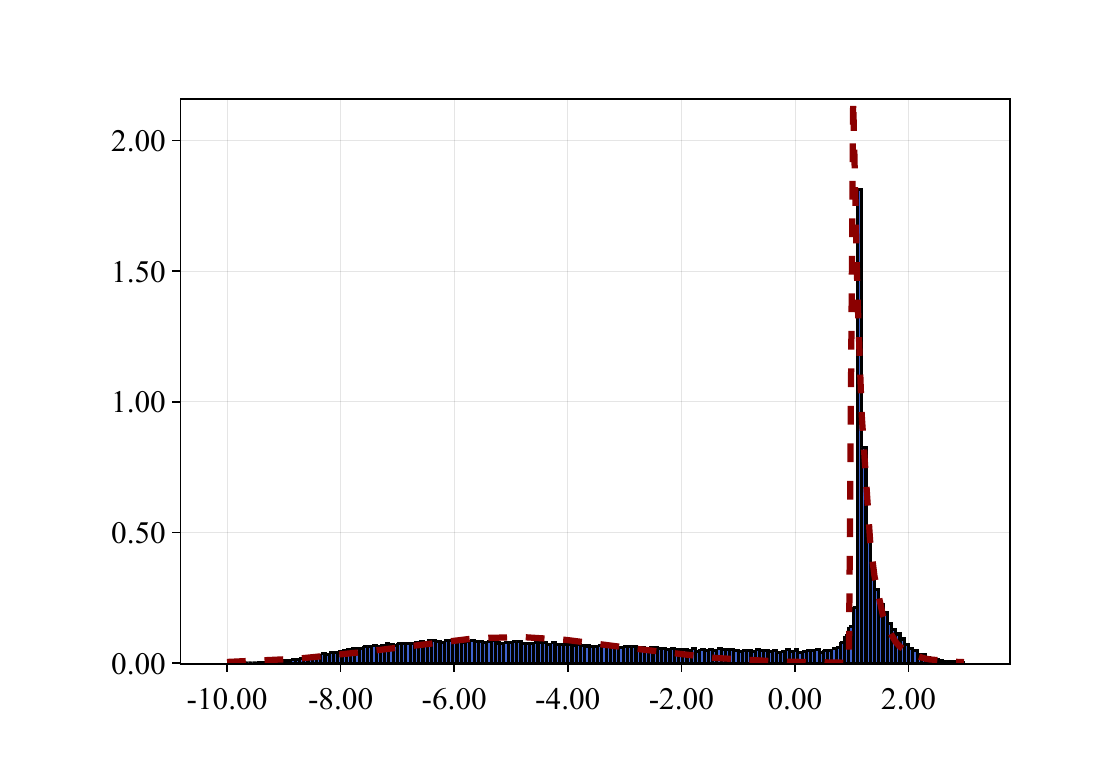}
     \caption{dual-ISL}
     \label{figure:dual-isl vs isl vs mmd-gan mixture normal pareto:dual-ISL}
   \end{subfigure}
   \hfill
   \begin{subfigure}{0.32\textwidth}
     \centering
     \includegraphics[width=\linewidth]{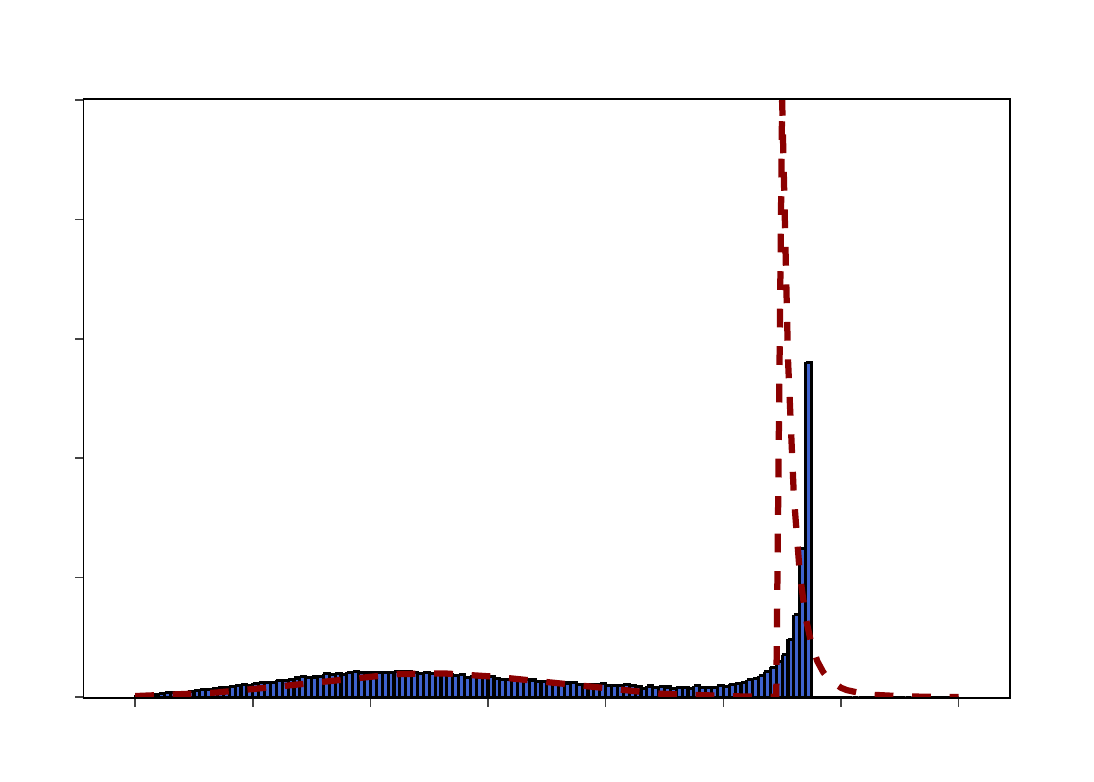}
     \caption{ISL}
     \label{figure:dual-isl vs isl vs mmd-gan mixture normal pareto:isl}
   \end{subfigure}
   \hfill
   \begin{subfigure}{0.32\textwidth}
     \centering
     \includegraphics[width=\linewidth]{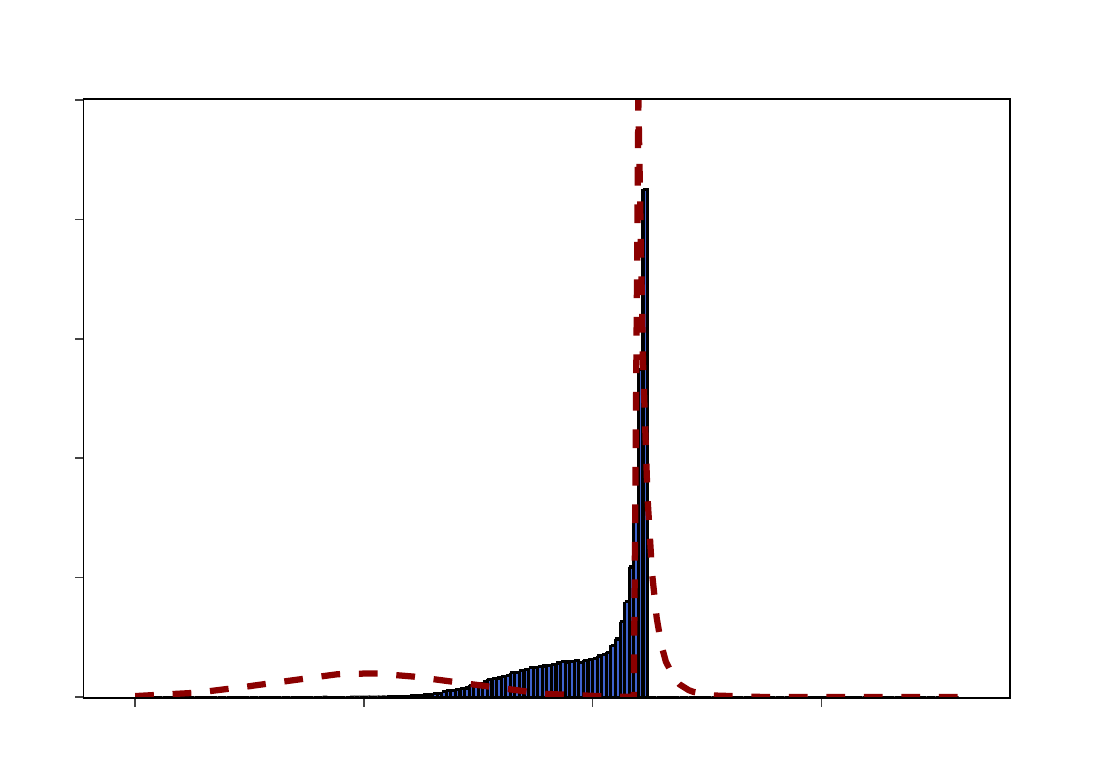}
     \caption{MMD-GAN}
     \label{figure:dual-isl vs isl vs mmd-gan mixture normal pareto:mmd-gan}
   \end{subfigure}
   \caption{\small Comparison of dual-ISL, standard ISL, and MMD-GAN for modeling a mixture of Pareto and Normal distributions. Subfigure \ref{figure:dual-isl vs isl vs mmd-gan mixture normal pareto:dual-ISL} displays the dual-ISL results, Subfigure \ref{figure:dual-isl vs isl vs mmd-gan mixture normal pareto:isl} illustrates the performance of the standard ISL approach, and Subfigure \ref{figure:dual-isl vs isl vs mmd-gan mixture normal pareto:mmd-gan} showcases the outcomes obtained via MMD-GAN. Further comparisons—including diffusion models and additional target distributions—are provided in Appendix~\ref{Supplementary experiments}}
   \label{figure:dual-isl vs isl vs mmd-gan mixture normal pareto}
\end{figure}

Additional experiments in the supplementary material provide detailed runtime benchmarks, demonstrating the computational advantages of dual‐ISL over the standard ISL formulation. In Appendix \ref{Moment-Agnostic Optimal Transport via Monotonicity-Penalized ISL}, we also propose a new ISL-based method with a monotonicity penalty that guarantees recovery of the optimal-transport map even for distributions without finite moments (e.g., heavy-tailed), an advantage over the \(p\)-Wasserstein distance which requires finite \(p\)th moments.

\section{An \(L^2\)-projection view of $d_{K}$}
\label{section: A rank-based / binomial mapping}



We adopt a projection‐based view of ISL. From this point on, we treat $p$ and $\tilde p$ interchangeably—so that, with a slight abuse of notation, our framework covers both standard ISL ($q=p/\tilde p$) and dual‐ISL ($q=\tilde p/p$).  Specifically, we show that the discrete pmf $\mathbb{Q}_K(n)$ coincides with the $L^2$-projection coefficients of the density ratio $q=p(x)/\tilde p(x)$ onto the degree-$K$ Bernstein basis $\{b_{n,K}\}_{n=0}^K$ .  In this light, ISL becomes a purely likelihood-free density-ratio divergence—comparing projection coefficients rather than intractable likelihoods —and we conclude by deriving sharp convergence rate bounds.

\subsection{Projection interpretation}

To reveal the underlying geometry of $d_{K}(\cdot, \cdot)$, we define a linear operator that collects the $K+1$ probabilities $\mathbb{Q}_{K}(0), \ldots, \mathbb{Q}_{K}(K)$, into a single vector. We then show that each entry $\mathbb{Q}_{K}(n)$ is precisely the $L^2$ inner product between a density ratio and its corresponding Bernstein basis function.


\begin{definition}[Binomial mapping]
\label{def:PhiK}
Let \(p,\tilde p\in C(\mathbb{R})\) be two continuous pdfs with cdfs \(F\) and \(\tilde F\).  For any integer \(K\ge1\), define the operator, 
\begin{align*}
\Phi_{K}(p,\tilde p)
:=
\bigl(\mathbb{Q}_{K}(0),\,\mathbb{Q}_{K}(1),\,\dots,\,\mathbb{Q}_{K}(K)\bigr)
\;\in\;\mathbb{R}^{K+1}.
\end{align*}
\end{definition}


It is straightforward from the integral representation that, for each fixed $\tilde p$, the map $p \;\mapsto\;\Phi_K(p,\tilde p)$ is linear and continuous under mild regularity conditions on $p$ and $\tilde{p}$.  A full statement of these and related properties appears in Theorem \ref{thm:PhiK_Properties}.


The next result shows that \(\Phi_K\) admits a Riesz representation (see \cite[Theorem~4.11]{brezis2011functional}), expressing each probability mass $\mathbb{Q}_{K}$ as an \(L^2\) inner product with a Bernstein basis function.  Let the \(n\)th Bernstein polynomial of degree \(K\) be defined as
\begin{align*}
b_{n,K}(t)
\;:=\;
\binom{K}{n}\,t^n\,(1-t)^{K-n},
\qquad
t\in[0,1].
\end{align*}



\begin{theorem}(Riesz representation of $\Phi_{K}$)\label{theorem:Riesz Representation}
Let \(\tilde p\) be a fixed continuous density on \(\mathbb{R}\) with cdf \(\tilde F\).  Then for any \(K\ge0\), the operator
\(\Phi_K(\cdot, \tilde p)\) mapping \(p\mapsto(\mathbb{Q}_K(0),\dots,\mathbb{Q}_K(K))\) satisfies
\[
\mathbb{Q}_K(n)
=\int_{\mathbb{R}}b_{n,K}\bigl(\tilde F(x)\bigr)\,p(x)\,\mathrm{d}x
=\bigl\langle\,b_{n,K}\circ\tilde F,\;p\bigr\rangle_{L^2(\mathbb{R})},
\qquad n=0,\dots,K.
\]
Moreover, if $\tilde p(x)>0$ for all $x\in\mathbb{R}$, then defining the density ratio \(q(x)=p(x)/\tilde p(x)\), we get
\[
\mathbb{Q}_K(n)
=\int_{0}^{1} b_{n,K}(t)\,q\bigl(\tilde F^{-1}(t)\bigr)\,\mathrm{d}t
=\bigl\langle\,b_{n,K},\,q\circ\tilde F^{-1}\bigr\rangle_{L^2([0,1])}, \quad n=0,\ldots,K.
\]
\end{theorem}
\begin{proof}
    See Appendix \ref{appendix:proofs section A rank-based / binomial mapping}.
\end{proof}

Theorem~\ref{theorem:Riesz Representation} implies that, if we define $\tilde{q}(t) = q\bigl(\tilde F^{-1}(t)\bigr)$ for $t\in[0,1]$ then each coefficient $\mathbb Q_K(n)$ is exactly the \(L^2([0,1])\)-projection of \(\tilde{q}\) onto the Bernstein polynomial \(b_{n,K}\).  Equivalently, the vector $\{\mathbb{Q}_K(n)\}_{n=0}^{K}$ collects the best mean‐square approximation coefficients of the  \emph{push-forward density ratio} \(q\circ\tilde F^{-1}\) in the degree-\(K\) Bernstein basis.


\begin{theorem}[Bernstein‐basis truncation for the density‐ratio]
\label{thm:bernstein_truncation_ratio}
Let \(p,\tilde p\in C(\mathbb{R})\) with \(\tilde p(x)>0\) for all $x\in \mathbb{R}$. Then \(\tilde{q}\in C([0,1])\) admits the Bernstein‐polynomial expansion
\begin{align*}
\tilde q(t)
\;=\;
\sum_{n=0}^{\infty}\alpha_{n}\,b_{n,K}(t),
\end{align*}
where \(\alpha=(\alpha_n)_{n\ge0}\) are the unique Bernstein‐basis coordinates of \(\tilde q\). Its degree-\(K\) truncation can be expressed as
\begin{align*}
\tilde{q}_{K}(t)
\;:=\;
\sum_{n=0}^{K}\alpha_{n}\,b_{n,K}(t) = \sum_{n=0}^{K} \mathbb{Q}_{K}(n) \, \tilde{b}_{n,K}(t),
\end{align*}
where \(\{\tilde b_{n,K}\}_{n=0}^K\) is the dual Bernstein basis \cite{jiittler1998dual}.
\end{theorem}
\begin{proof}
    See Proof in Appendix \ref{appendix:proofs section A rank-based / binomial mapping}.
\end{proof}

\begin{remark}
With \(\tilde p\) fixed, the map $\Phi_K: p \;\longmapsto\; \bigl(\mathbb Q_K(0),\dots,\mathbb Q_K(K)\bigr)$, cannot distinguish between any two target densities \(p_1, p_2\) whose pushed‐through ratios \(q_i\circ\tilde F^{-1}\) have the same degree-\(K\) Bernstein projections.  Equivalently,
\[
\Phi_K(p_1)=\Phi_K(p_2)
\quad\Longleftrightarrow\quad
\bigl\langle b_{n,K},\,q_1\circ\tilde F^{-1}-q_2\circ\tilde F^{-1}\bigr\rangle_{L^2([0,1])}
=0
\quad\forall\,n=0,\dots,K.
\]
Thus \(\Phi_K\) factors through the quotient of \(L^2([0,1])\) by the subspace orthogonal to \(\mathrm{span}\{b_{n,K}\}\), inducing a bijection onto its image.
\end{remark}

\subsection{Approximation error}


We now quantify the truncation error in the approximation ratio $\tilde{q}_{K}(t)$ is an estimate of $q(x)$ and it remains uniformly close to 1, with its sup-norm deviation bounded by the discrepancy \(d_K(p,\tilde p)\).  

\begin{theorem}\label{lemma: uniform boudn for q_K convergence to 1}
Let $p,\tilde{p}\in C(\mathbb{R})$ be pdfs. Then $q_K(x)$ satisfies
\begin{align*}
\|q_K - 1\|_{\infty} \leq (K+1)^{2}\, d_K(p,\tilde{p}).
\end{align*}
\end{theorem}
\begin{proof}
    See Proof in Appendix \ref{appendix:proofs section A rank-based / binomial mapping}
\end{proof}


By standard Bernstein‐approximation theory, one can bound the truncation error via the modulus of continuity of \(q\).  In particular, if \(q\in C^2(\mathbb{R})\), then 
\[
\|q - q_K\|_{\infty} = O(K^{-1}),
\]
and more generally, if \(q\) is \(\alpha\)–Hölder continuous on \(\mathbb{R}\), then
\[
\|q - q_K\|_{\infty} = O\bigl(K^{-\alpha/2}\bigr).
\]
See \cite{gzyl1997weierstrass} for the \(C^2\) case and \cite{mathe1999approximation} for the Hölder regime.

\begin{remark} \label{remark:extansion of bertein to other domains}
 If we assume that $q\in C^{2}(\mathbb{R})$, by the triangle inequality, we have 
        \begin{align}\label{eq:convergence of qK to 1}
            \|q(x) - 1\|_{\infty}\leq \|1 - q_{K}(x)\|_{\infty} + \|q_{K}(x) - q(x)\|_{\infty} \leq (K+1)^{2} d_{K}(p, \tilde{p}) + \dfrac{\|q(x)''\|_{\infty}}{8K}.
        \end{align}
\end{remark}


To empirically validate Equation \ref{eq:convergence of qK to 1}, we train the same NN architecture under identical hyperparameters as in our earlier experiments. The model receives as input noise \(z\sim\mathcal{N}(0,1)\) and approximates a mixture of Cauchy distributions. We then recover the estimated density \(\tilde p\) via kernel density estimation and compute the second derivative of the quotient \(q\) with sixth‐order central finite differences. Each experiment is repeated ten times, and the mean results are plotted in Figure \ref{fig:empirical-convergence-rate}.


\begin{figure}[htbp]
  \centering
  \captionsetup[subfigure]{justification=centering, font=small, labelfont=bf}
  \captionsetup{font=small, labelfont=bf, skip=4pt}
  \begin{subfigure}[t]{0.48\linewidth}
    \centering
    \includegraphics[width=\linewidth,keepaspectratio]{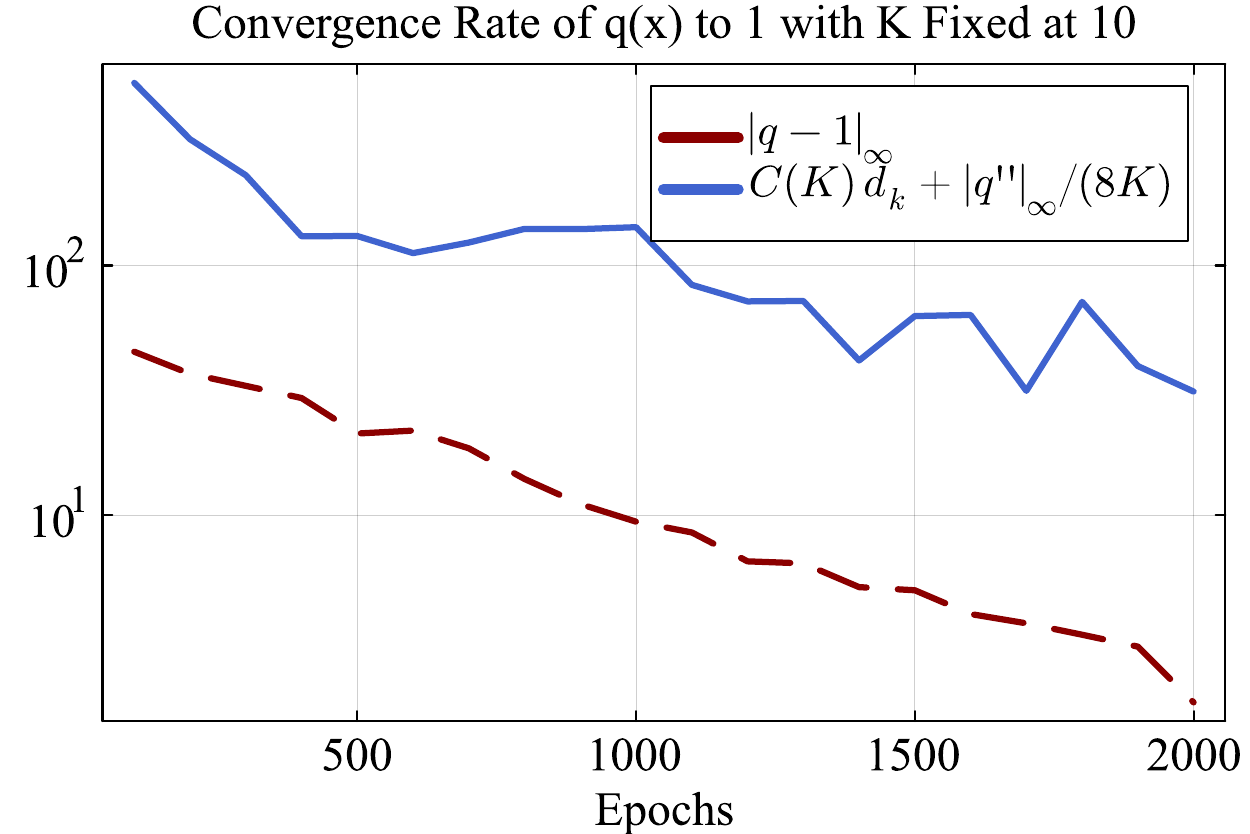}
    \caption{\small Fixed \(K=10\). NN trained with \(N=1000\) samples, lr \(10^{-3}\).}
    \label{fig:conv-rate-fixed-K}
  \end{subfigure}
  \hfill
  \begin{subfigure}[t]{0.48\linewidth}
    \centering
    \includegraphics[width=\linewidth,keepaspectratio]{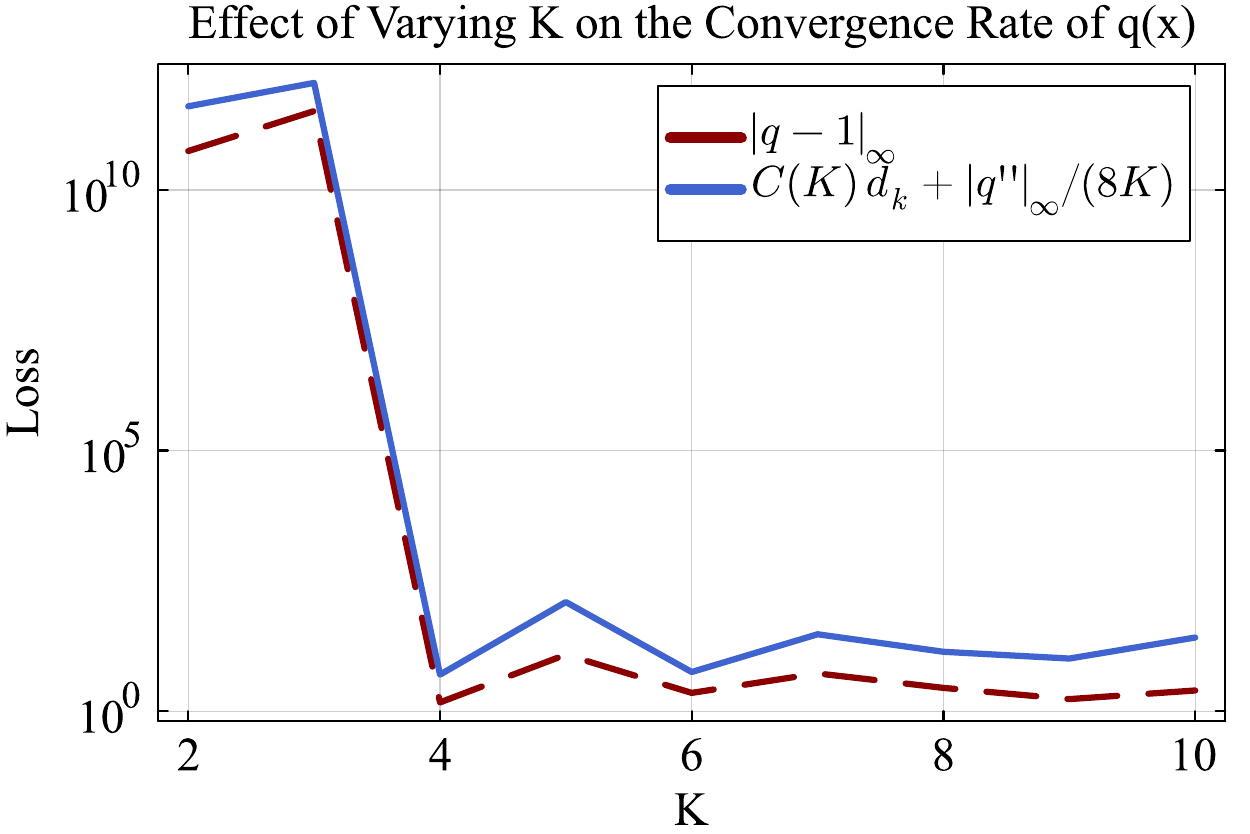}
    \caption{Varying \(K\). Each run uses 1000 epochs, \(N=1000\), lr \(10^{-3}\).}
    \label{fig:conv-rate-vary-K}
  \end{subfigure}
  \caption{\small Empirical convergence of ISL’s Bernstein approximation (cf.\ Eq.~\ref{eq:convergence of qK to 1}).  
    The solid blue curve shows the mean theoretical upper bound \(\|q_K - 1\|_\infty\le (K+1)^2d_K\), and the dashed red curve shows the observed \(\|q - 1\|_\infty\).}
  \label{fig:empirical-convergence-rate}
\end{figure}

\begin{theorem}[Explicit density approximation]\label{theorem:general_forward_pushforward_expanded}
Let \(p,\tilde p\in C(\mathbb{R})\) with \(\tilde p(x)>0\) for all \(x\in\mathbb{R}\), and let $\tilde F(x)$ be the cdf of \(\tilde p\).  Define
\begin{align}\label{eq:pK definition}
p_K(x)
:=\tilde p(x)\,\sum_{m=0}^{K}\mathbb{Q}_{K}(m)\,\tilde b_{m,K}\!\bigl(\tilde F(x)\bigr).
\end{align}
Then for every \(x\in\mathbb{R}\),
\begin{align*}
\lim_{K\to\infty}p_K(x)
= p(x).
\end{align*}
\end{theorem}

\begin{proof}
    See Proof in Appendix \ref{appendix:proofs section A rank-based / binomial mapping}.
\end{proof}

\begin{remark}
In practice, one draws latent samples \(z_1,\dots,z_N\overset{\mathrm{i.i.d.}}{\sim} p_z\) and computes
$x_i = f(z_i),$ where \(f\) is the neural network pushing \(p_z\) forward to \(\tilde p\).  One then forms the empirical cdf and density estimates
\begin{align*}
\widehat{\tilde F}(x)
=\frac{1}{N}\sum_{i=1}^N\mathbf{1}\{x_i\le x\},
\qquad
\widehat{\tilde p}(x)
=\frac{\widehat{\tilde F}(x+\delta)-\widehat{\tilde F}(x-\delta)}{2\delta}.
\end{align*}
Substituting these into Equation \eqref{eq:pK definition} yields the Monte Carlo approximation
\begin{align}\label{eq:pK estimation}
\widehat p_K(x)
=\widehat{\tilde p}(x)\sum_{m=0}^K\mathbb Q_K(m)\,\tilde b_{m,K}\bigl(\widehat{\tilde F}(x)\bigr).
\end{align}
\end{remark}

In Figure~\ref{fig:density-estimation} we illustrate ISL’s capability—via its Bernstein polynomial approximation—to recover the true density in both one-dimensional and two-dimensional settings. Figure \ref{fig:density-estimation:a} compares the ground-truth mixture Gaussian (red) with dual-ISL estimates at \(K=2\) (light blue) and \(K=15\) (dark blue), while Figure \ref{fig:density-estimation:b} overlays the estimated density contours on the two-moons sample scatter. Additional experiments and implementation details, are provided in Appendix \ref{appendix:Density estimation}.

\begin{figure}[htbp]
  \captionsetup[subfigure]{justification=centering, font=small, labelfont=bf, skip=1pt}
  \centering
  \begin{subfigure}[bt]{0.45\textwidth}
    \centering
\includegraphics[width=\linewidth,height=5cm,keepaspectratio]{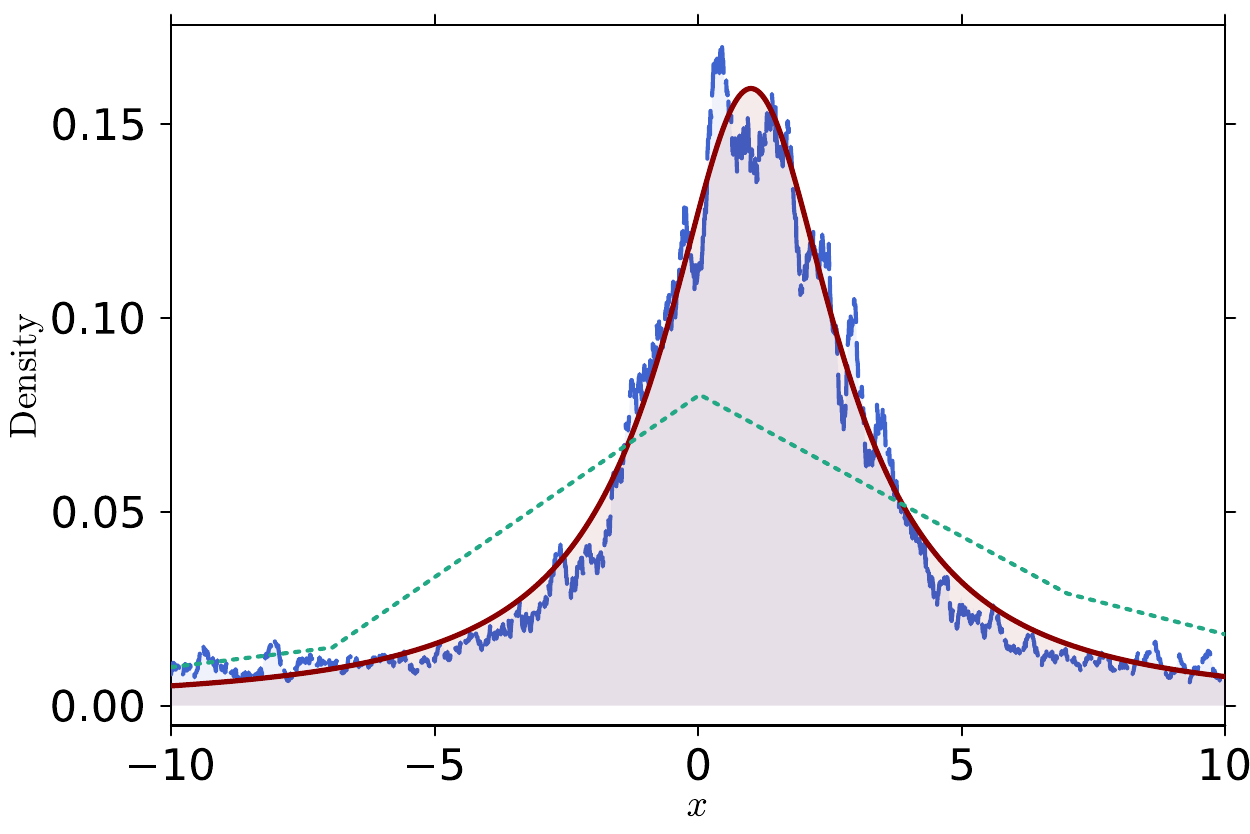}
    \caption{\small Comparison of the true $\mathrm{Cauchy}(1,2)$ density (red), the dual-ISL estimate for $K=10$ (blue), and the kernel density estimate (green).}
    \label{fig:density-estimation:a}
  \end{subfigure}
  \hfill
  \begin{subfigure}[bt]{0.45\textwidth}
    \centering
    \includegraphics[width=\linewidth,height=4cm,keepaspectratio]{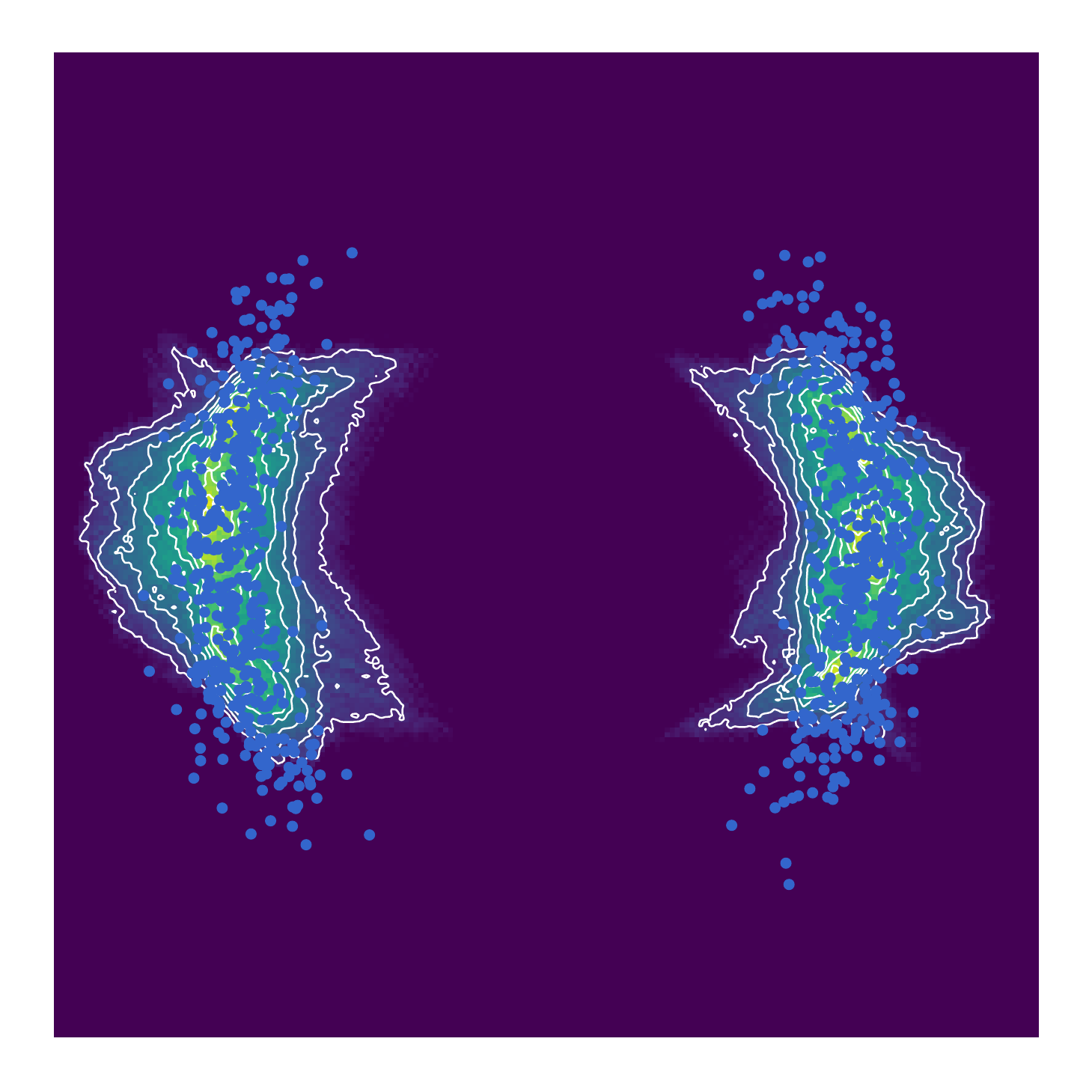}
    \caption{\small Two‐moons target: sample scatter (blue points) with dual-ISL density overlay.}
    \label{fig:density-estimation:b}
  \end{subfigure}
  \caption{\small Dual-ISL density estimation results. (a) On a 1D Cauchy target, dual-ISL (blue) closely matches the true density (red) and outperforms the KDE baseline (green). (b) On a 2D two-moons dataset, dual-ISL accurately captures the manifold structure, with learned contours aligning tightly with the sample cloud.}
  \label{fig:density-estimation}
\end{figure}

\section{Sliced multivariate ISL via Bernstein polynomial approximation} \label{section:Sliced Multivariate ISL via Bernstein Polynomial Approximation}

When the data are multidimensional, the target $p(x)$ is a pdf on $\mathbb{R}^{d}$, with $d>1$, and there is no finite set of univariate statistics that uniquely characterizes an arbitrary density (cf. Theorem \ref{thm:identifiability} in 1D).  Instead, we employ a sliced strategy: we assess a $d$-dimensional distribution by projecting it onto many random directions, computing the one-dimensional ISL discrepancy along each slice, and then averaging these values \cite{de2024robust,kolouri2019generalized}.

\subsection*{One–dimensional projected statistic} 
For any unit vector in the $d$-dimensional sphere \(s\in\mathbb S^{d}\subset\mathbb{R}^{d+1}\), denote by $s\#p$ the pdf of the one-dimensional projection $y = s^\top x$ with associated cdf denoted by \(\tilde F_s\).  Then the pmf of order \(K\) in direction \(s\) is
\begin{align*}
\mathbb{Q}_K^s(n) &= \int_{\mathbb{R}} \binom{K}{n}\,[\tilde{F}_s(y)]^n\,[1-\tilde{F}_s(y)]^{K-n}\,(s\#p)(y)\,\mathrm{d}y \\&= \int_{\mathbb{R}}\binom{K}{n} t^{n} (1 - t)^{K-n} \dfrac{s\#p(\tilde{F}_{s}^{-1}(t))}{s\#\tilde{p}(\tilde{F}_{s}^{-1}(t))}\mathrm{d}t = \langle b_{n,K}, q^{s} \circ \tilde{F}_{s}^{-1}\rangle_{L^{2}(\mathbb{R})}, \qquad n=0,\ldots, K. 
\end{align*}
where we have denoted by $q^{s}(x) = s\#q(x) = \dfrac{s\#p}{s\#\tilde{p}}(x) $ the push-forward of the quotient $q$ by the linear transformation $s$.

\subsection*{Sliced ISL divergence}
We then define the \emph{sliced} ISL discrepancy by integrating over the unit sphere,
\begin{align}\label{eq:sliced-d}
d_{K}^{\mathbb S^d}(p,\tilde p)
\;=\;\int_{\mathbb S^d}
d_K\bigl(s\#p,\;s\#\tilde p\bigr)\,\mathrm{d}s,
\end{align}
where \(d_K\) is the discrepancy in Definition~\ref{eq:dK}. In practice, to approximate the integral in Equation \ref{eq:sliced-d}, one randomly samples a finite set of directions \(\{s_\ell\}_{\ell=1}^L\) and averages the resulting evaluations.

The following Theorem is derived using the bounds of Equation~\ref{eq:convergence of qK to 1}, and shows that, under mild smoothness assumptions on $q(x)=p(x)/\tilde p(x)$, if $\lim_{K\to\infty}d^{\mathbb S^d}_{K}(p,\tilde p)=0$ then every one‐dimensional projected ratio $q^s$ converges uniformly to 1.  By the Cramér--Wold theorem \cite[Thm.~29.4]{billingsley2017probability}, this ensures that $p=\tilde p$ on $\mathbb{R}^d$, and hence $d_{\mathbb S^d}^K$ becomes a proper divergence as $K\to\infty$.
\begin{theorem}[Uniform convergence under slicing]\label{thm:qs_convergence}
Let $p,\tilde p\in C^2\left(\mathbb{R}^d\right)$. Then there is a constant \(C_d=\mathcal L(\mathbb S^d)\) such that
\begin{enumerate}[label=(\roman*)]
  \item\label{thm:qs_convergence:i}
    $\displaystyle
      \int_{\mathbb S^d}\|q^s-1\|_\infty\,ds
      \le (K+1)^2\,d_{K}^{\mathbb S^d}(p,\tilde p)
      + C_d\;\frac{\|\nabla^2 q\|_\infty}{8K},
    $ 
  \item\label{thm:qs_convergence:ii}
    $\displaystyle
      \sup_{s\in\mathbb S^d}\|q^s-1\|_\infty
      \le (K+1)^2\,
      \sup_{s\in\mathbb S^d}d_K\bigl(s\#p,s\#\tilde p\bigr)
      + \frac{\|\nabla^2 q\|_\infty}{8K}.
    $
\end{enumerate}

Here \(\|\nabla^2 q\|_\infty=\sup_{x\in[0,1]^d}\|\nabla^2 q(x)\|\) and 
\(\mathcal L(\mathbb S^d)\) is the surface measure of the sphere.
\end{theorem}

\begin{proof}
See Appendix~\ref{section: Multidimensional Extension of the Projection Schema}.
\end{proof}


Since \(s\mapsto s\#(\cdot)\) is linear, compactness of \(\mathbb S^d\) plus
Theorems~\ref{thm:rankWeakConv}–\ref{Theorem: Convexity in the First Argument} imply that
\(\;(p,\tilde p)\mapsto d_{K}^{\mathbb S^d}(p,\tilde p)\) is continuous
and convex in its first argument.  Consequently, by interchanging the roles of the model and target distributions in the slicing framework we obtain a \emph{sliced dual‐ISL} method that retains both convexity and differentiability (almost everywhere) under mild smoothness of the network parameters. Pseudocode for its implementation is given in Appendix \ref{pseudocodes}.

\section{Summary and concluding remarks}

In this paper, we introduced dual-ISL, a novel likelihood-free objective that significantly advances the training of implicit generative models. By interchanging the roles of the target and model distributions within the Invariant Statistical Loss (ISL) framework, dual-ISL provides a convex optimization problem in the space of model densities, addressing common challenges like instability, non-convexity, and mode collapse prevalent in existing methods.

A central theoretical contribution is the interpretation of dual-ISL as an explicit \(L^2\)-projection of the \emph{push-forward density ratio} $q \circ \tilde{F}^{-1}$ onto a Bernstein polynomial basis. This innovative projection approach yields an explicit closed-form approximation of the density ratio, enabling efficient and analytically tractable density evaluation—a capability traditionally missing in implicit modeling. We derived precise error bounds and convergence rates leveraging classical results from polynomial approximation theory, thus ensuring both theoretical rigor and practical stability.

We further generalized dual-ISL to multivariate distributions through a sliced projection methodology, maintaining convexity, continuity, and analytic tractability in higher-dimensional settings. Empirically, dual-ISL consistently demonstrated improved convergence, smoother training dynamics, and reduced mode collapse compared to classical ISL, GAN variants (including WGAN and MMD-GAN), and normalizing flow baselines across a variety of synthetic benchmarks.

In conclusion, dual-ISL bridges an important gap in implicit generative modeling by providing both strong theoretical foundations and practical advantages. Future directions include exploring adaptive slicing strategies, extending the theoretical analysis to broader classes of polynomial bases, and applying dual-ISL to large-scale generative modeling tasks in diverse real-world domains.

\section*{Acknowledgments}

This work has been supported by the the Office of Naval Research (award N00014-22-1-2647) and Spain's {\em Agencia Estatal de Investigaci\'on} (refs. PID2021-125159NB-I00 TYCHE and PID2021-123182OB-I00 EPiCENTER). Pablo M. Olmos also acknowledges the support by Comunidad de Madrid under grants IND2022/TIC-23550 and ELLIS Unit Madrid.


\printbibliography


\newpage

\restoreTOCentries

\appendix

\tableofcontents
\vfill
\newpage

\section{Proofs of Theorems Section \ref{section: new theoretical results about discrepancy measure}} \label{appendix: proofs of new theoretical results about discrepancy measure}


In this appendix, we establish three key analytic properties of the rank‐based divergence \(d_K\)
\begin{itemize}
  \item Continuity in its first argument under weak convergence (Theorem \ref{thm:rankWeakConv}).
  \item Continuity in its second argument under \(L^1\) norm (Theorem \ref{thm:continuity-second-arg}).
  \item Convexity in its first argument (Theorem \ref{Theorem: Convexity in the First Argument}).
\end{itemize}

\subsection{Proof of Theorem \ref{thm:rankWeakConv}}
\begin{proof}[Proof of Theorem \ref{thm:rankWeakConv}]
  \smallskip           
  \noindent
\begin{enumerate}[label=(\roman*)]
\item For each fixed \(m \in \{0,1,\dots,K\}\), define
\begin{align*}
h_{m}(y)=\binom{K}{m}\,[\tilde{F}(y)]^m\,[1-\tilde{F}(y)]^{K-m}.
\end{align*}
Since \(\tilde{F}(y)\) is the cdf of the fixed density \(\tilde{p}\), the function \(h_{m}(y)\) is continuous and bounded on $\mathbb{R}$. By the definition of weak convergence $p_n \overset{w}{\longrightarrow} p$ we have that for every bounded continuous function $h$,
\begin{align*}
\lim_{n\to\infty}\int_{\mathbb{R}} h(y)\,p_n(y)\,dy = \int_{\mathbb{R}} h(y)\,p(y)\,dy.
\end{align*}
Taking \(h(y)=h_{m}(y)\) yields
\begin{align*}
\lim_{n\to\infty}\mathbb{Q}_K^{(n)}(m) = \lim_{n\to\infty}\int_{\mathbb{R}} h_m(y)\,p_n(y)\,dy = \int_{\mathbb{R}} h_m(y)\,p(y)\,dy = \mathbb{Q}_K(m).
\end{align*}
\item The discrepancy between $p$ and $\tilde{p}$ is
\[
d_K(p,\tilde{p}) = \frac{1}{K+1}\sum_{m=0}^{K}\left|\frac{1}{K+1} - \mathbb{Q}_K(m)\right|.
\]
Since for each \(m\) we have shown that
\[
\lim_{n\to\infty}\mathbb{Q}_K^{(n)}(m) = \mathbb{Q}_K(m),
\]
and because the absolute value function is continuous, it follows that
\[
\lim_{n\to\infty}\left|\frac{1}{K+1} - \mathbb{Q}_K^{(n)}(m)\right| = \left|\frac{1}{K+1} - \mathbb{Q}_K(m)\right|.
\]
As the sum is finite (from \(m=0\) to \(K\)), we can exchange the limit and the summation to conclude that
\begin{align*}
\lim_{n\to\infty} d_K(p_n,\tilde{p}) &= \frac{1}{K+1}\sum_{m=0}^{K}\lim_{n\to\infty}\left|\frac{1}{K+1} - \mathbb{Q}_K^{(n)}(m)\right| \\&= \frac{1}{K+1}\sum_{m=0}^{K}\left|\frac{1}{K+1} - \mathbb{Q}_K(m)\right| = d_K(p,\tilde{p}).
\end{align*}
\end{enumerate}
\end{proof}

\subsection{Proof of Theorem \ref{thm:continuity-second-arg}}
We now state and prove that the divergence \(d_K\) is continuous with respect to its second argument in the \(L^1\) norm.

\begin{theorem}[Continuity in the second argument]
\label{thm:continuity-second-arg}
Let $p$ and $\{\tilde{p}_n\}_{n \ge 1}$ be continuous densities on $\mathbb{R}$ such that $\tilde{p}_n \to \tilde{p}$ in the $L^1$ norm. Then, the discrepancy function $d_{K}$ is continuous in its second argument, i.e.
\begin{align*}
\lim_{n\to\infty} d_{K}\bigl(p,\tilde{p}_n\bigr)
\;=\;
d_{K}(p,\tilde{p}).
\end{align*}
\end{theorem}

\begin{proof} 
Recall that for any pair of densities \(p,\tilde p\) we define (in Section \ref{section: A rank-based / binomial mapping})
\[
\Phi_K(p,\tilde p)
=\bigl(\mathbb{Q}_K(0),\,\mathbb{Q}_K(1),\,\dots,\,\mathbb{Q}_K(K)\bigr)
\;\in\;\mathbb{R}^{K+1},
\]
and we write the \(m\)th component of $\Phi_{K}(p, \tilde{p})$ as
\[
[\Phi_K(p,\tilde p)]_m \;=\;\mathbb{Q}_K(m).
\]
Thus, if \(\tilde p_n\) has cdf \(\tilde F_n(y)\), then by definition
\[
[\Phi_K(p,\tilde p_n)]_m
=\int_{\mathbb{R}}
\binom{K}{m}\,
\bigl[\tilde F_n(y)\bigr]^m
\bigl[1-\tilde F_n(y)\bigr]^{K-m}
\,p(y)\,\mathrm{d}y.
\]
Hence,
\begin{align*}
\|\Phi_K(p,\tilde p_n) &- \Phi_K(p,\tilde p)\|_{\ell^1}\\
=&\sum_{m=0}^K
\left|
\int_{\mathbb{R}}
\binom{K}{m}\,
\Bigl(
[\tilde{F}_n(y)]^m [1 - \tilde{F}_n(y)]^{K-m}
\;-\;
[\tilde{F}(y)]^m [1 - \tilde{F}(y)]^{K-m}
\Bigr)
p(y)\,\mathrm{d}y
\right|
\end{align*}
and by the triangle inequality we have
\begin{align*}
\|\Phi_K(p,\tilde p_n) &- \Phi_K(p,\tilde p)\|_{\ell^1}
\,\,\\
\le&\sum_{m=0}^K \binom{K}{m}
\int_{\mathbb{R}}
\Bigl|
[\tilde{F}_n(y)]^m [1 - \tilde{F}_n(y)]^{K-m}
-
[\tilde{F}(y)]^m   [1 - \tilde{F}(y)]^{K-m}
\Bigr|
p(y)\,\mathrm{d}y.
\end{align*}
The function 
\begin{align*}
f(a) \;=\; a^m \,(1 - a)^{K-m},
\end{align*}
is Lipschitz on $[0,1]$ with a Lipschitz constant $C_{K,m}<\infty$. As a consequence,
\begin{align*}
\bigl|f(\tilde{F}_n(y)) - f(\tilde{F}(y))\bigr|
\le
C_K\,\bigl|\tilde{F}_n(y) - \tilde{F}(y)\bigr|,
\end{align*}
where $C_{K}=\sup_{0\leq m\leq K}C_{K,m}<\infty$, and therefore,
\begin{align}\label{eq1:thereomA1}
\|\Phi_K(p,\tilde p_n) &- \Phi_K(p,\tilde p)\|_{\ell^1}
\,\le\,
C'_K
\int_{\mathbb{R}}
\bigl|\tilde{F}_n(y) - \tilde{F}(y)\bigr|
\,p(y)\,\mathrm{d}y,
\end{align}
where the constant $C'_K = C_{K}\sum_{m=0}^{K}\binom{K}{m}<\infty$ depends only on $K$.

Note that
\begin{align*}
\tilde{F}_n(y) \;-\; \tilde{F}(y)
\;=\;
\int_{-\infty}^y \bigl(\tilde{p}_n(t) - \tilde{p}(t)\bigr)\,\mathrm{d}t,
\end{align*}
hence
\begin{align*}
\bigl|\tilde{F}_n(y) - \tilde{F}(y)\bigr|
\;\le\;
\int_{\mathbb{R}} \bigl|\tilde{p}_n(t) - \tilde{p}(t)\bigr|\,\mathrm{d}t
\;=\;
\|\tilde{p}_n - \tilde{p}\|_{L^1},
\end{align*}
and, as a consequence,
\begin{align}\label{eq2:thereomA1}
\int_{\mathbb{R}}
\bigl|\tilde{F}_n(y) - \tilde{F}(y)\bigr|
p(y)\,\mathrm{d}y
\;\le\;
\|\tilde{p}_n - \tilde{p}\|_{L^1}
\int_{\mathbb{R}} p(y)\,\mathrm{d}y
\;=\;
\|\tilde{p}_n - \tilde{p}\|_{L^1},
\end{align}
since $\int_{\mathbb{R}} p(y)\,\mathrm{d}y = 1$.

Combining Equation \ref{eq1:thereomA1} and \ref{eq2:thereomA1} we arrive at
\begin{align*}
\|\Phi_K(p,\tilde p_n) &- \Phi_K(p,\tilde p)\|_{\ell^1}
\;\le\;
C'_K \,\|\tilde{p}_n - \tilde{p}\|_{L^1}
\;\longrightarrow\;
0
\quad\text{as } n\to\infty.
\end{align*}

Finally, since
\[
d_K\bigl(p,\tilde{p}_n\bigr)
\;=\;
\frac{1}{K+1}
\;\Bigl\|\Phi_K(p,\tilde p_n) - \mathbb{U}_K\Bigr\|_{\ell^1},
\]
we have
\begin{align}\label{eq3:theoremA1}
\bigl|d_K\bigl(p,\tilde{p}_n\bigr) \;-\; d_K\bigl(p,\tilde{p}\bigr)\bigr|
\;=\;
\frac{1}{K+1}
\bigl|
\|\Phi_K(p,\tilde p_n) - \mathbb{U}_K\|_{\ell^1}
-
\|\Phi_K(p,\tilde p)  - \mathbb{U}_K\|_{\ell^1}
\bigr|.
\end{align}
By the triangle inequality for $\ell^1$,
\begin{align}\label{eq4:theoremA1}
\left|
\|\Phi_K(p,\tilde p_n) - \mathbb{U}_K\|_{\ell^1}
-
\|\Phi_K(p,\tilde p)   - \mathbb{U}_K\|_{\ell^1}
\right|
\;\le\;
\|\Phi_K(p,\tilde p_n) - \Phi_K(p,\tilde p)\|_{\ell^1},
\end{align}
hence taking \ref{eq3:theoremA1} and \ref{eq4:theoremA1} together yields
\[
\bigl|d_K\bigl(p,\tilde{p}_n\bigr) \;-\; d_K\bigl(p,\tilde{p}\bigr)\bigr|
\;\le\;
\frac{1}{K+1} \,\|\Phi_K(p,\tilde p_n) - \Phi_K(p,\tilde p)\|_{\ell^1}.
\]
Since $\|\Phi_K(p,\tilde p_n) - \Phi_K(p,\tilde p)\|_{\ell^1} \to 0$, it follows that
\[
\lim_{n\to\infty} d_K\bigl(p,\tilde{p}_n\bigr)
\;=\;
d_K\bigl(p,\tilde{p}\bigr),
\]
\end{proof}

\subsection{Proof of Theorem \ref{Theorem: Convexity in the First Argument}}
Finally we give the proof of the convexity of $d_{K}$ w.r.t. the first argument.

\begin{proof}[Proof of Theorem \ref{Theorem: Convexity in the First Argument}]
Let \(p_1,p_2\) be two densities and \(\lambda\in[0,1]\).  Define
\[
p(y)=\lambda p_1(y)+(1-\lambda)p_2(y).
\]
Since
\[
[\Phi_K(p,\tilde p)]_n
=\int_{\mathbb{R}}\binom{K}{n}\,[\tilde F(y)]^n[1-\tilde F(y)]^{K-n}\,p(y)\,\mathrm{d} y,
\]
linearity of the integral gives
\[
\Phi_K(p,\tilde p)
=\lambda\,\Phi_K(p_1,\tilde p)+(1-\lambda)\,\Phi_K(p_2,\tilde p).
\]
Hence, for each \(n\),
\begin{eqnarray*}
\left|[\Phi_K(p,\tilde p)]_n-\tfrac1{K+1}\right|
&=&\left|\lambda\bigl([\Phi_K(p_1,\tilde p)]_n-\tfrac1{K+1}\bigr)
+(1-\lambda)\bigl([\Phi_K(p_2,\tilde p)]_n-\tfrac1{K+1}\bigr)\right|\\
&\le&\lambda\,\left|[\Phi_K(p_1,\tilde p)]_n-\tfrac1{K+1}\right|
+(1-\lambda)\,\left|[\Phi_K(p_2,\tilde p)]_n-\tfrac1{K+1}\right|.
\end{eqnarray*}
Summing over \(n=0,\dots,K\) and dividing by \(K+1\) yields
\[
d_K(p,\tilde p)
=\frac1{K+1}\sum_{n=0}^K\bigl|\Phi_K(p,\tilde p)_n-\tfrac1{K+1}\bigr|
\le\lambda\,d_K(p_1,\tilde p)+(1-\lambda)\,d_K(p_2,\tilde p).
\] 
\end{proof}

\section{Proofs of Section \ref{section: A rank-based / binomial mapping}}\label{appendix:proofs section A rank-based / binomial mapping}

In this appendix we give complete proofs for the theorems and claims in Sections \ref{section: A rank-based / binomial mapping}.
\begin{itemize}
  \item Theorem \ref{thm:PhiK_Properties}. Characterizes the binomial mapping 
    \(\Phi_{K}\), showing it is well-defined, linear in its first argument, bounded, and continuous under mild regularity conditions.
  \item Theorem \ref{theorem:Riesz Representation}. Establishes the Riesz representation of \(\Phi_{K}\), expressing each probability mass $\mathbb{Q}_{K}(n)$ as an \(L^{2}\) inner product with a Bernstein basis function.
  \item Theorem \ref{lemma: uniform boudn for q_K convergence to 1}. Provides a uniform bound on the deviation \(\|q_{K}-1\|_{\infty}\) in terms of the discrepancy \(d_{K}(p,\tilde p)\).
  \item Theorem \ref{theorem:general_forward_pushforward_expanded}. Derives an explicit Bernstein-based representation for a push-forward density \(p_K\) via a continuously differentiable map.
\end{itemize}

\subsection{Properties of the map \(\Phi_{K}\)}
\begin{theorem}
\label{thm:PhiK_Properties}
Let $p$ and $\tilde{p}$ be pdfs on $\mathbb{R}$, with cdfs $F$ and $\tilde{F}$, respectively. For each integer $K \ge 1$, recall that 
\[
\Phi_K(p,\tilde p)
=\bigl(\mathbb{Q}_K(0),\,\mathbb{Q}_K(1),\,\dots,\,\mathbb{Q}_K(K)\bigr)
\;\in\;\mathbb{R}^{K+1},
\]
where
\begin{align}\label{eq1:Properties of the map}
\mathbb{Q}_K(n)
\;=\;
\int_{\mathbb{R}}
\binom{K}{n}\,
\bigl[\tilde{F}(y)\bigr]^n
\bigl[1-\tilde{F}(y)\bigr]^{K-n}
\,p(y)\,\mathrm{d}y,
\quad n=0,\ldots,K.
\end{align}
Then the following properties hold
\begin{enumerate}[label=(\roman*)]
    \item \emph{Well-definedness}.
    For each fixed pair \((p,\tilde{p})\) and integer \(K\), the integral on the right -hand side of Equation \ref{eq1:Properties of the map} uniquely determines a pmf \(\mathbb{Q}_{K}\).

    \item \emph{Non-surjectivity}.
    Let $\Delta^K$ be the set of all pmfs on $\{0,1,\ldots,K\}$. The image of $\Phi_{K}$ is strictly contained in $\Delta^K$.

    \item \emph{Continuity}.
    Assume that $\|p_n - p\|_{L^1(\mathbb{R})}\;\rightarrow\;0$
    and $\|\tilde F_n - \tilde F\|_{\infty}\;\rightarrow\;0$.
    Then
    \[
    \bigl\|\Phi_K\bigl(p_n,\tilde p_n\bigr)\;-\;\Phi_K\bigl(p,\tilde p\bigr)\bigr\|_{\ell^1}
    \;\longrightarrow\;0.
    \]

    \item \emph{Linearity}. The operator $\Phi_{K}$ is linear in its first argument. Thus, for any $\alpha \in [0,1]$,
    \begin{align*}
    \Phi_{K}\big(\alpha p_{1} + (1 - \alpha) p_{2}, \tilde{p}\big) = \alpha\, \Phi_{K}(p_{1}, \tilde{p}) + (1 - \alpha)\, \Phi_{K}(p_{2}, \tilde{p}).
    \end{align*}

    \item \emph{Bounded operator}. $\Phi_{K}$ is a bounded operator from the space of continuous pdfs \( p \) on \( \mathbb{R} \) to the space of pmfs \( \mathbb{Q}_K \) on \( \{0, 1, \ldots, K\} \). Specifically
    \begin{align*}
    \triplenorm{\Phi_{K}} = \sup_{\|p\|_{L^{1}}\leq 1}\|\Phi_{K}(p, \tilde{p})\|_{\text{TV}} = \|\mathbb{Q}_{K}\|_{TV}=1,
    \end{align*}
    where $\triplenorm{\cdot}$ denotes the operator norm and $\|\cdot\|_{\text{TV}}$ the total variation norm.
\end{enumerate}
\end{theorem}

\begin{proof}[Proof of Theorem \ref{thm:PhiK_Properties}]
  \smallskip           
  \noindent
\begin{enumerate}[label=(\roman*)]
    \item Fix \(K\ge0\) and recall that
        \[
        \mathbb{Q}_K(n)
        =\int_{\mathbb{R}}\binom{K}{n}\,[\tilde F(y)]^n\,[1-\tilde F(y)]^{K-n}\,p(y)\,\mathrm{d} y,
        \qquad n=0,\dots,K.
        \]
        
        \emph{Non‐negativity.}  Since \(p(y)\ge0\) and each Bernstein‐integrand 
        \(\binom{K}{n}[\tilde F(y)]^n[1-\tilde F(y)]^{K-n}\ge0\),
        it follows that \(\mathbb{Q}_K(n)\ge0\) for every \(n\).
        
        \emph{Normalization.}  By Fubini’s theorem we may interchange sum and integral, i.e.,
        \[
        \sum_{n=0}^K \mathbb{Q}_K(n)
        =\int_{\mathbb{R}}p(y)\sum_{n=0}^K\binom{K}{n}[\tilde F(y)]^n[1-\tilde F(y)]^{K-n}\,\mathrm{d} y.
        \]
        But the inner sum is \((\tilde F(y)+(1-\tilde F(y)))^K=1^K=1\) by the binomial theorem, hence
        \[
        \sum_{n=0}^K \mathbb{Q}_K(n)
        =\int_{\mathbb{R}}p(y)\,\mathrm{d} y
        =1.
        \]
        
        \emph{Uniqueness.}  Each \(\mathbb{Q}_K(n)\) is defined by a single integral depending only on \(p\) and \(\tilde F\).  Thus the mapping 
        \(\Phi_K:(p,\tilde p)\mapsto(\mathbb{Q}_K(0),\ldots,\mathbb{Q}_K(K))\)
        is well‐defined and unique for each choice of \((p,\tilde p)\) and \(K\).
    \item To see that \(\Phi_K\) is not surjective, it suffices to exhibit a pmf in \(\Delta^K\) that cannot arise from any \((p,\tilde p)\).  We do this for \(K=2\).

    Define
    \[
    \mathbb Q = (0,1,0)\in\Delta^2,
    \]
    so that \(\mathbb Q(0)=0\), \(\mathbb Q(1)=1\), and \(\mathbb Q(2)=0\).  Suppose, for the sake of contradiction, that there exist densities \(p,\tilde p\) on \(\mathbb{R}\) with cdf \(\tilde F\) such that
    \(\Phi_2(p,\tilde p)=\mathbb Q\).  Then by definition
    \begin{align*}
    \mathbb Q(0)
    &=\int_\mathbb{R} [1-\tilde F(y)]^2\,p(y)\,\mathrm{d} y = 0,\\
    \mathbb Q(1)
    &=2\int_\mathbb{R} \tilde F(y)\,[1-\tilde F(y)]\,p(y)\,\mathrm{d} y = 1,\\
    \mathbb Q(2)
    &=\int_\mathbb{R} [\tilde F(y)]^2\,p(y)\,\mathrm{d} y = 0.
    \end{align*}
    The first and third equations force
    \[
    [1-\tilde F(y)]^2=0
    \quad\text{and}\quad
    [\tilde F(y)]^2=0
    \quad
    \text{\(p\)-almost everywhere,}
    \]
    hence \(\tilde F(y)=1\) and \(\tilde F(y)=0\) \(p\)-a.e.  Since \(p\) is a probability density, its support has positive measure, so we cannot have \(\tilde F\equiv1\) and \(\tilde F\equiv0\) on that support.  This contradiction shows no \((p,\tilde p)\) can produce \(\mathbb Q=(0,1,0)\).  Therefore \(\Phi_2\), and hence \(\Phi_K\) for general \(K\), fails to be surjective.

    \item Let \(p_i\to p\) in \(L^1(\mathbb{R})\) and let \(\tilde p_j\) have cdfs \(\tilde F_j\to\tilde F\) uniformly on $\mathbb{R}$.  We show
    \[
    \bigl\|\Phi_K(p_n,\tilde p_n)-\Phi_K(p,\tilde p)\bigr\|_{1}
    \;\to\;0.
    \]
    By the triangle inequality,
    \begin{align}\label{eq1:proof continuity of phi}
    \nonumber\bigl\|\Phi_K(p_i,\tilde p_n)-\Phi_K(p,\tilde p)\bigr\|_{\ell^1}
    &\le 
    \bigl\|\Phi_K(p_i,\tilde p_n)-\Phi_K(p_i,\tilde p)\bigr\|_{\ell^1}\\
    &\quad 
    +\;\bigl\|\Phi_K(p_i,\tilde p)-\Phi_K(p,\tilde p)\bigr\|_{\ell^1},
    \end{align}
    and we handle each term separately.
    
    \medskip\noindent
    \textbf{Continuity in \(\tilde p\).}  Fix \(p_i\).  For each \(m=0,\dots,K\), set
    \[
    g_j(y)
    =[\tilde F_j(y)]^m\,[1-\tilde F_j(y)]^{K-m},
    \quad
    g(y)
    =[\tilde F(y)]^m\,[1-\tilde F(y)]^{K-m}.
    \]
    Uniform convergence \(\|\tilde F_j-\tilde F\|_\infty\to0\) implies \(\|g_j-g\|_\infty\to0\).  Hence for each coordinate $m$ we have
    \[
    \bigl|\Phi_K(p_i,\tilde p_n)_m-\Phi_K(p_i,\tilde p)_m\bigr|
    \le\binom Km\int_{\mathbb{R}}|g_i(y)-g(y)|\,p_i(y)\,\mathrm{d} y
    \le\binom Km\|g_n-g\|_\infty,
    \]
    and summing over \(m\) yields
    \begin{align}\label{eq2:proof continuity of phi}
    \bigl\|\Phi_K(p_i,\tilde p_j)-\Phi_K(p_i,\tilde p)\bigr\|_{\ell^1}
    \;\le\;
    \sum_{m=0}^K\binom Km\|g_j-g\|_\infty
    =2^K\,\|g_j-g\|_\infty
    \;\longrightarrow\;0.
    \end{align}
    
    \medskip\noindent
    \textbf{Continuity in \(p\).}  
    Fix \(\tilde p\) (and write \(\tilde F\) for its cdf).  For each \(m=0,\dots,K\),
    \[
    \bigl[\Phi_K(p_i,\tilde p)\bigr]_m
    -\bigl[\Phi_K(p,\tilde p)\bigr]_m
    =\binom Km
    \int_{\mathbb{R}}
    [\tilde F(y)]^m[1-\tilde F(y)]^{K-m}
    \bigl(p_i(y)-p(y)\bigr)\,\mathrm{d} y.
    \]
    Taking absolute values and using \(\int|p_i-p|=\|p_i-p\|_{L^1}\) yields
    \[
    \bigl|\Phi_K(p_i,\tilde p)_m-\Phi_K(p,\tilde p)_m\bigr|
    \;\le\;\binom Km\,
    \|p_i-p\|_{L^1}.
    \]
    Summing over \(m\) gives
    \begin{align}\label{eq3:proof continuity of phi}
    \nonumber \bigl\|\Phi_K(p_i,\tilde p)-\Phi_K(p,\tilde p)\bigr\|_{\ell^1}
    =&\sum_{m=0}^K\bigl|\Phi_K(p_i,\tilde p)_m-\Phi_K(p,\tilde p)_m\bigr|
    \\\;\le\;&
    \nonumber \Bigl(\sum_{m=0}^K\binom Km\Bigr)
    \|p_i-p\|_{L^1}
    \\\;=\;&2^K\,\|p_i-p\|_{L^1}\longrightarrow 0.
    \end{align}
    Since \(p_i\to p\) in \(L^1\), the right‐hand side tends to zero.  Hence \(\Phi_K(\cdot,\tilde p)\) is continuous in its first argument. 

    \medskip\noindent
    Combining in Equations \ref{eq1:proof continuity of phi}, \ref{eq2:proof continuity of phi} and \ref{eq3:proof continuity of phi} yields
    \[
    \bigl\|\Phi_K(p_i,\tilde p_j)-\Phi_K(p,\tilde p)\bigr\|_{\ell^1}
    \;\longrightarrow\;0,
    \]
    i.e.\ \(\Phi_K\) is jointly continuous in \((p,\tilde p)\).

    \item  
    \medskip\noindent
    \textbf{Linearity in \(p\).}  
    Let \(p_1,p_2\) be two probability densities on \(\mathbb{R}\) and \(\alpha\in[0,1]\).  Set
    \[
    p=\alpha\,p_1+(1-\alpha)\,p_2.
    \]
    Then for each \(n=0,\dots,K\),
    \[
    \Phi_K(p,\tilde p)_n
    =\binom{K}{n}
    \int_{\mathbb{R}}[\tilde F(y)]^n\,[1-\tilde F(y)]^{K-n}\,p(y)\,\mathrm{d} y.
    \]
    By linearity of the integral,
    \begin{align*}
    \Phi_K(p,\tilde p)_n
    &=\binom{K}{n}\!\int_{\mathbb{R}}[\tilde F]^n[1-\tilde F]^{K-n}
    \bigl(\alpha p_1+(1-\alpha)p_2\bigr)\,\mathrm{d} y\\
    &=\alpha\,\Phi_K(p_1,\tilde p)_n+(1-\alpha)\,\Phi_K(p_2,\tilde p)_n.
    \end{align*}
    Since this holds for every coordinate \(n\), we conclude
    \[
    \Phi_K\bigl(\alpha p_1+(1-\alpha)p_2,\;\tilde p\bigr)
    =\alpha\,\Phi_K(p_1,\tilde p)+(1-\alpha)\,\Phi_K(p_2,\tilde p),
    \]
    i.e.\ \(\Phi_K(\cdot,\tilde p)\) is linear.
    
    \item Boundedness and operator norm.  
    Recall the total‐variation norm on \(\mathbb{R}^{K+1}\) is just the \(\ell^1\)‐norm.  For any pair of densities \(p, \tilde{p}\),
    \[
    \|\Phi_K(p,\tilde p)\|_{\mathrm{TV}}
    =\sum_{n=0}^K[\Phi_K(p,\tilde p)]_n
    =\int_{\mathbb{R}}p(y)\sum_{n=0}^K\binom{K}{n}[\tilde F(y)]^n[1-\tilde F(y)]^{K-n}\,\mathrm{d} y.
    \]
    By the binomial theorem the inner sum is \(\bigl(\tilde F+(1-\tilde F)\bigr)^K=1\).  Hence
    \[
    \|\Phi_K(p,\tilde p)\|_{\mathrm{TV}}
    =\int_{\mathbb{R}}p(y)\,\mathrm{d} y
    =\|p\|_{L^1}.
    \]
    Taking the supremum over all \(p\) with \(\|p\|_{L^1}\le1\) shows
    \[
    \triplenorm{\Phi_K} := \sup_{\|p\|_{L^1}\le1}\|\Phi_K(p,\tilde p)\|_{\mathrm{TV}} = \sum_{n=0}^{K} \mathbb{Q}_{K}(n)
    =1.
    \]
    Thus \(\Phi_K\) is a bounded linear operator with \(\|\Phi_K\|=1\).
    \end{enumerate}
\end{proof}

\subsection{Proof of Theorem \ref{theorem:Riesz Representation}}
\begin{proof}[Proof of Theorem \ref{theorem:Riesz Representation}]
  Recall that, by Theorem \ref{thm:PhiK_Properties}, for each fixed $\tilde p$, the map
  \[
    \Phi_K(\cdot,\tilde p):C(\mathbb{R})\;\longrightarrow\;\mathbb{R}^{K+1}
  \]
  is linear.  To exhibit its Riesz representation, it suffices to find, for each $n=0,1,\dots,K$, a function $f_n(y)\in L^{2}(\mathbb{R})$ such that
  \[
    [\Phi_K(p,\tilde p)]_n
    \;=\;
    \mathbb{Q}_K(n)
    \;=\;
    \int_{\mathbb{R}} f_n(y)\,p(y)\,\mathrm{d}y
    \;=\;
    \bigl\langle f_n,\,p\bigr\rangle_{L^2(\mathbb{R})}.
  \]

  However, by definition of $\Phi_K$,
  \[
    \mathbb{Q}_K(n)
    \;=\;
    \binom{K}{n}
    \int_{\mathbb{R}}
      \bigl[\tilde F(y)\bigr]^n\,\bigl[1-\tilde F(y)\bigr]^{K-n}
      \,p(y)\,\mathrm{d}y=\int_{\mathbb{R}}\left(b_{n,K}\circ \tilde{F}\right)(y)\,p(y)\,\mathrm{d}y,
  \]
  which is exactly $\langle\,b_{n,K}\circ\tilde F,\;p\rangle$ for the Bernstein polynomial of degree $K$. Therefore
  \[
    \Phi_K(p,\tilde p)
    \;=\;
    \bigl(
      \langle f_0,p\rangle,\,
      \langle f_1,p\rangle,\,
      \dots,\,
      \langle f_K,p\rangle
    \bigr),
  \]
  and the Riesz theorem yields the claimed representation, with
  \(\{f_n\}_{n=0}^K\) playing the role of the dual elements.

  \medskip\noindent
  Now assume $\tilde p(x)>0$ for all $x$.  Then $\tilde F'(x)=\tilde p(x)>0$. So $\tilde F$ is continuous and strictly increasing.  Hence $\tilde F$ is a bijection and admits the inverse $\tilde F^{-1}$.

  Setting $q(x)=p(x)/\tilde p(x)$ and changing variables $t=\tilde F(x)$ (so $x=\tilde F^{-1}(t)$ and $\mathrm{d} t=\tilde p(\tilde F^{-1}(t))\,\mathrm{d} x$) yields
  \[
    \mathbb{Q}_K(n)
    = \int_{\mathbb{R}} b_{n,K}\bigl(\tilde F(x)\bigr)\,p(x)\,\mathrm{d} x
    = \int_{0}^{1} b_{n,K}(t)\,q\bigl(\tilde F^{-1}(t)\bigr)\,\mathrm{d} t
    = \bigl\langle b_{n,K},\,q\circ\tilde F^{-1}\bigr\rangle_{L^2([0,1])}.
  \]
  This completes the proof.
\end{proof}

\subsection{Proof of Theorem \ref{thm:bernstein_truncation_ratio}}

\begin{proof}[Proof of Theorem \ref{thm:bernstein_truncation_ratio}]

\medskip\noindent  
\textbf{Existence of the Bernstein expansion.}  
Since \(p,\tilde p\in C(\mathbb{R})\) and \(\tilde p(x)>0\) for all \(x\), the ratio $q(x)\;=\;\frac{p(x)}{\tilde p(x)}$ is well‐defined and continuous on \(\mathbb{R}\).  Moreover, the cdf $\tilde{F}$ is \(C^1\) with \(\tilde F'(x)=\tilde p(x)>0\); hence \(\tilde F\) is strictly increasing and continuous.  It follows that
\(\tilde F^{-1}\) exists and is continuous .  Therefore, setting
\[
\tilde q(t)\;=\;q\bigl(\tilde F^{-1}(t)\bigr),
\]
we obtain \(\tilde q\in C([0,1])\).

Since \(\{b_{n,K}\}_{n\ge0}\) is a Schauder basis of \(C([0,1])\) (see \cite[Theorem 1.1.1]{lorentz2012bernstein} and \cite[Chapter 4]{megginson2012introduction}), there are unique coefficients \(\{\alpha_n\}_{n\ge0}\) such that
\[
\tilde q(t)
=\sum_{n=0}^\infty \alpha_n\,b_{n,n}(t),
\qquad t\in[0,1].
\]
Truncating at degree \(K\) yields
\[
\tilde q_K(t)
:=\sum_{n=0}^K \alpha_n\,b_{n,K}(t).
\]

\medskip\noindent  
\textbf{Dual‐basis (projection) representation.} Let 
\[
V_K=\mathrm{span}\{b_{0,K},\dots,b_{K,K}\}\subset L^2([0,1]),
\]
and define the Gram matrix $G$ with entries
\[
G_{n,m}
=\bigl\langle b_{n,K},\,b_{m,K}\bigr\rangle
=\int_0^1 b_{n,K}(t)\,b_{m,K}(t)\,\mathrm{d}t, \qquad n,m\in \{0,\ldots,K\},
\]
which is nonsingular (positive-definite).  The dual (biorthogonal) basis \(\{\tilde b_{n,K}\}\subset V_K\) is given by (see \cite[Section 2]{jiittler1998dual})
\[
\tilde b_{n,K}(t)
=\sum_{m=0}^K (G^{-1})_{n,m}\,b_{m,K}(t),
\]
so that \(\langle\tilde b_{n,K},b_{m,K}\rangle=\delta_{n,m}\).

Let \(P_K:L^2([0,1])\to V_K\) the orthogonal projection constructed as \(P_K(f)\in V_K\) that satisfies \(f-P_K(f)\perp V_K\). We write
\[
P_K(f)=\sum_{n=0}^K c_n\,\tilde b_{n,K},
\]
and impose for each $m$,
\[
0
=\bigl\langle f-P_K(f),\,b_{m,K}\bigr\rangle =\bigl\langle f,\,b_{m,K}\bigr\rangle
          - \sum_{n=0}^K c_n\,\bigl\langle \tilde b_{n,K},\,b_{m,K}\bigr\rangle
=\bigl\langle f,\,b_{m,K}\bigr\rangle 
- c_m,
\]
where \(c_m=\langle f,b_{m,K}\rangle\).  Thus
\[
P_K(f)
=\sum_{n=0}^K \bigl\langle f,b_{n,K}\bigr\rangle_{L^2([0,1])}\,\tilde b_{n,K}.
\]
Applying this to \(f=\tilde q\) gives
\[
\tilde q_K
=P_K(\tilde q)
=\sum_{n=0}^K 
\bigl\langle \tilde q,\,b_{n,K}\bigr\rangle_{L^2([0,1])}
\,\tilde b_{n,K}(t).
\]

\medskip\noindent  
\textbf{Identification of the coefficients.}  
Since \(\tilde q(t)=q(\tilde F^{-1}(t))\) and \(q(x)=p(x)/\tilde p(x)\), the change of variables \(t=\tilde F(x)\) yields
\begin{align*}
\bigl\langle \tilde q,\,b_{n,K}\bigr\rangle
=\int_0^1 b_{n,K}(t)\,\tilde q(t)\,\mathrm{d}t
&=\int_\mathbb{R} b_{n,K}\bigl(\tilde F(x)\bigr)\,q(x)\,\tilde p(x)\,\mathrm{d}x
\\&=\int_\mathbb{R} b_{n,K}\bigl(\tilde F(x)\bigr)\,p(x)\,\mathrm{d}x
=\mathbb{Q}_K(n).
\end{align*}
Hence
\[
\tilde q_K(t)
=\sum_{n=0}^K \mathbb{Q}_K(n)\,\tilde b_{n,K}(t).
\]
\end{proof}

\subsection{Proof Theorem \ref{lemma: uniform boudn for q_K convergence to 1}}

In this section we derive a uniform bound on the Bernstein–truncated ratio \(q_K\).  The main tool is the dual‐basis expansion of \(\tilde q_K\) and the fact (Lemma \ref{lemma:dual_basis_sum}) that the dual Bernstein functions sum to \(K+1\).

\begin{lemma}
\label{lemma:dual_basis_sum}
The dual Bernstein basis functions \(\tilde{b}_{m,K}(x)\) satisfy
\begin{align*}
\sum_{m=0}^{K}\tilde{b}_{m,K}(x)=K+1.
\end{align*}
\end{lemma}


\begin{proof}
We begin by showing that the constant vector 
\[
e = (1,1,\dots,1)^\top
\]
is an eigenvector of the Gram matrix $G$ with eigenvalue $\lambda = 1/(K+1)$.  Writing out the $n$th component of $G\,e$ yields
\[
(G\,e)_n
= \sum_{m=0}^K G_{n m}\,e_m
= \sum_{m=0}^K \int_0^1 b_{n,K}(x)\,b_{m,K}(x)\,\mathrm{d}x
= \int_0^1 b_{n,K}(x)\Bigl(\sum_{m=0}^K b_{m,K}(x)\Bigr)\mathrm{d}x.
\]
Since the Bernstein polynomials form a partition of unity, $\sum_{m=0}^K b_{m,K}(x)=1$ (by the Binomial theorem), and each integrates to $1/(K+1)$ (indeed, 
$\int_{0}^{1}b_{n,K}(x)\,\mathrm{d}x =\binom{K}{n}B(n+1,K-n+1) =\frac{1}{K+1}\!$ by the Beta‐function identity), we obtain
\[
(G\,e)_n = \int_0^1 b_{n,K}(x)\,\mathrm{d}x = \frac1{K+1}.
\]
Hence
\begin{align}\label{eq1:lemma truncation}
G\,e = \frac1{K+1}\,e,
\end{align}
i.e.\ $e$ is an eigenvector with eigenvalue $1/(K+1)$.

Left-multiplying by $G^{-1}$ (positive-definite matrix) on both sides of Equation \ref{eq1:lemma truncation} yields
\[
G^{-1}e = (K+1)\,e.
\]
Which implies that
\[
(G^{-1}e)_{n}=\sum_{m=0}^K (G^{-1})_{n,m} = K+1
\quad\text{for each }n.
\]

Finally, by the definition of the dual basis,
\begin{eqnarray*}
\sum_{m=0}^K \tilde b_{m,K}(x)
&=& \sum_{m=0}^K \sum_{n=0}^K (G^{-1})_{n,m}\,b_{n,K}(x)\\
&=& \sum_{n=0}^K b_{n,K}(x)\,\Bigl(\sum_{m=0}^K (G^{-1})_{n m}\Bigr)\\
&=& (K+1)\sum_{n=0}^K b_{n,K}(x)\\
&=& K+1.
\end{eqnarray*}
where we again are used the partition‐of‐unity property $\sum_{n=0}^K b_{n,K}(x)=1$.
\end{proof}

\begin{proof}[Proof of Theorem \ref{lemma: uniform boudn for q_K convergence to 1}]
We have
\begin{align*}
\tilde{q}_K(t)=\sum_{m=0}^{K}\mathbb{Q}_K(m)\tilde{b}_{m,K}(x)=\sum_{m=0}^{K}\left[\frac{1}{K+1}-\left(\frac{1}{K+1}-\mathbb{Q}_K(m)\right)\right]\tilde{b}_{m,K}(t).
\end{align*}
Since \(\sum_{m=0}^{K}\tilde{b}_{m,K}(x)=K+1\) (see Lemma \ref{lemma:dual_basis_sum}), we rewrite this explicitly as
\begin{align*}
\tilde{q}_K(t)=1-\sum_{m=0}^{K}\left(\frac{1}{K+1}-\mathbb{Q}_K(m)\right)\tilde{b}_{m,K}(t).
\end{align*}
By the triangle inequality, we have
\begin{align*}
|\tilde{q}_K(t)-1|\leq \sum_{m=0}^{K}\left|\frac{1}{K+1}-\mathbb{Q}_K(m)\right|\;|\tilde{b}_{m,K}(x)| \leq (K+1)\sum_{m=0}^{K}\left|\frac{1}{K+1}-\mathbb{Q}_K(m)\right|,
\end{align*}
where the inequality \(|\tilde{b}_{m,K}(x)|\leq K+1\) follows from Lemma \ref{lemma:dual_basis_sum}. Recognizing the definition of \(d_{K}(p,\tilde{p})\), we have
\begin{align*}
|\tilde{q}_K(t)-1|\leq (K+1)^{2}d_{K}(p,\tilde{p}).
\end{align*}
Since \(q_K(x)=\tilde q_K\bigl(\tilde F(x)\bigr)\), the supremum over \(x\in\mathbb{R}\) coincides with the supremum over \(t\in[0,1]\).  Hence
\begin{align*}
    \|q_{K}(x)-1\|_{\infty} \leq (K+1)^{2}d_{K}(p,\tilde{p}).
\end{align*}
\end{proof}

\subsection{Proof of Theorem \ref{theorem:general_forward_pushforward_expanded}}

\begin{proof}[Proof of Theorem \ref{theorem:general_forward_pushforward_expanded}] 
Recall from Theorem \ref{thm:bernstein_truncation_ratio} that
\begin{align}\label{eq1:theorem truncation}
\tilde q_K(t)
=\sum_{m=0}^K\mathbb Q_K(m)\,\tilde b_{m,K}(t)
\end{align}
is precisely the degree-\(K\) truncation of the Bernstein expansion of \(\tilde q\).  Since \(\{b_{n,K}\}_{n\ge0}\) is a Schauder basis of \(\bigl(C([0,1]),\|\cdot\|_\infty\bigr)\), these truncations satisfy
\[
\|\tilde q_K - \tilde q\|_\infty
=\sup_{t\in[0,1]}|\tilde q_K(t)-\tilde q(t)|
\;\longrightarrow\;0
\qquad \text{when }\;K\to\infty ,
\]
i.e.\ \(\tilde q_K\to\tilde q\) uniformly on \([0,1]\) (see \cite[Section 1.1]{lorentz2012bernstein}). From Equation. \eqref{eq:pK definition} and \ref{eq1:theorem truncation},
\[
p_K(x)
=\tilde p(x)\,\tilde q_K\bigl(\tilde F(x)\bigr).
\]
Since \(\tilde p\) and \(\tilde F\) are continuous, for each fixed \(x\) the argument 
\(t_x=\tilde F(x)\in[0,1]\) is constant, and uniform convergence of \(\tilde q_K\) gives
\[
\lim_{K\to\infty}\tilde q_K\bigl(t_x\bigr)
=\tilde q\bigl(t_x\bigr).
\]
Hence
\[
\lim_{K\to\infty}p_K(x)
=\tilde p(x)\,\lim_{K\to\infty}\tilde q_K(\tilde F(x))
=\tilde p(x)\,\tilde q\bigl(\tilde F(x)\bigr)
=\tilde p(x)\,\frac{p(x)}{\tilde p(x)}
= p(x).
\]
\end{proof}

\section{Proofs of Section \ref{section:Sliced Multivariate ISL via Bernstein Polynomial Approximation}} \label{section: Multidimensional Extension of the Projection Schema}


In this subsection we prove Theorem \ref{thm:qs_convergence}, which shows that for each projection direction \(s\), the Bernstein‐projected density ratio \(q^s\) converges uniformly to 1 at the rate controlled by \(d_K(p,\tilde p)\) and the Hessian of \(q\).    

\begin{proof}[Proof of Theorem \ref{thm:qs_convergence}] 

    Integrating over $\mathbb{S}^{d}$ on both the left-hand and right hand sides of expression \eqref{eq:convergence of qK to 1} yields
    \begin{align*}
    \int_{\mathbb{S}^{d}} \|q^s - 1\|_\infty\, ds \le (K+1)^2 \int_{\mathbb{S}^{d}} d_K(s\#p, s\#\tilde{p})\, ds + \frac{1}{8K}\int_{\mathbb{S}^{d}} \|(q^s)''\|_\infty\, ds.
    \end{align*}
    Since $\tilde{p}(x)>0$ for any $x\in\mathbb{R}$ then \(q\in C^2(\mathbb{R}^d)\). Note that for any unit vector \(v\), the second directional derivative is \(D^2_{v,v}q(x) = v^\top \nabla^2q(x)v\), and hence \(|D^2_{v,v}q(x)| \le \|\nabla^2q(x)\|\). Taking the supremum over \(x\) yields the uniform bound \(\|\nabla^2q\|_\infty\).  Therefore, we can further write
    \begin{align*}
    \int_{\mathbb{S}^{d}} \|q^s - 1\|_\infty\, \mathrm{d}s \le (K+1)^2\, d_K^{\mathbb{S}^{d}}(p,\tilde{p}) + \frac{\|\nabla^2 q\|_\infty}{8K}\, \mathcal{L}(\mathbb{S}^{d}),
    \end{align*}
    where $\mathcal{L}(\mathbb{S}^{d})$ denotes the Lebesgue measure of $\mathbb{S}^{d}$ and $\|\nabla^2 q\|_\infty$ denotes the supremum over $[0,1]^d$ of the operator norm of the Hessian of $q$.
    Similarly, we obtain
    \begin{align*}
    \sup_{s\in \mathbb{S}^{d}} \left\|q^{s}-1\right\|_{\infty} \le (K+1)^{2} \, \sup_{s\in \mathbb{S}^{d}} d_{K}(s\#p, s\#\tilde{p}) + \frac{\|\nabla^{2} q\|_{\infty}}{8K}.
    \end{align*}

\end{proof}

\newpage

\section{Supplementary experiments} \label{Supplementary experiments}

\subsection{Evaluating dual-ISL on 1D Target Distributions} \label{appendix: 1D experiments}

Following the evaluation setup of \cite{zaheer2017gan,de2024training}, we draw \(N=1000\) i.i.d.\ samples from each of six one‐dimensional benchmark targets. The first three targets are classical pdf (e.g.\ \(\mathcal{N}(0,1)\), Cauchy, and Pareto), while the remaining three are equally‐weighted mixtures:
\begin{itemize}
  \item \textbf{Model\textsubscript{1}:} \(\tfrac12\mathcal{N}(5,2) + \tfrac12\mathcal{N}(-1,1)\).
  \item \textbf{Model\textsubscript{2}:} \(\tfrac13\mathcal{N}(5,2) + \tfrac13\mathcal{N}(-1,1) + \tfrac13\mathcal{N}(-10,3)\).
  \item \textbf{Model\textsubscript{3}:} \(\tfrac12\mathcal{N}(-5,2) + \tfrac12\mathrm{Pareto}(5,1)\).
\end{itemize}

All non-diffusion methods (Dual-ISL, ISL, GAN, WGAN and MMD-GAN) use the same generator architecture: a four-layer MLP with ELU activations and layer widths $[7,\,13,\,7,\,1]$. Each is trained for $10^4$ epochs using Adam with a fixed learning rate of $10^{-2}$. By contrast, the DDPM baseline employs a four-layer ELU-MLP score network with identical widths, augmented by a 16-dimensional sinusoidal time embedding. It is also trained for $10^4$ epochs with Adam (lr = $10^{-2}$) across $T=200$ diffusion steps, where the noise schedule $\{\beta_t\}$ is linearly spaced from $10^{-4}$ to $2\times10^{-2}$. Table \ref{paper Learning1D_table} summarizes the quantitative results, and the corresponding visualizations are shown in Figure \ref{fig:1d-distributions}.

\begin{figure}[!htbp]
  \centering
  \setlength{\tabcolsep}{4pt}
  \renewcommand{\arraystretch}{1.05}
  \begin{tabular}{@{}r|cccccc@{}}
    \toprule
    \rowcolor{gray!30}
    & \textbf{dual-ISL}
    & \textbf{ISL}
    & \textbf{WGAN}
    & \textbf{MMD-GAN}
    & \textbf{Diffusion}
    \\
    \midrule
    \rotatebox{90}{$\normdist{4}{2}$}
      & \includegraphics[width=2.2cm,height=2.2cm]{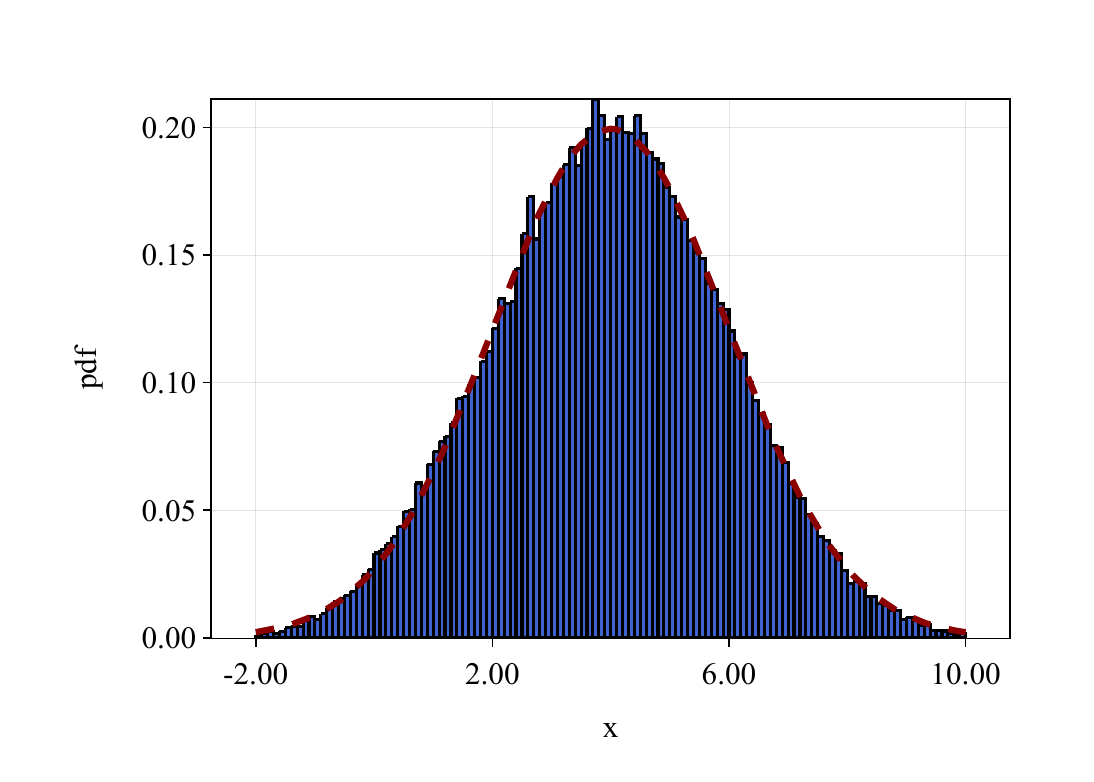}
      & \includegraphics[width=2.2cm,height=2.2cm]{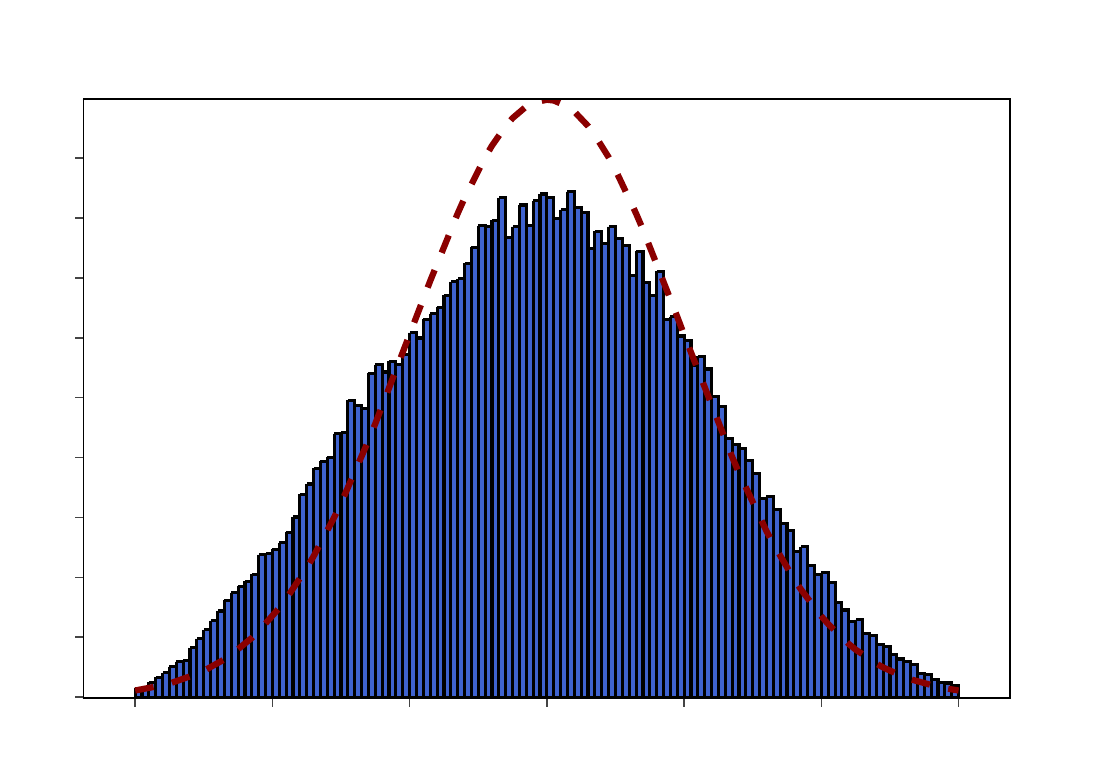}
      & \includegraphics[width=2.2cm,height=2.2cm]{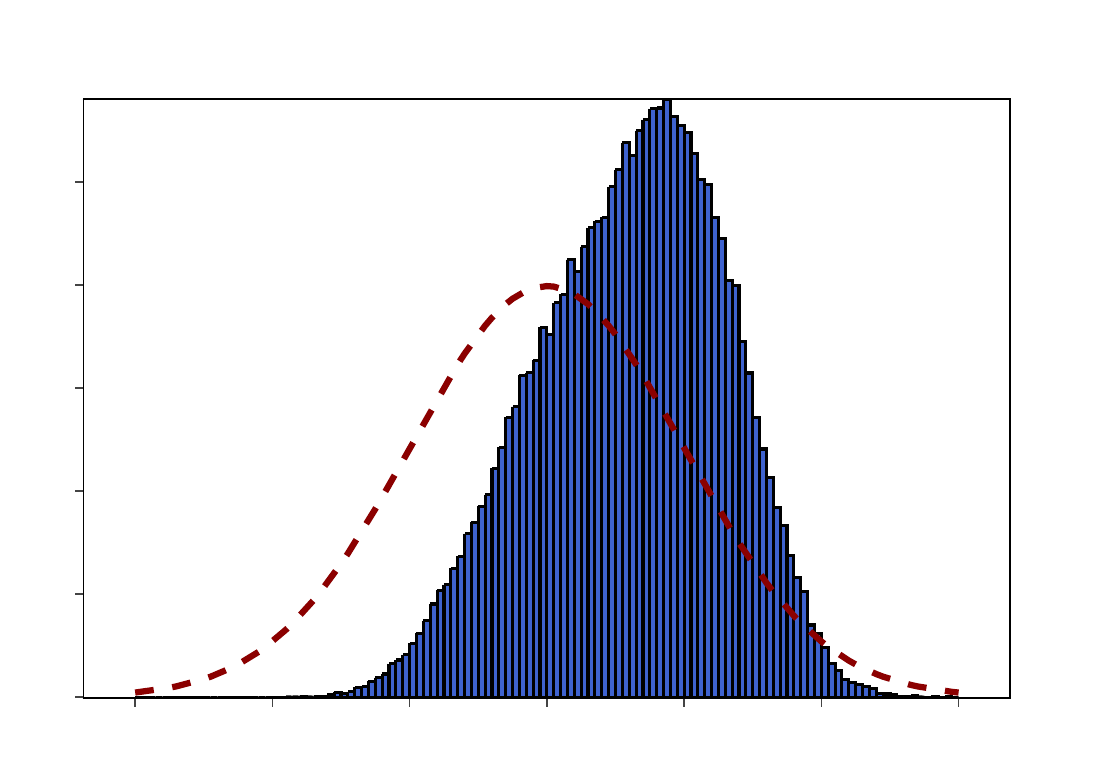}
      & \includegraphics[width=2.2cm,height=2.2cm]{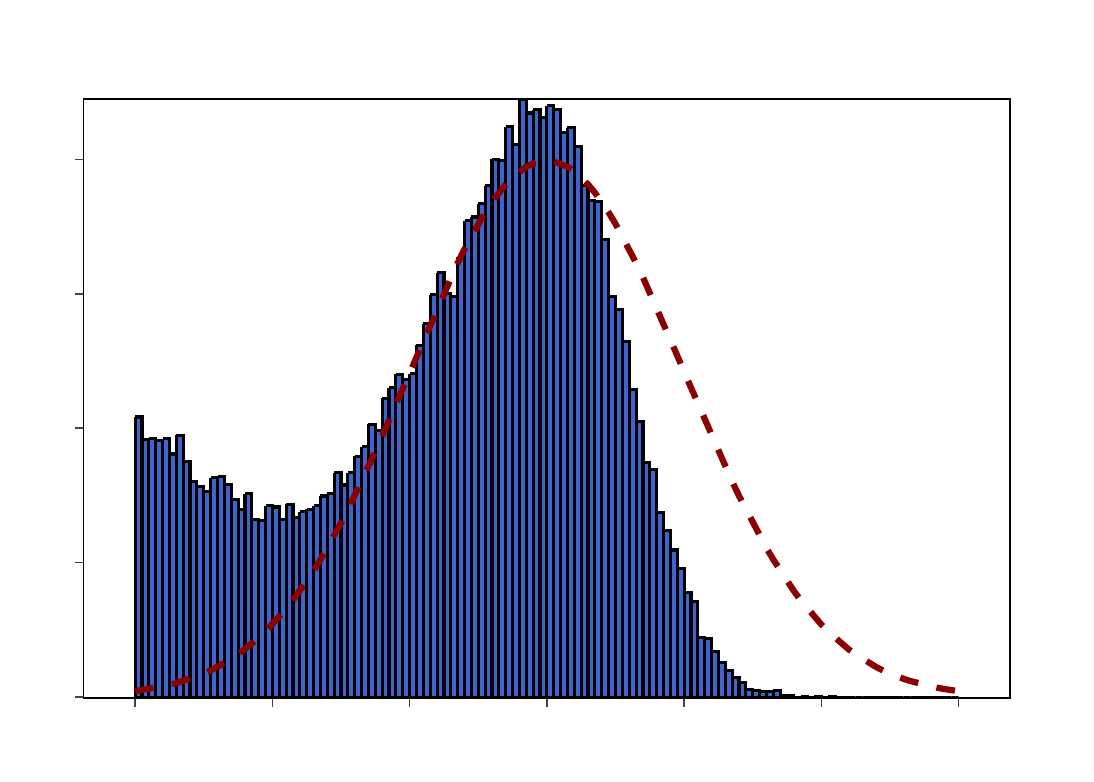}
      & \includegraphics[width=2.2cm,height=2.2cm]{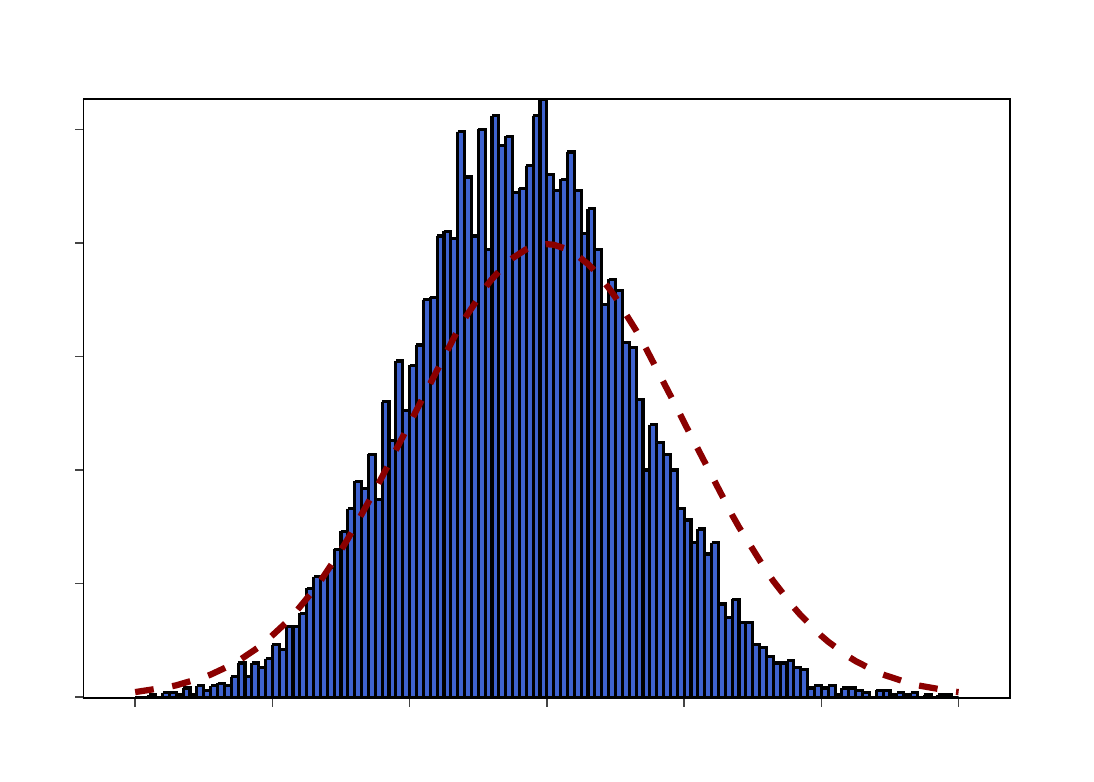}
    \\[2pt]
    \rotatebox{90}{$\text{Cauchy}(1,2)$}
      & \includegraphics[width=2.2cm,height=2.2cm]{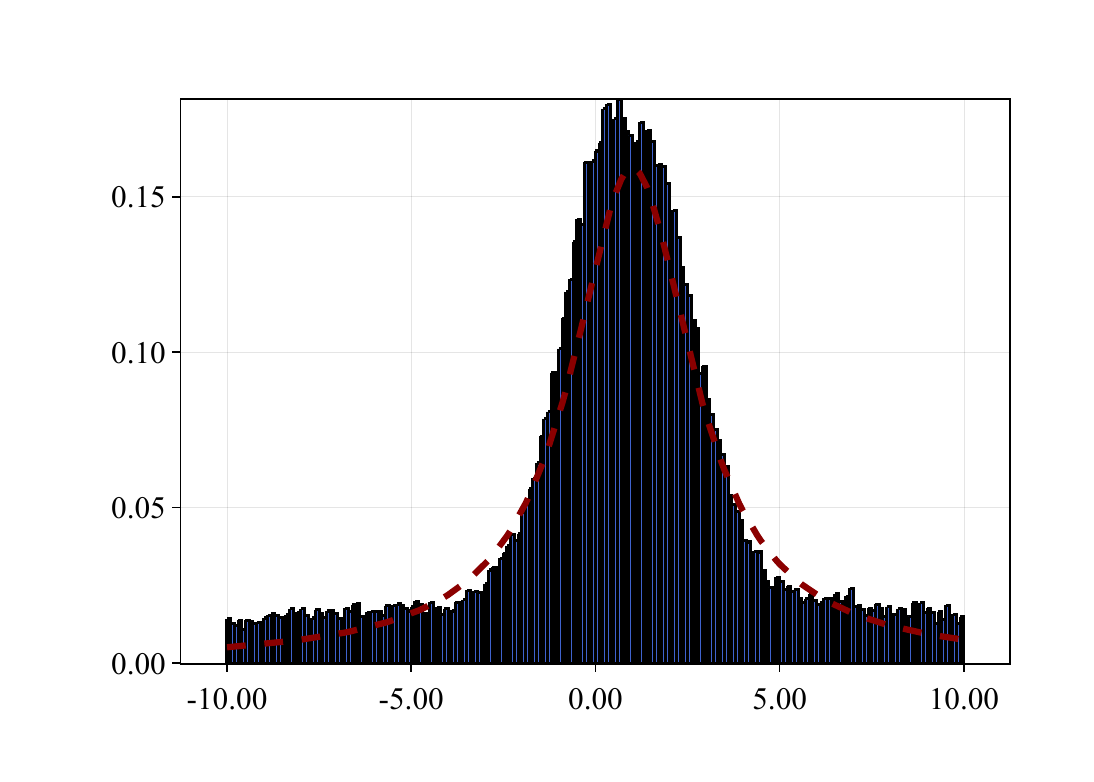}
      & \includegraphics[width=2.2cm,height=2.2cm]{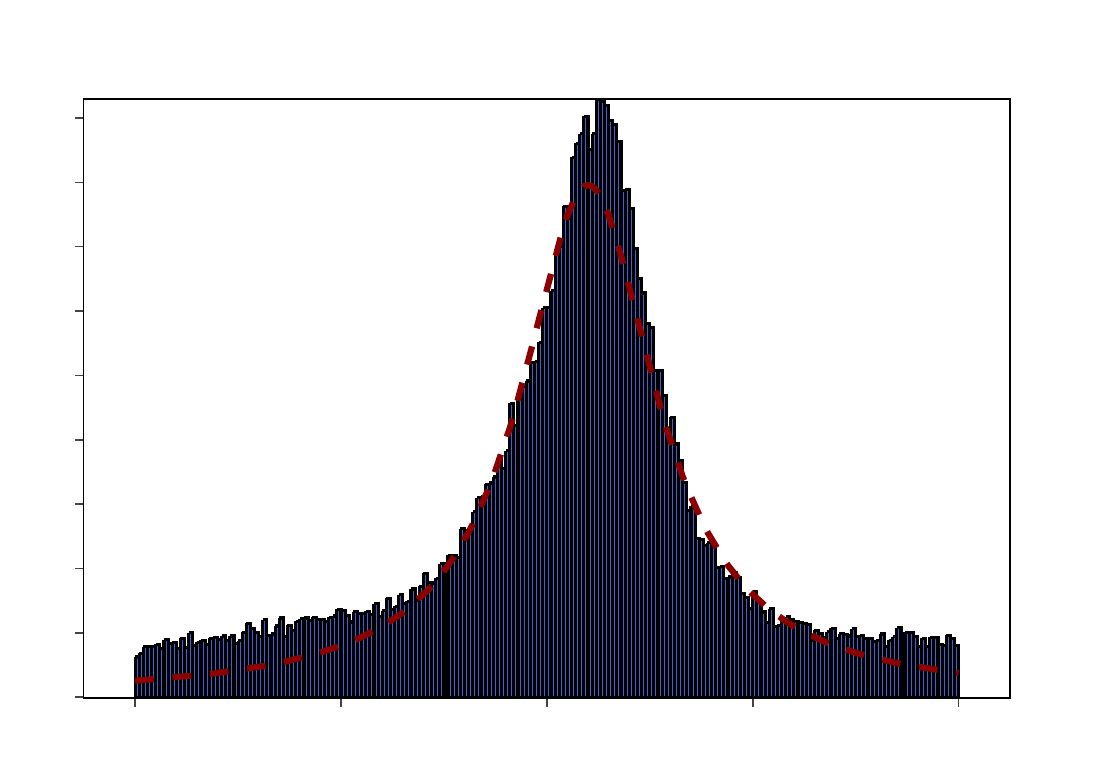}
      & \includegraphics[width=2.2cm,height=2.2cm]{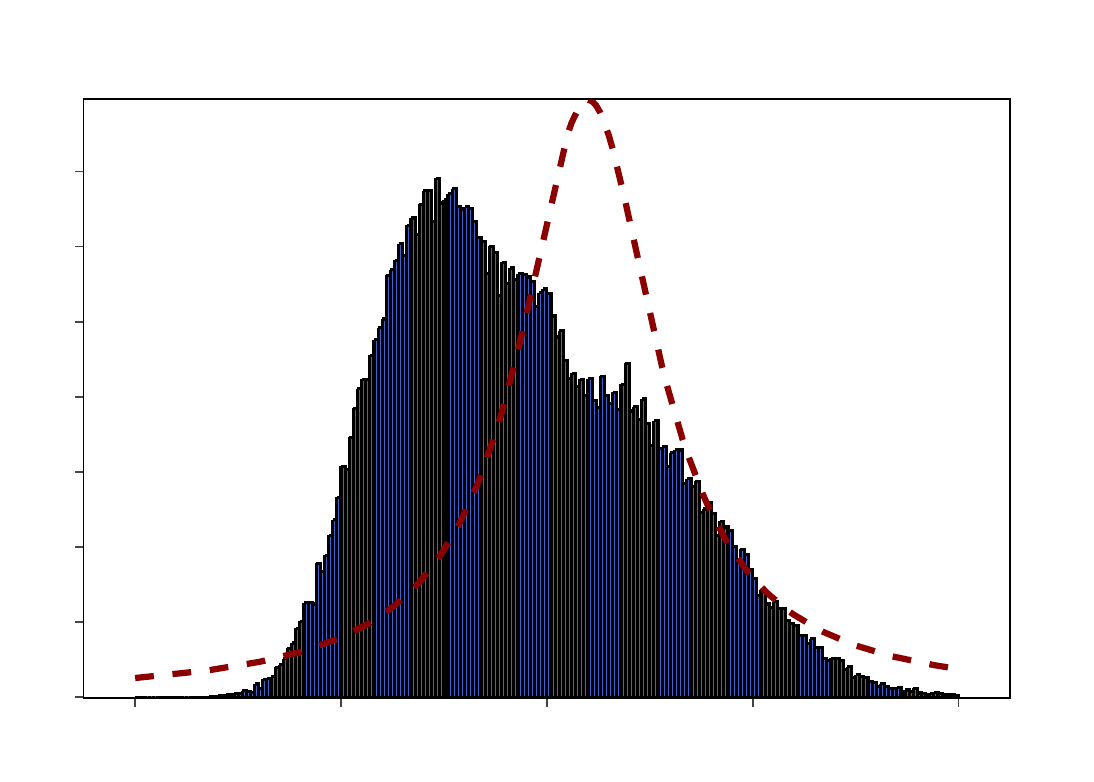}
      & \includegraphics[width=2.2cm,height=2.2cm]{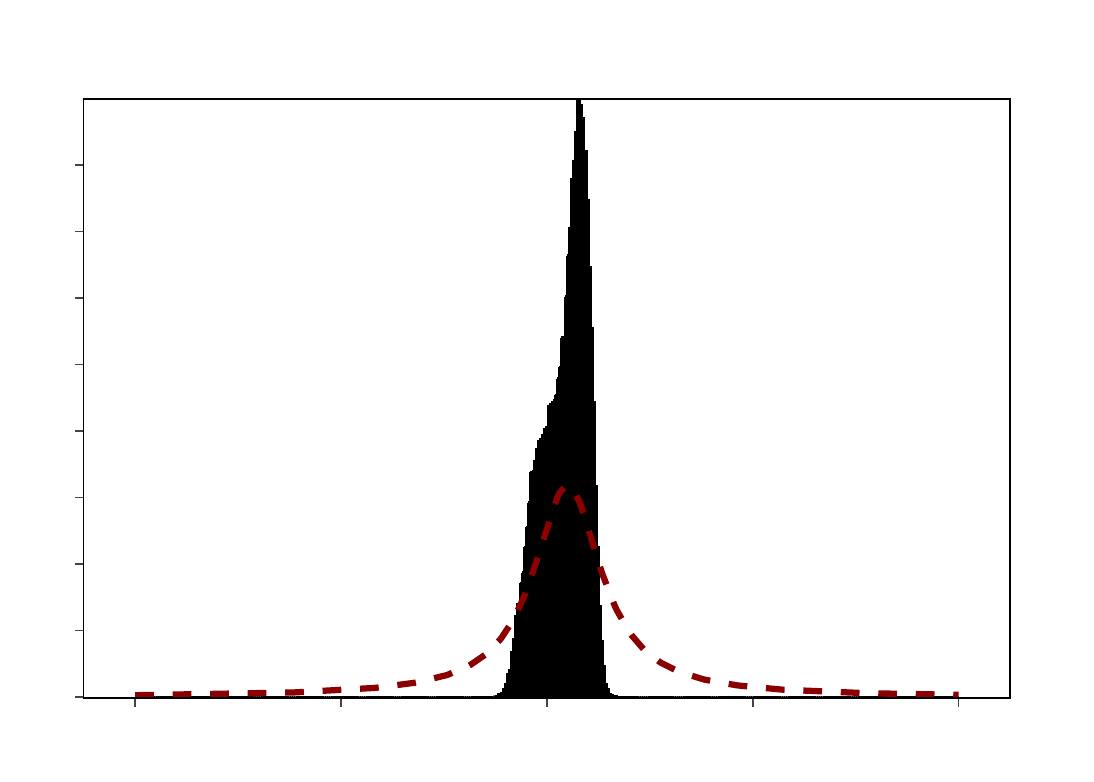}
      & \includegraphics[width=2.2cm,height=2.2cm]{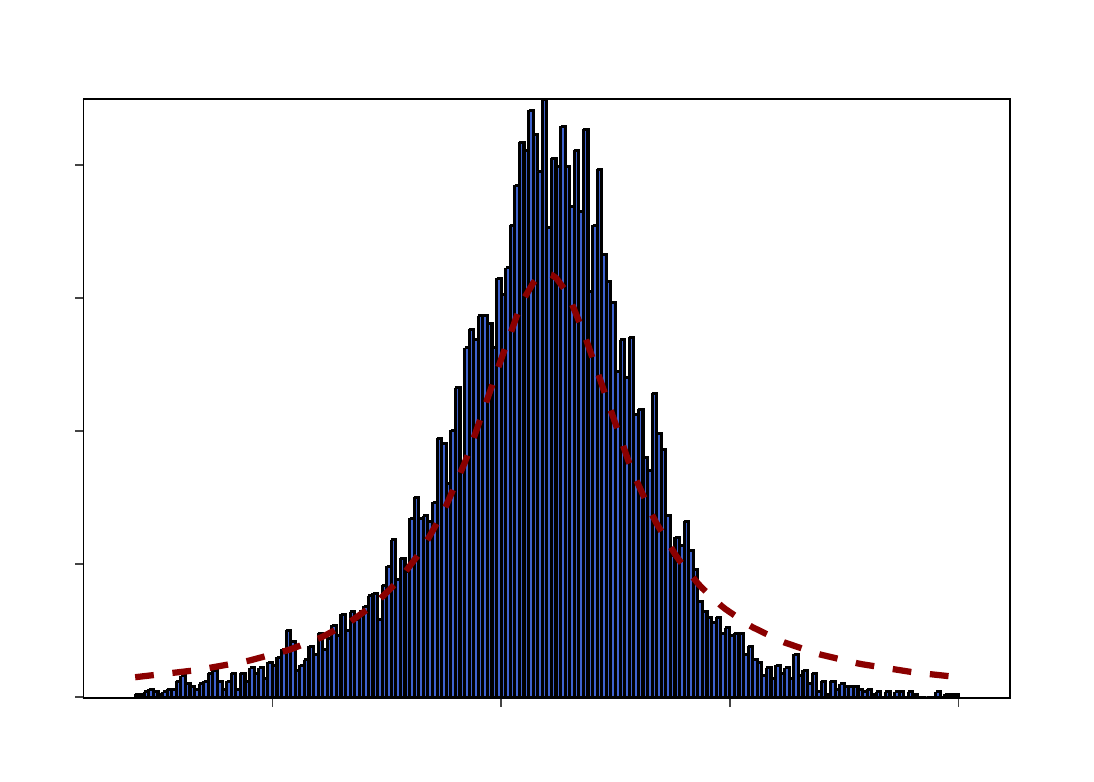}
    \\[2pt]
    \rotatebox{90}{$\text{Pareto}(1,1)$}
      & \includegraphics[width=2.2cm,height=2.2cm]{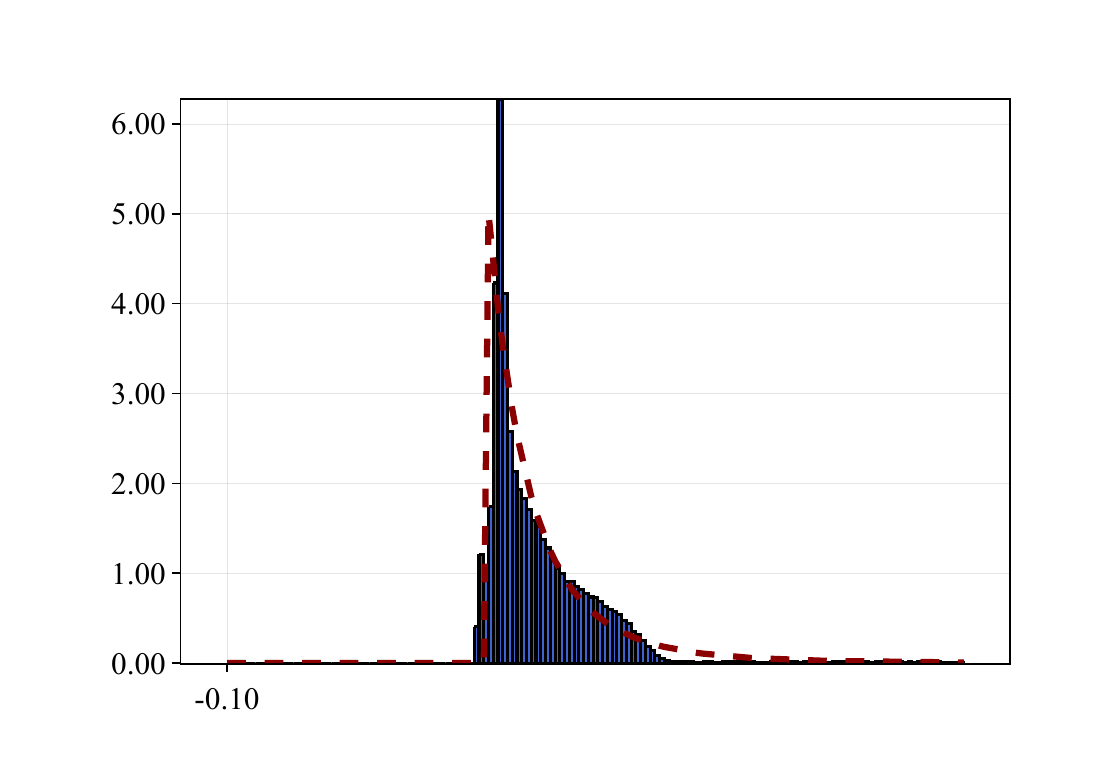}
      & \includegraphics[width=2.2cm,height=2.2cm]{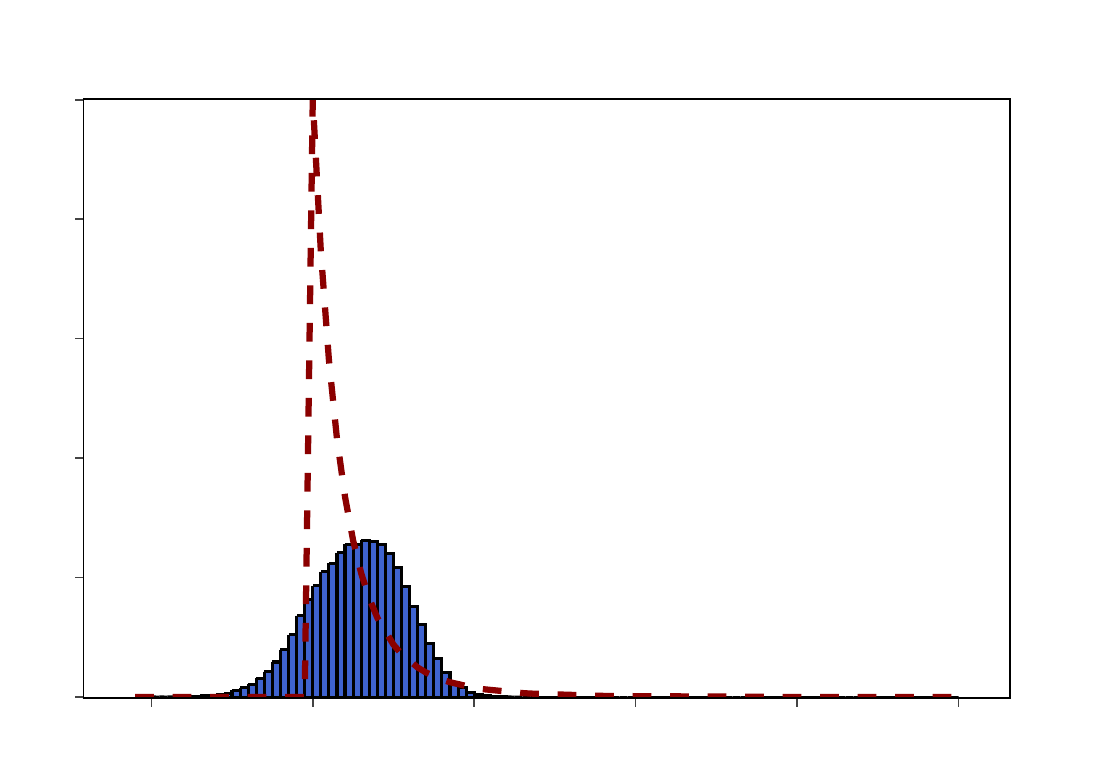}
      & \includegraphics[width=2.2cm,height=2.2cm]{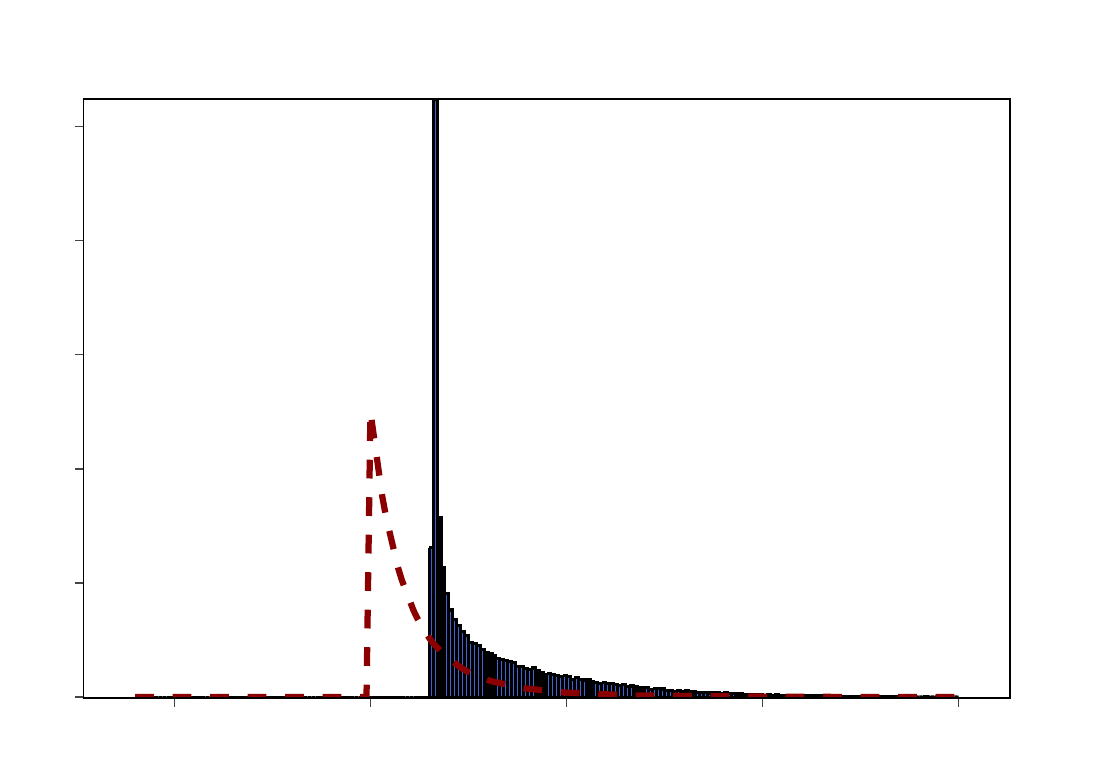}
      & \includegraphics[width=2.2cm,height=2.2cm]{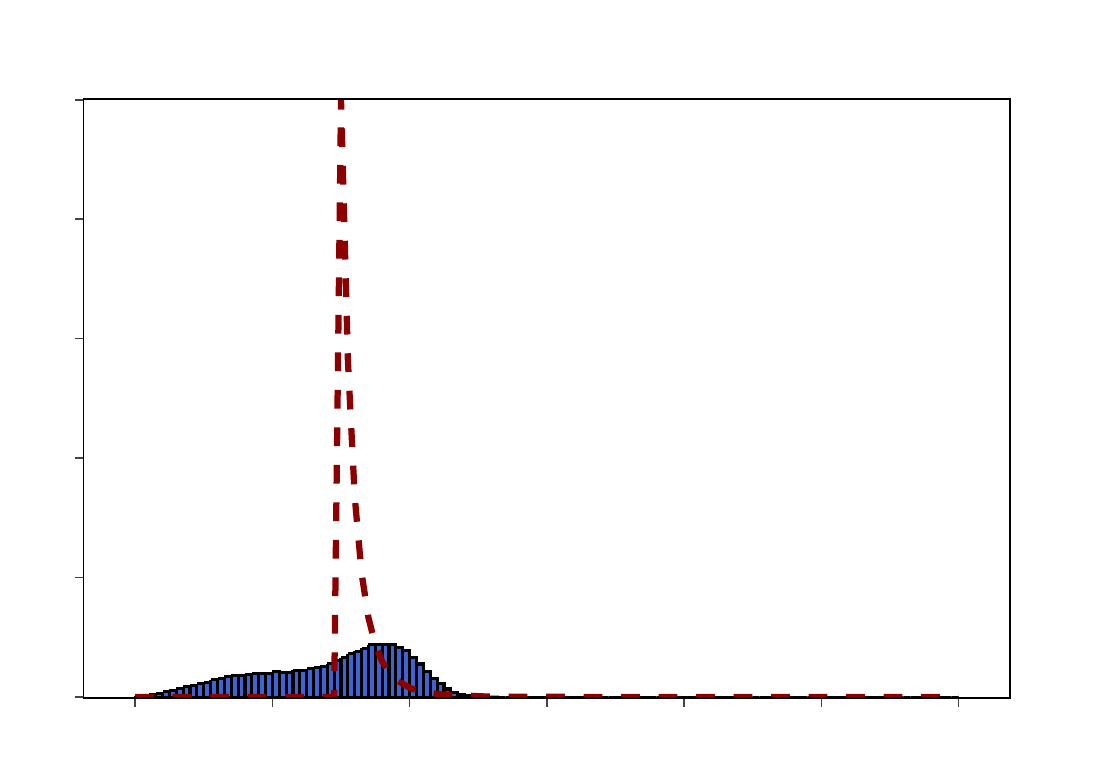}
      & \includegraphics[width=2.2cm,height=2.2cm]{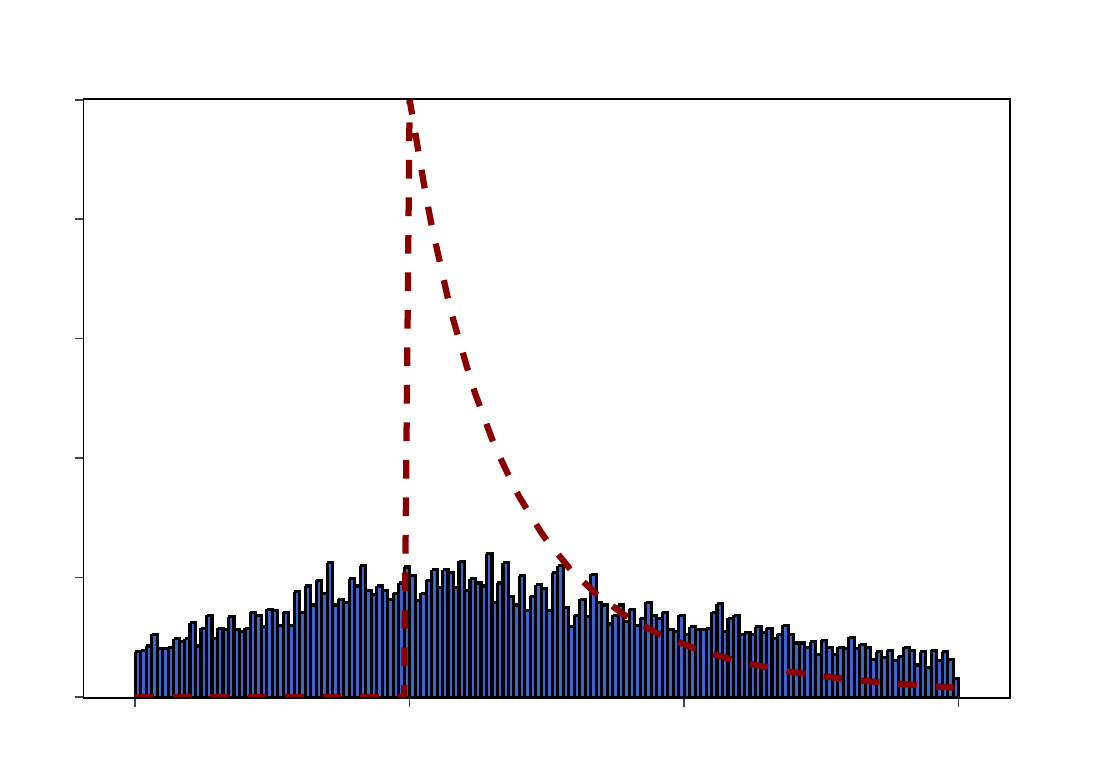}
    \\[2pt]
        \rotatebox{90}{$\text{Model}_1$}
      & \includegraphics[width=2.2cm,height=2.2cm]{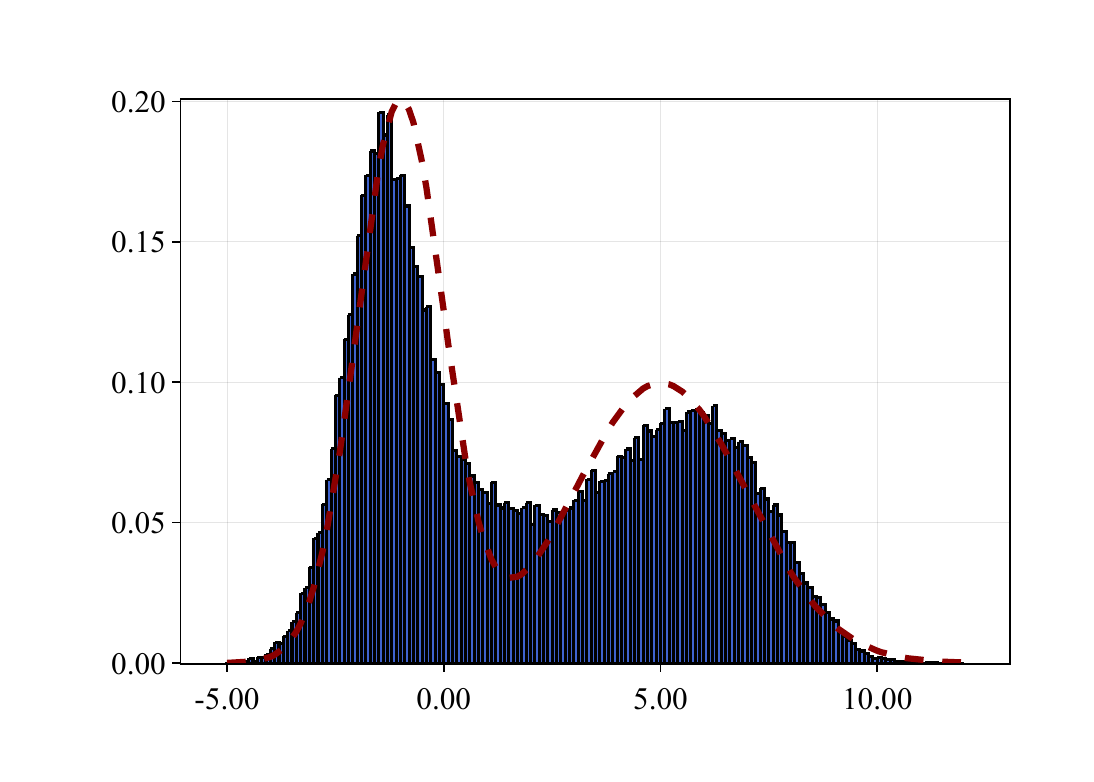}
      & \includegraphics[width=2.2cm,height=2.2cm]{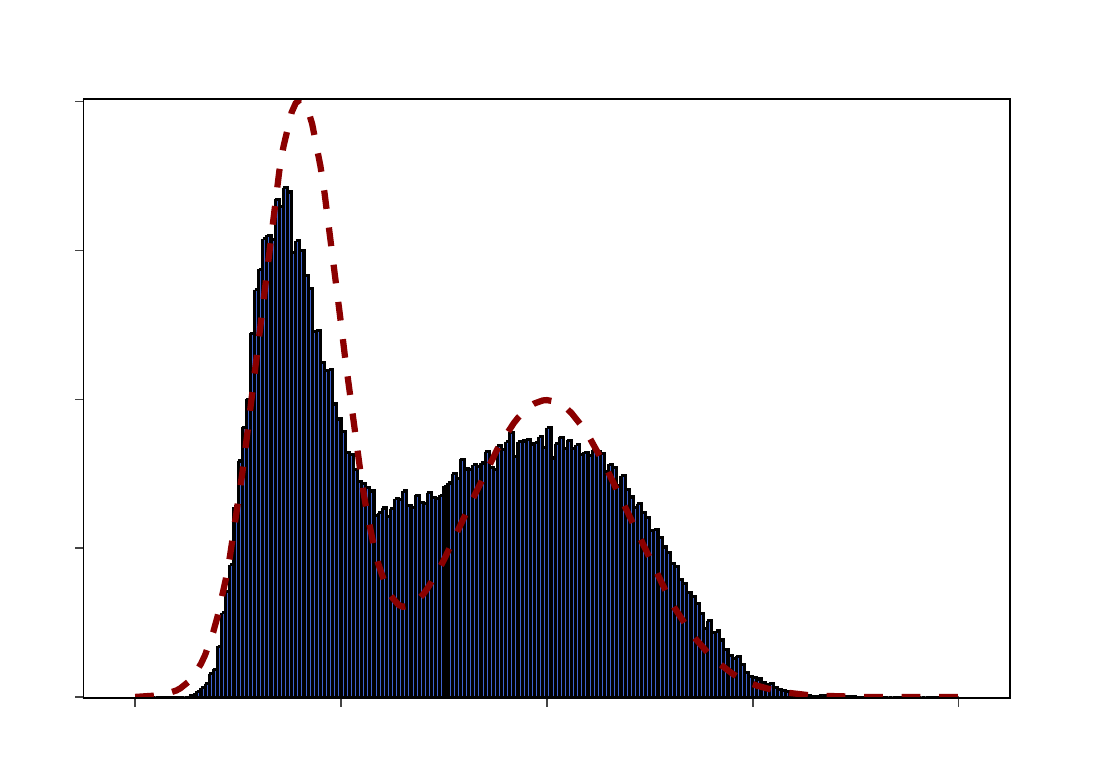}
      & \includegraphics[width=2.2cm,height=2.2cm]{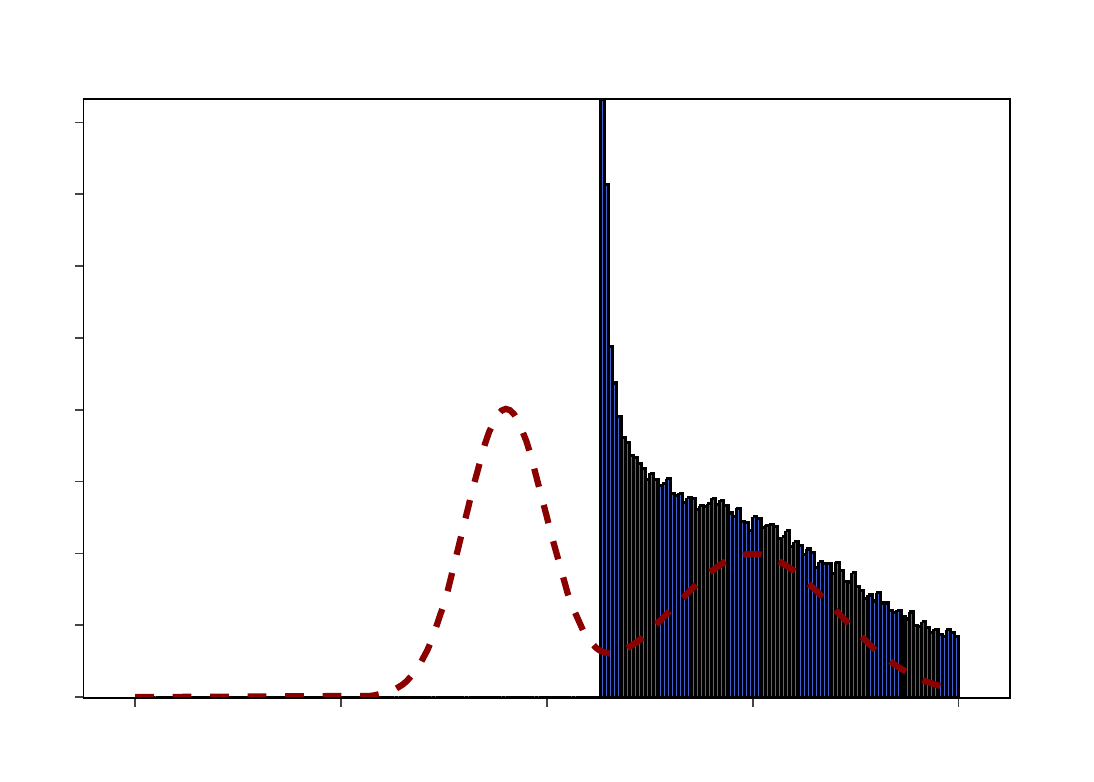}
      & \includegraphics[width=2.2cm,height=2.2cm]{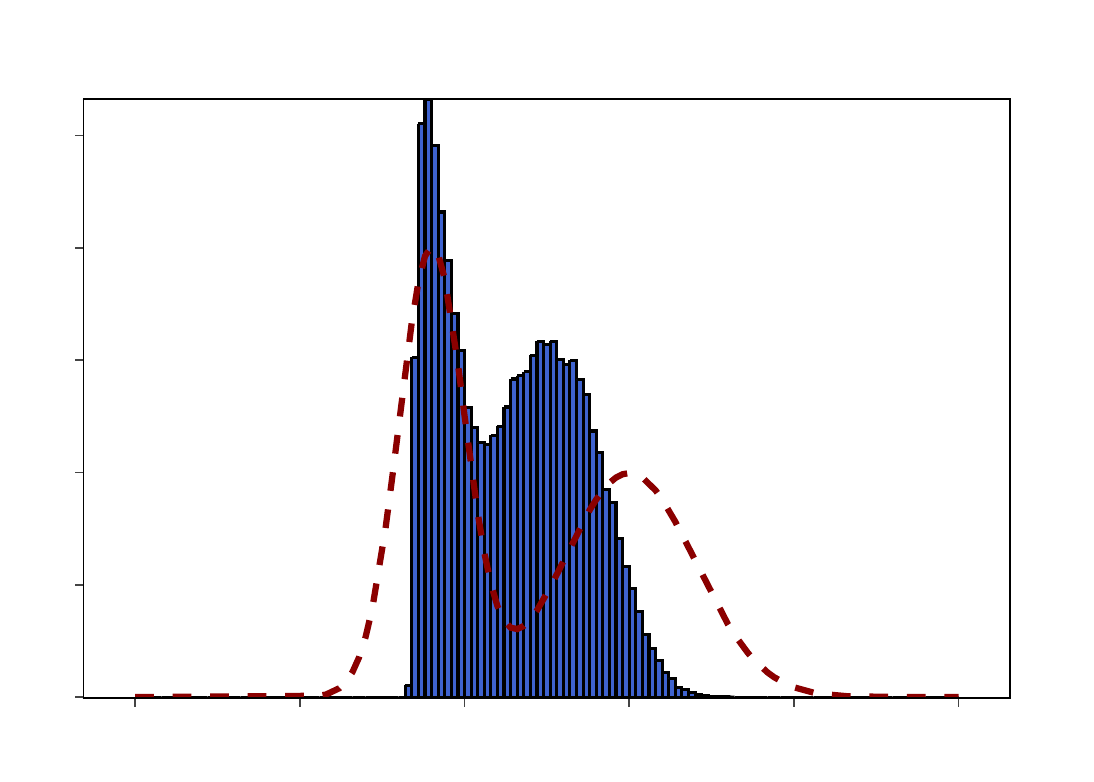}
      & \includegraphics[width=2.2cm,height=2.2cm]{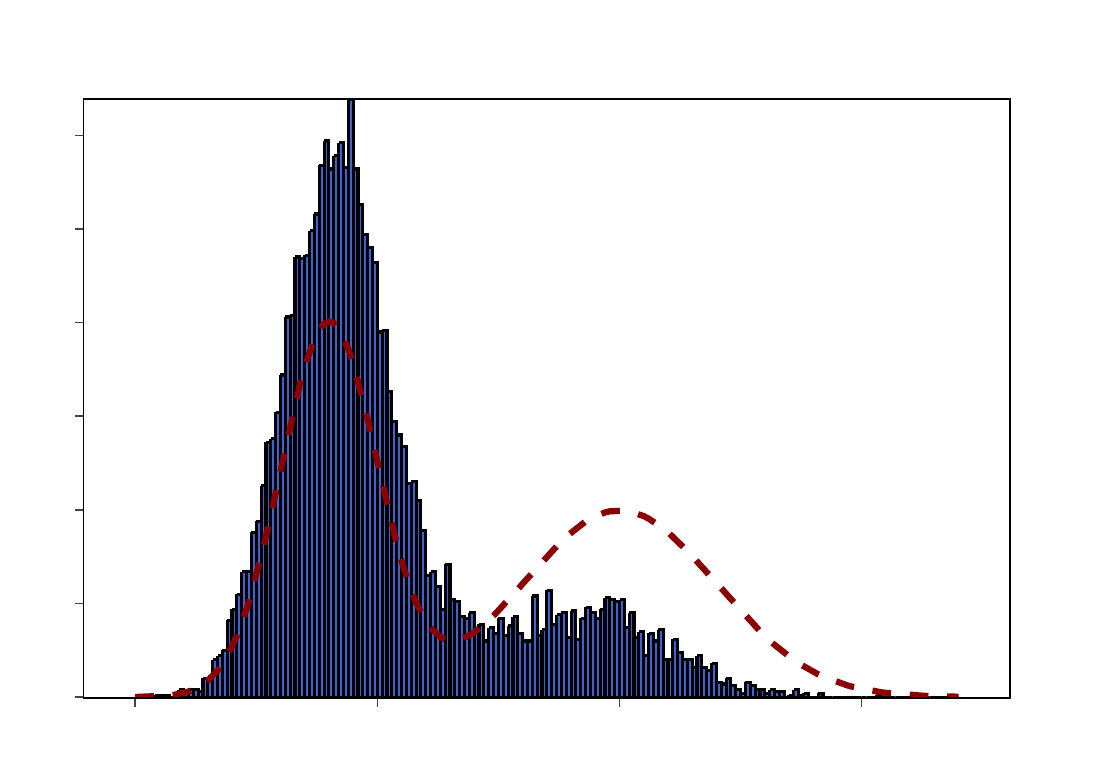}
      \\[2pt]
        \rotatebox{90}{$\text{Model}_2$}
      & \includegraphics[width=2.2cm,height=2.2cm]{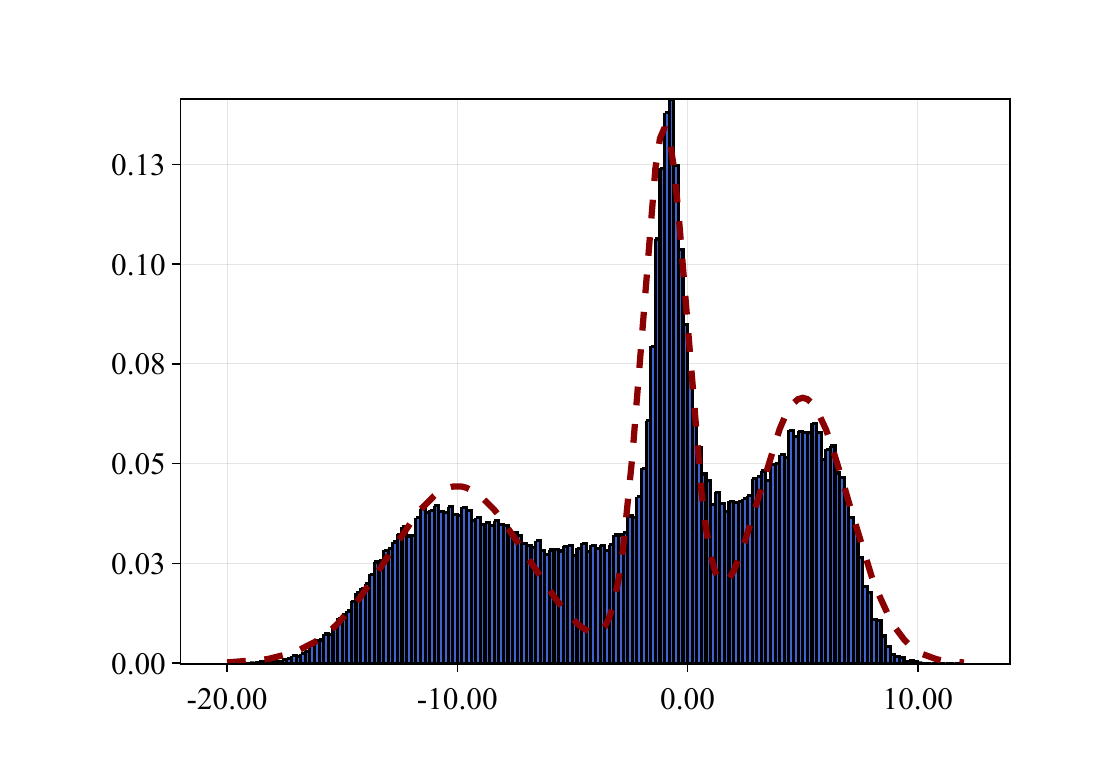}
      & \includegraphics[width=2.2cm,height=2.2cm]{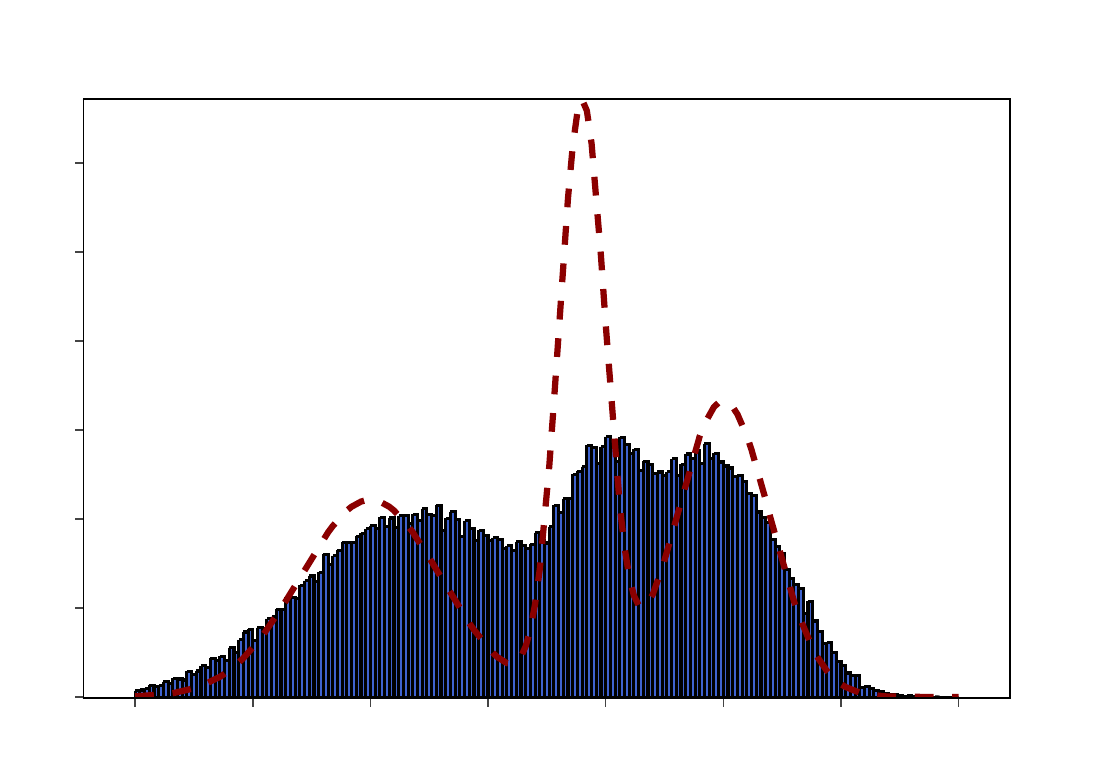}
      & \includegraphics[width=2.2cm,height=2.2cm]{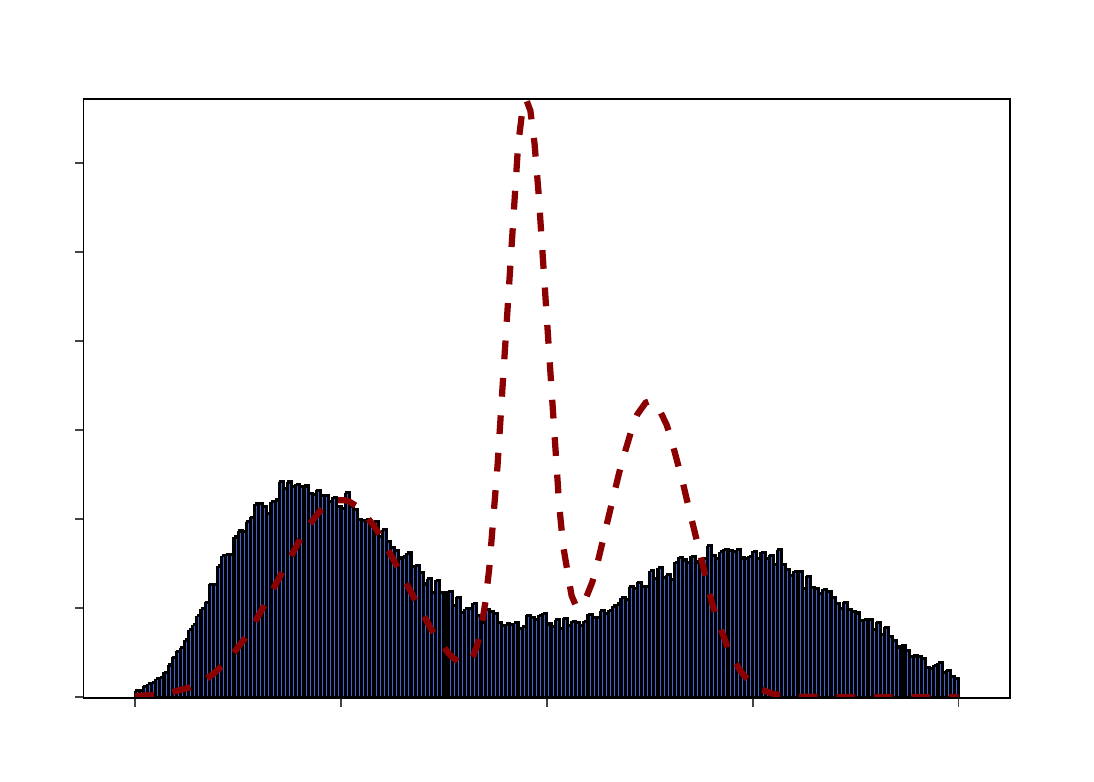}
      & \includegraphics[width=2.2cm,height=2.2cm]{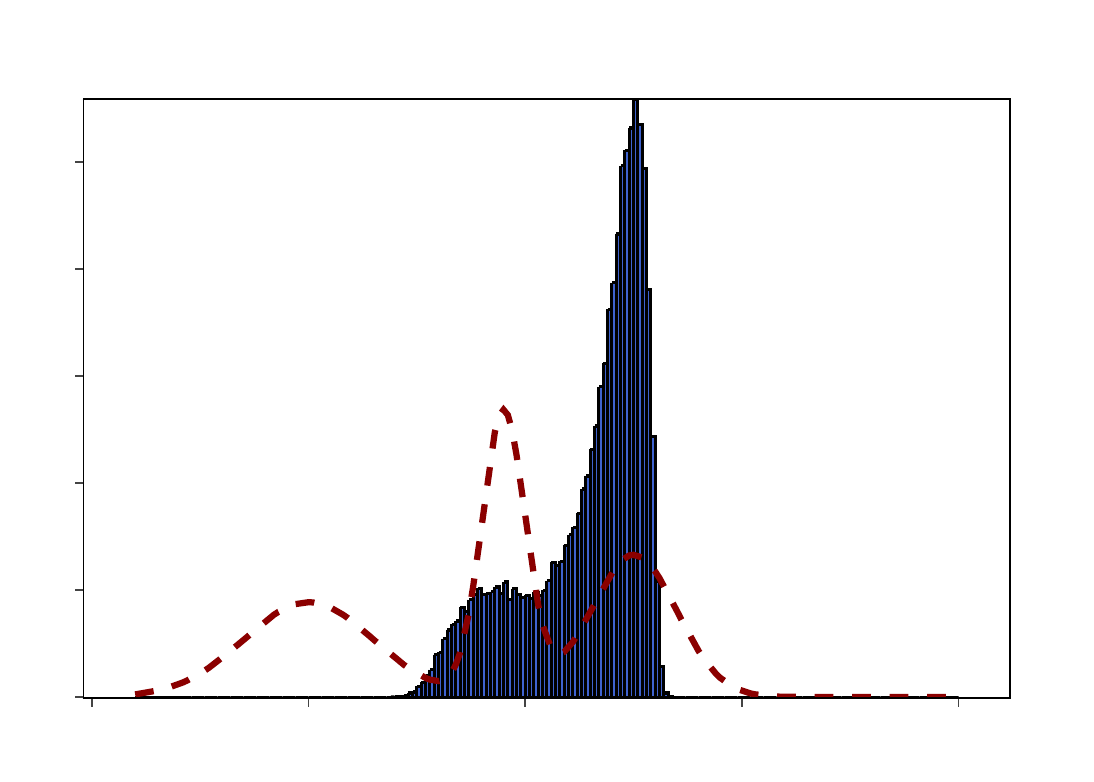}
      & \includegraphics[width=2.2cm,height=2.2cm]{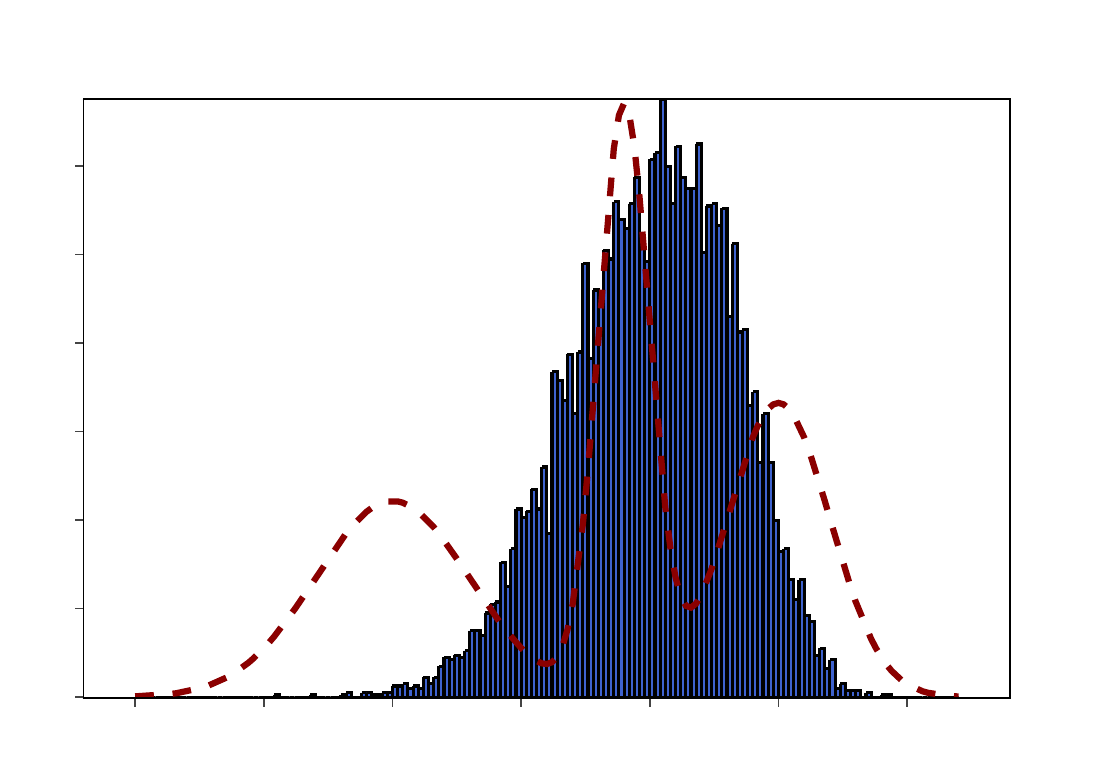}
      \\[2pt]
        \rotatebox{90}{$\text{Model}_3$}
      & \includegraphics[width=2.2cm,height=2.2cm]{img/1d_experiments/dualIsl_model3.pdf}
      & \includegraphics[width=2.2cm,height=2.2cm]{img/1d_experiments/isl_model3.pdf}
      & \includegraphics[width=2.2cm,height=2.2cm]{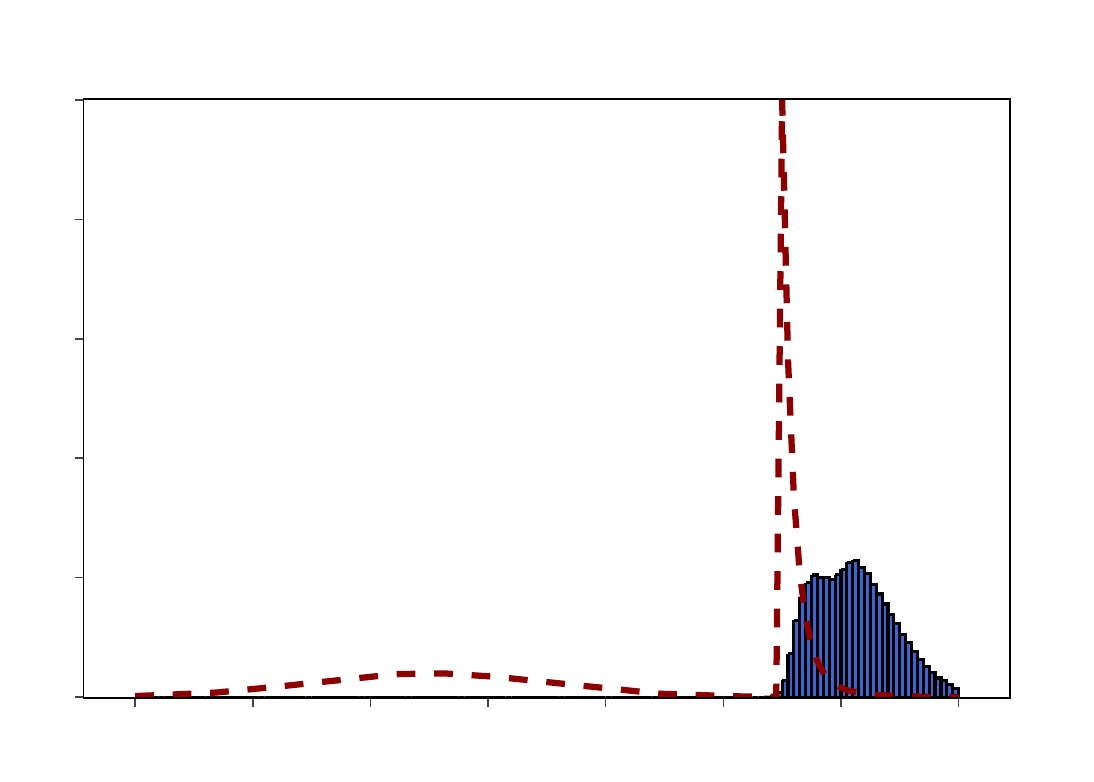}
      & \includegraphics[width=2.2cm,height=2.2cm]{img/1d_experiments/mmdgan_model3.pdf}
      & \includegraphics[width=2.2cm,height=2.2cm]{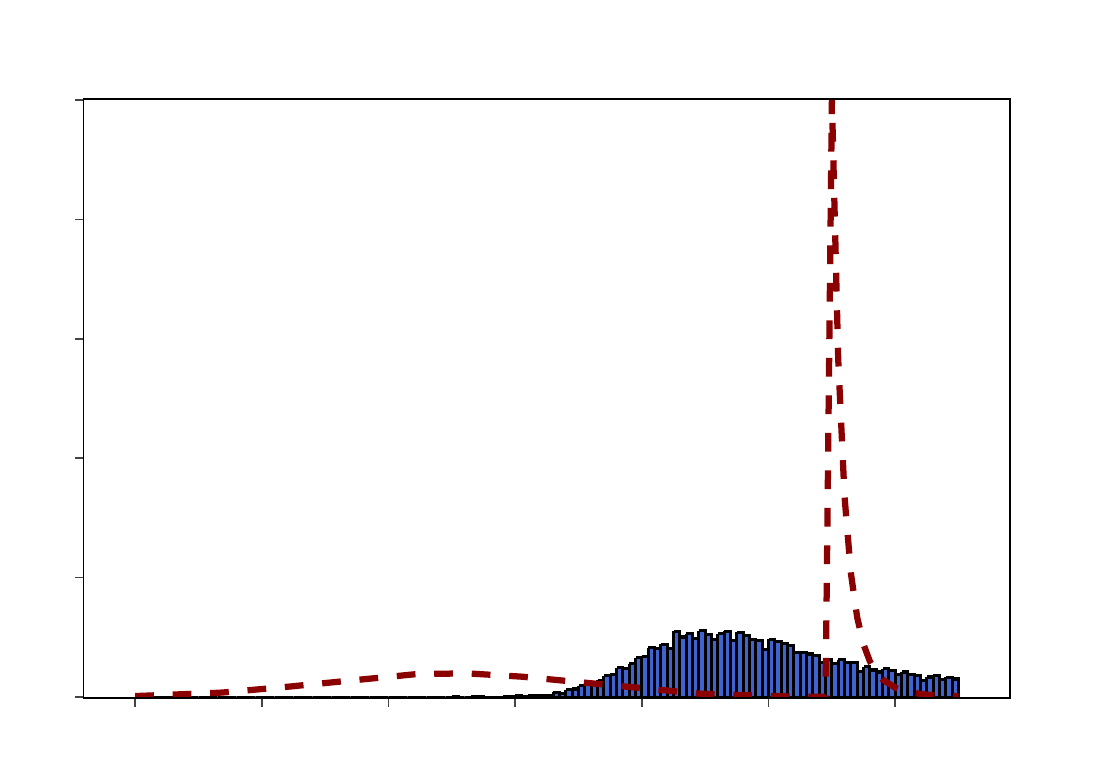}
      \\
    \bottomrule
  \end{tabular}
  \caption{\small One‐dimensional density estimation across six benchmark targets.  Each row corresponds to a different true distribution (top to bottom: \(\mathcal{N}(4,2)\), Cauchy\((1,2)\), Pareto\((1,1)\), Model$_1$, Model$_2$, Model$_3$).  In each subplot, the red curve shows the ground‐truth density and the blue curve shows the model’s estimated density.  Columns (left to right) compare dual‐ISL, classical ISL, WGAN, MMD‐GAN, and a DDPM diffusion baseline, respectively.  Dual‐ISL more accurately captures multi‐modal and heavy‐tailed shapes, with reduced mode‐collapse and smoother estimates.}
  \label{fig:1d-distributions}
\end{figure}

Referring to \cite[Theorem 1]{zaheer2017gan}, in the one‐dimensional case any generator that perfectly pushes forward a simple base distribution \(p_z\) (e.g.\ uniform or Gaussian) to a target distribution \(p\) must implement one of at most two continuous maps.  Concretely, if $F_{z}$ and $F$ are the cdfs of $p_z$ and $p$, then the two solutions are
\begin{align*}
f_{+}(z) \;=\; F^{-1}\bigl(F_{z}(z)\bigr)
\qquad\text{and}\qquad
f_{-}(z) \;=\; F^{-1}\bigl(1 - F_{z}(z)\bigr).
\end{align*}
In practice, a learning algorithm that truly captures the underlying structure of a multimodal or heavy‐tailed distribution should recover one of these two “probability‐integral‐transform” maps.

In Figure \ref{figure:dual-isl vs isl vs mmd-gan mixture normal pareto:appendix}, we therefore plot, for each method, the learned generator \(f_{\theta}(z)\) against the theoretical target map \(f_{+}(z)\) for the challenging mixture \(\tfrac12\mathcal{N}(-5,2) + \tfrac12\mathrm{Pareto}(5,1)\) (“Model\(_3\)”).  A perfect fit would lie exactly on the diagonal.  As seen, dual-ISL (Subfigure \ref{figure:dual-isl vs isl vs mmd-gan mixture normal pareto:dual-ISL:appendix}) closely tracks the true transformation across the full support of \(z\), whereas classical ISL (Subfigure \ref{figure:dual-isl vs isl vs mmd-gan mixture normal pareto:isl:appendix}) suffers from local distortions around the Pareto tail, and MMD-GAN (Subfigure \ref{figure:dual-isl vs isl vs mmd-gan mixture normal pareto:mmd-gan:appendix}) exhibits even larger deviations—especially where the two modes meet.  This visualization makes explicit how dual-ISL more faithfully learns the correct mapping, rather than merely matching moments or densities.

\begin{figure}[htbp] 
   \centering
   \begin{subfigure}{0.48\textwidth}
     \centering
     \includegraphics[width=\linewidth]{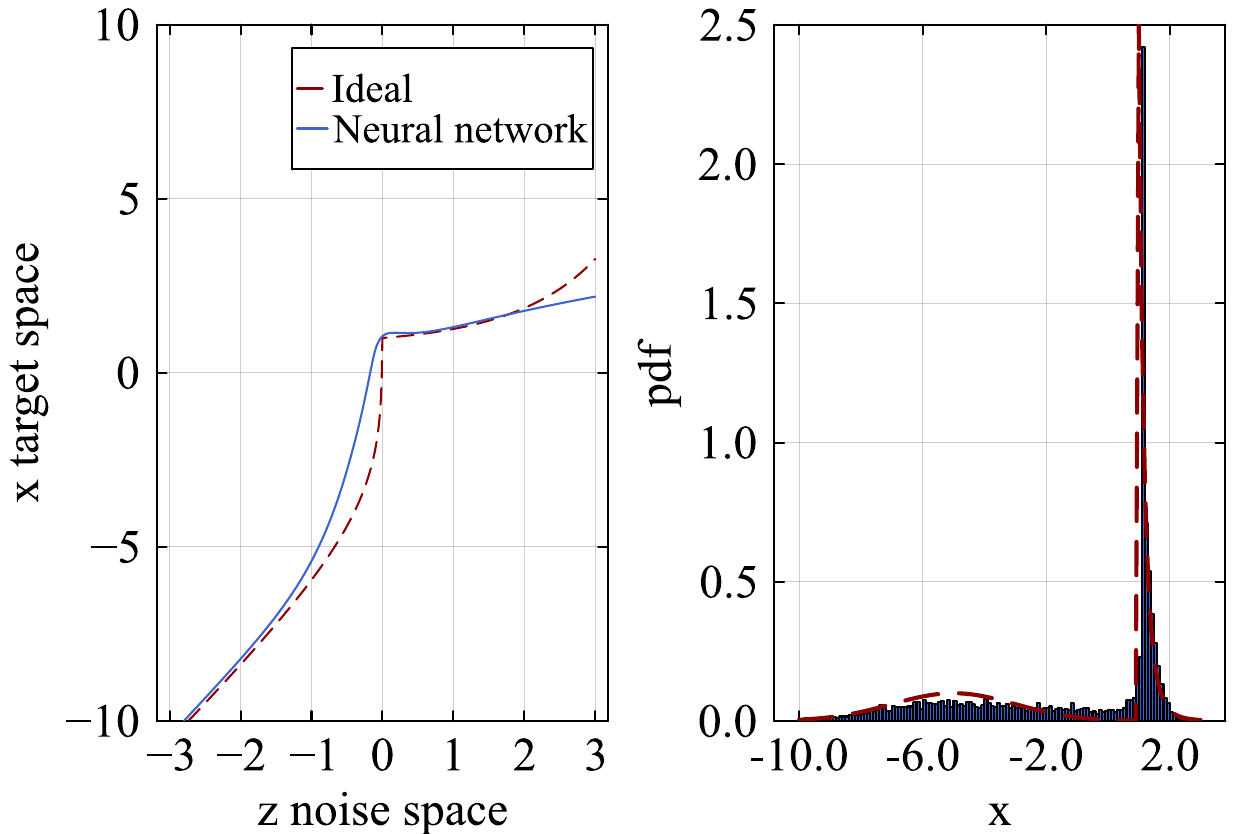}
     \caption{dual-ISL}
     \label{figure:dual-isl vs isl vs mmd-gan mixture normal pareto:dual-ISL:appendix}
   \end{subfigure}
   \hfill
   \begin{subfigure}{0.48\textwidth}
     \centering
     \includegraphics[width=\linewidth]{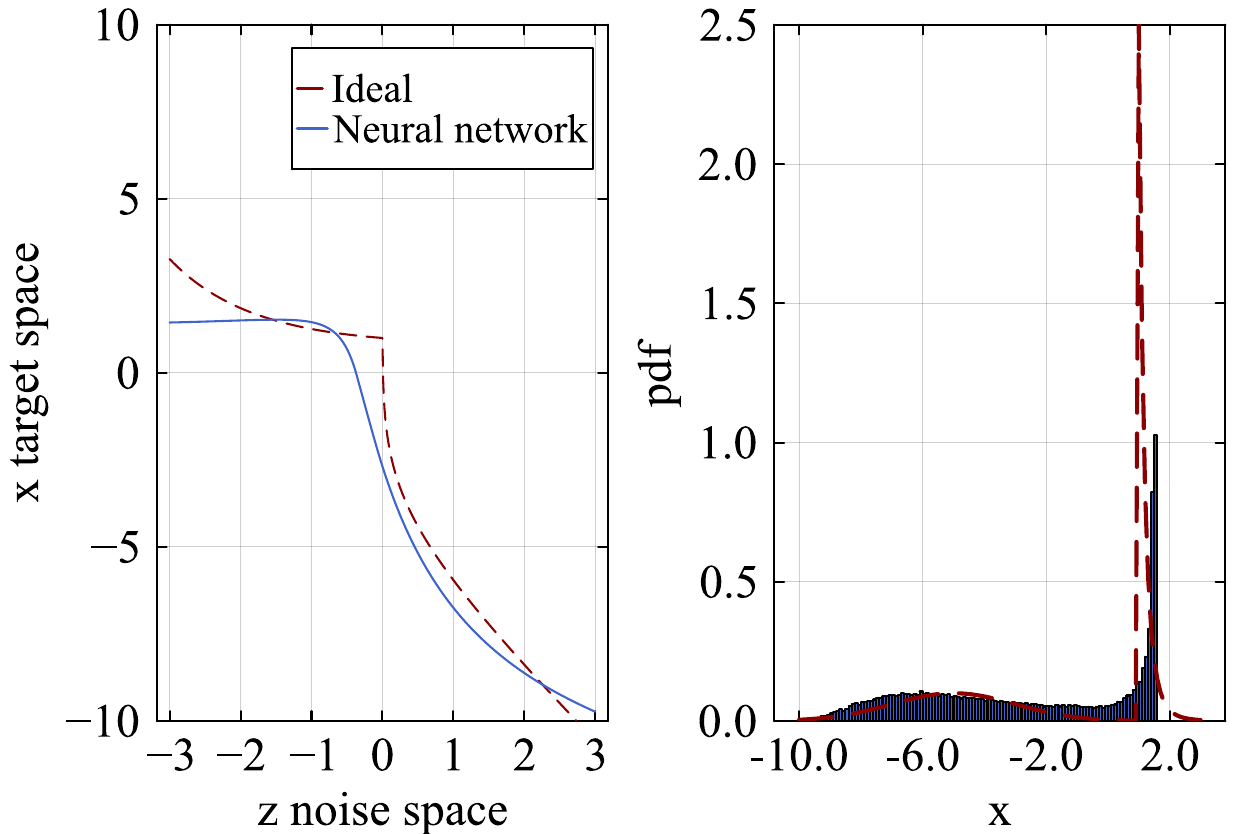}
     \caption{ISL}
     \label{figure:dual-isl vs isl vs mmd-gan mixture normal pareto:isl:appendix}
   \end{subfigure}
   \hfill
   \begin{subfigure}{0.48\textwidth}
     \centering
     \includegraphics[width=\linewidth]{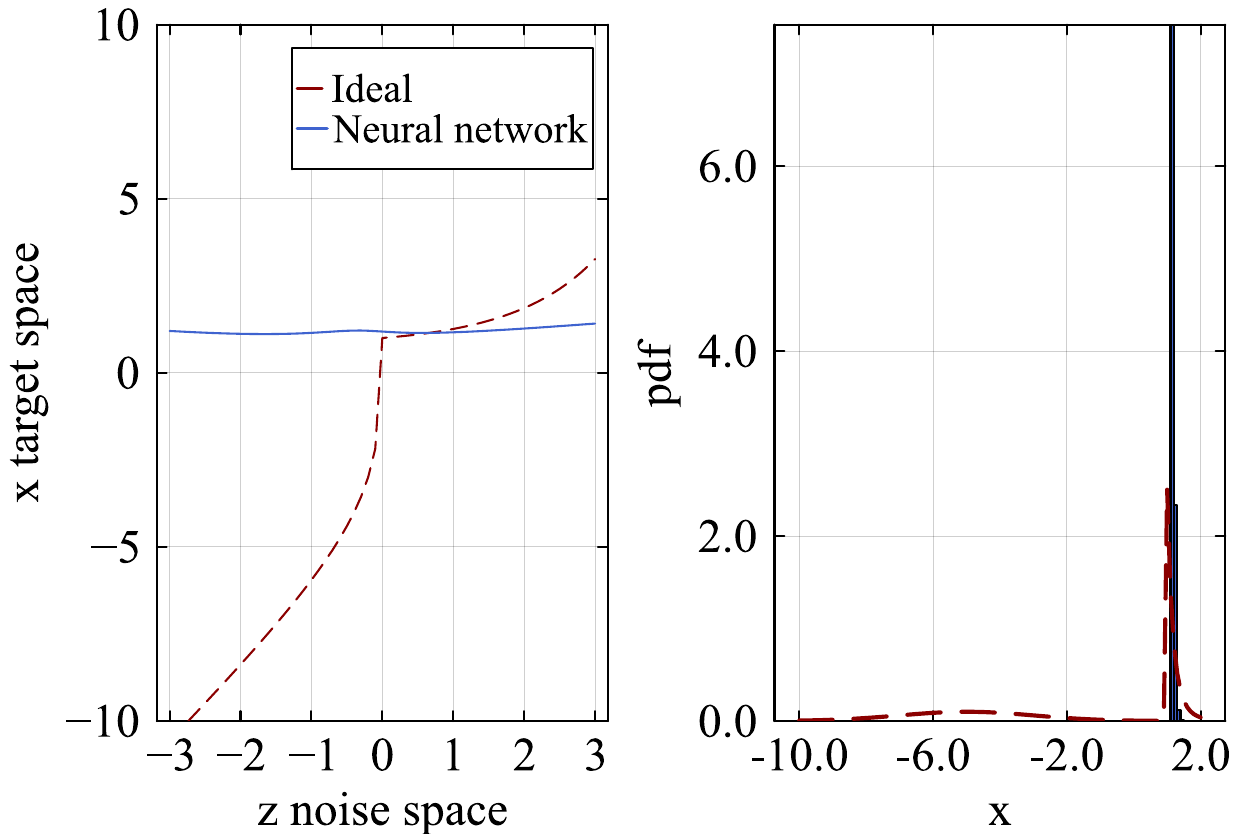}
     \caption{MMD-GAN}
     \label{figure:dual-isl vs isl vs mmd-gan mixture normal pareto:mmd-gan:appendix}
   \end{subfigure}
   \caption{\small Comparison of learned generator mappings \(f_{\theta}(z)\) against the true probability‐integral‐transform \(f_{+}(z)\) or \(f_{-}(z)\) for Model\(_3\). Dual-ISL closely follows the ideal map even in the heavy‐tailed region, while ISL and MMD-GAN display growing errors, particularly near the mode boundaries and in the Pareto tail.}
   \label{figure:dual-isl vs isl vs mmd-gan mixture normal pareto:appendix}
\end{figure}

\subsection{Computational Benchmarking of Dual-ISL vs (Classical) ISL}

Next, we measure execution time using dedicated benchmarking tools \cite{chen2016robust}. These tools first warm up and calibrate the code to determine the optimal number of iterations per measurement, then execute the code in bundled loops to collect multiple independent samples. We compute statistics—including the minimum, median, mean, and standard deviation—while tracking garbage collection time separately. Table \ref{tab:benchmarks dual-isl vs isl} summarizes the results for various target distributions at a fixed \(K\). Figure~\ref{figure:benchmarks dual-isl vs isl} combines two perspectives:  
(a) total runtime as a function of \(K\), and  
(b) estimation accuracy versus runtime at a fixed \(K\).  
These findings show that Dual‐ISL not only runs faster than classical ISL, but also achieves a superior accuracy–runtime balance—and this advantage grows even larger as $K$ increases.


\begin{table}[H]
\centering
\small
\rowcolors{2}{gray!10}{white}
\begin{tabular}{l 
                S[table-format=3.3] S[table-format=3.3(3)] S[table-format=3.2] 
                S[table-format=3.3] S[table-format=3.3(3)] S[table-format=3.2] }
\toprule
\textbf{Target} 
  & \multicolumn{3}{c}{\textbf{ISL}} 
  & \multicolumn{3}{c}{\textbf{Dual‐ISL}} \\
\cmidrule(lr){2-4} \cmidrule(lr){5-7}
  & \textbf{Median} & \textbf{Mean \(\pm\)\ \(\sigma\)} & \textbf{Memory} 
  & \textbf{Median} & \textbf{Mean \(\pm\)\ \(\sigma\)} & \textbf{Memory} \\
\midrule
\(\mathcal{N}(4,2)\)  
  & 239.374 & 241.280 \pm 0.434 & 17.42\,GiB  
  & 22.281  & 22.502  \pm 0.118 & 9.27\,GiB \\

\(\mathrm{Pareto}(1,1)\)  
  & 238.043 & 239.007 \pm 1.607 & 17.42\,GiB  
  & 22.207  & 22.109 \pm 0.518 & 9.27\,GiB \\

\(\mathrm{Model}_1\)  
  & 241.683 & 241.905 \pm 1.722 & 17.42\,GiB  
  & 21.685  & 21.709 \pm 0.044 & 9.27\,GiB \\

\(\mathrm{Model}_2\)  
  & 237.825 & 239.805 \pm  1.832 & 17.42\,GiB    
  & 22.921  & 22.886 \pm 0.453 & 9.27\,GiB \\
\bottomrule
\end{tabular}
\caption{\small Runtime (in seconds) and memory benchmarks for ISL vs.\ Dual‐ISL at fixed \(K=10\).  
Each cell reports median runtime, mean \(\pm\) standard deviation, and peak memory usage.}
\label{tab:benchmarks dual-isl vs isl}
\end{table}

\begin{figure}[htbp] 
  \centering
  \begin{subfigure}{0.45\textwidth}
    \centering
    \includegraphics[width=\linewidth]{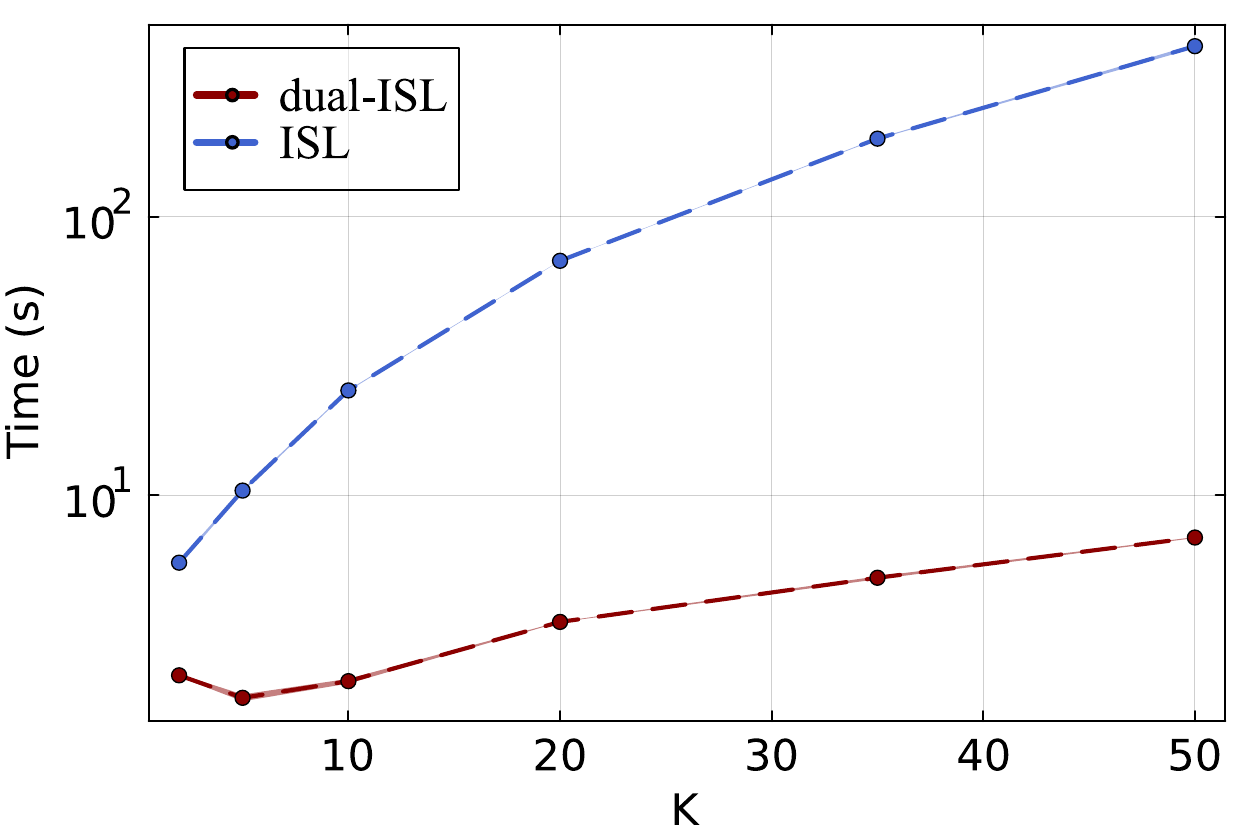}
    \caption{\small Target density Model\textsubscript{2}}
    \label{fig:image1}
  \end{subfigure}
  \hfill
  \begin{subfigure}{0.45\textwidth}
    \centering
    \includegraphics[width=\linewidth]{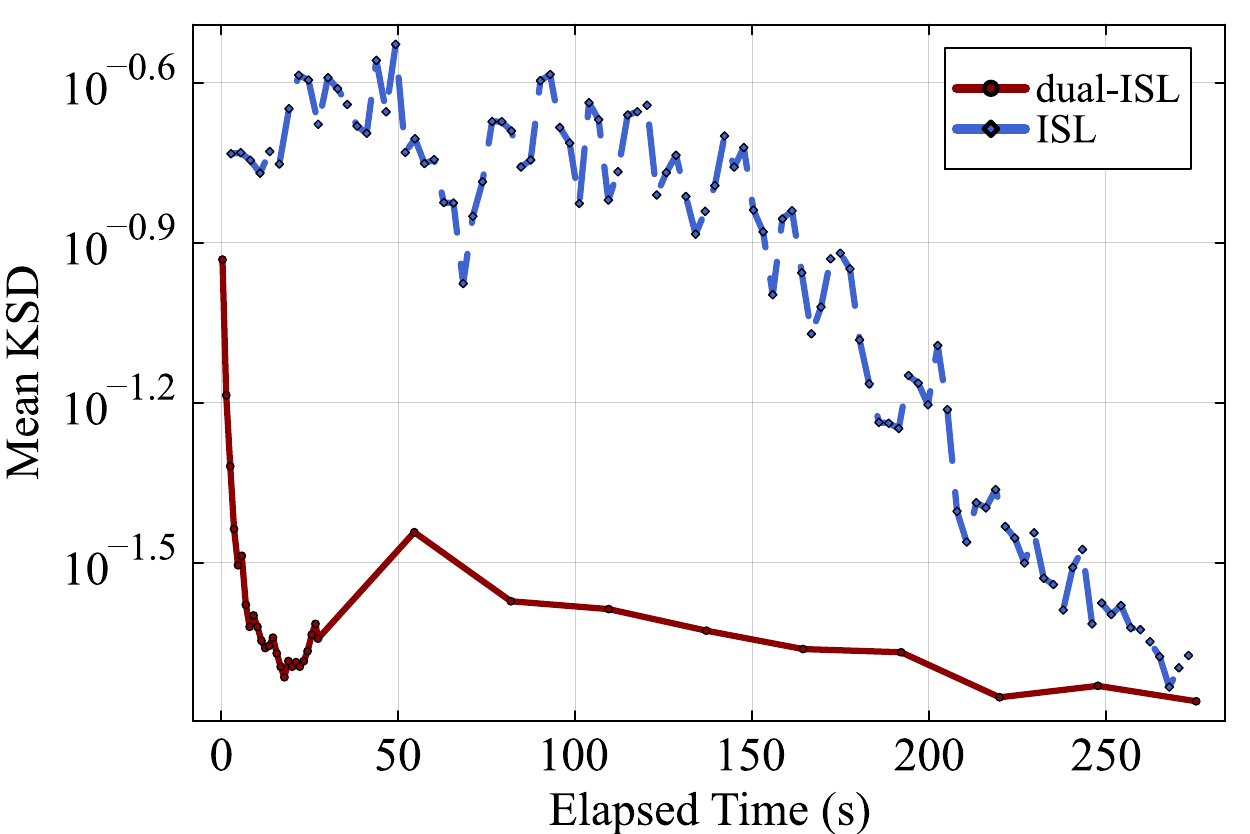}
    \caption{\small Target density $\text{Model}_2$}
    \label{fig:image2}
  \end{subfigure}
  \caption{\small Computation time and accuracy trade‐off for classical ISL versus Dual‐ISL as \(K\) increases.  
(a) Total runtime (in seconds) for 1000 training epochs with batch size \(N=1000\) on the \(\mathrm{Model}_2\) target.  
(b) Runtime versus mean KSD for both methods with $K=10$ and batch size \(N=1000\), illustrating how Dual‐ISL maintains lower runtimes improving also accuracy.  
Dual‐ISL consistently outperforms classical ISL in speed, with the gap widening at larger \(K\).}
  \label{figure:benchmarks dual-isl vs isl}
\end{figure}

\begin{figure}[htbp] 
   \centering
   \begin{subfigure}{0.45\textwidth}
     \centering
     \includegraphics[width=\linewidth]{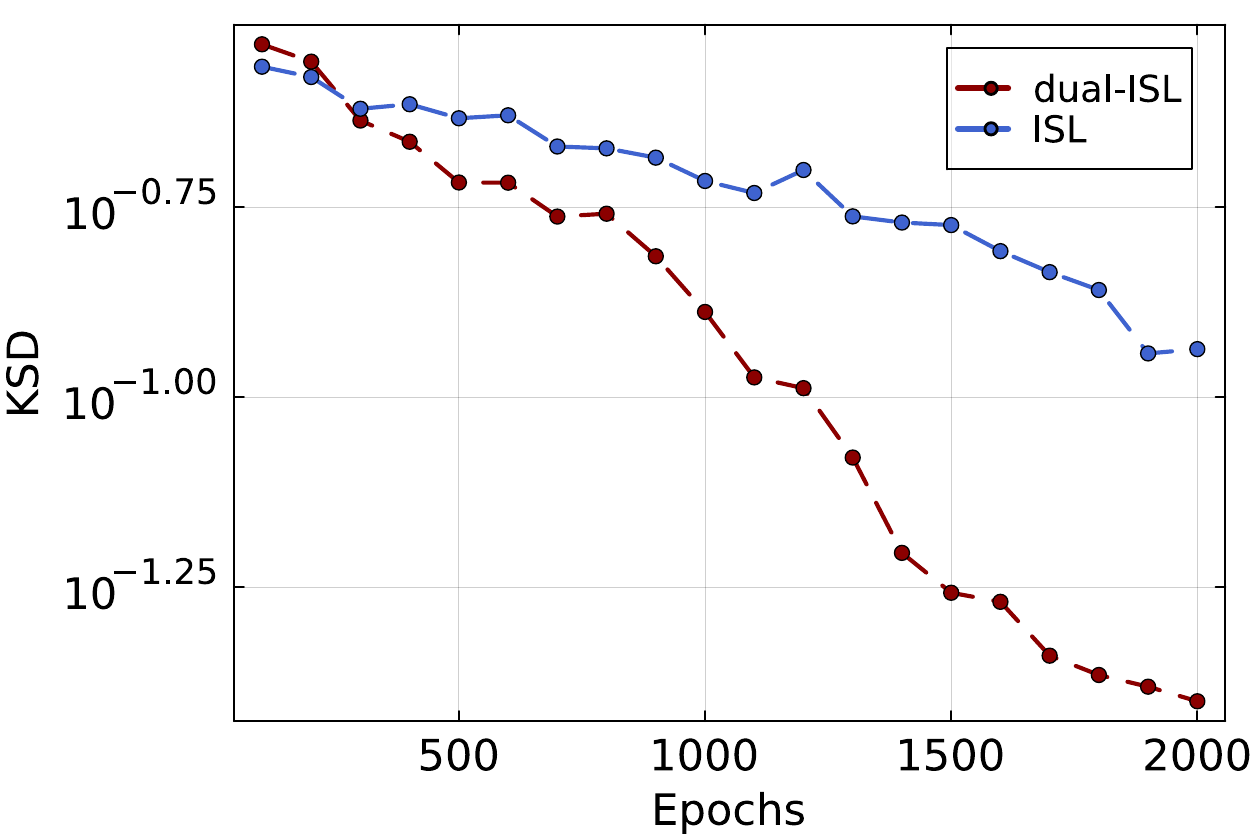}
     \caption{\small $\mathcal{N}(4,2)$}
     \label{figure: convergence rate in epochs dual vs isl:normal}
   \end{subfigure}
   \hfill
   \begin{subfigure}{0.45\textwidth}
     \centering
     \includegraphics[width=\linewidth]{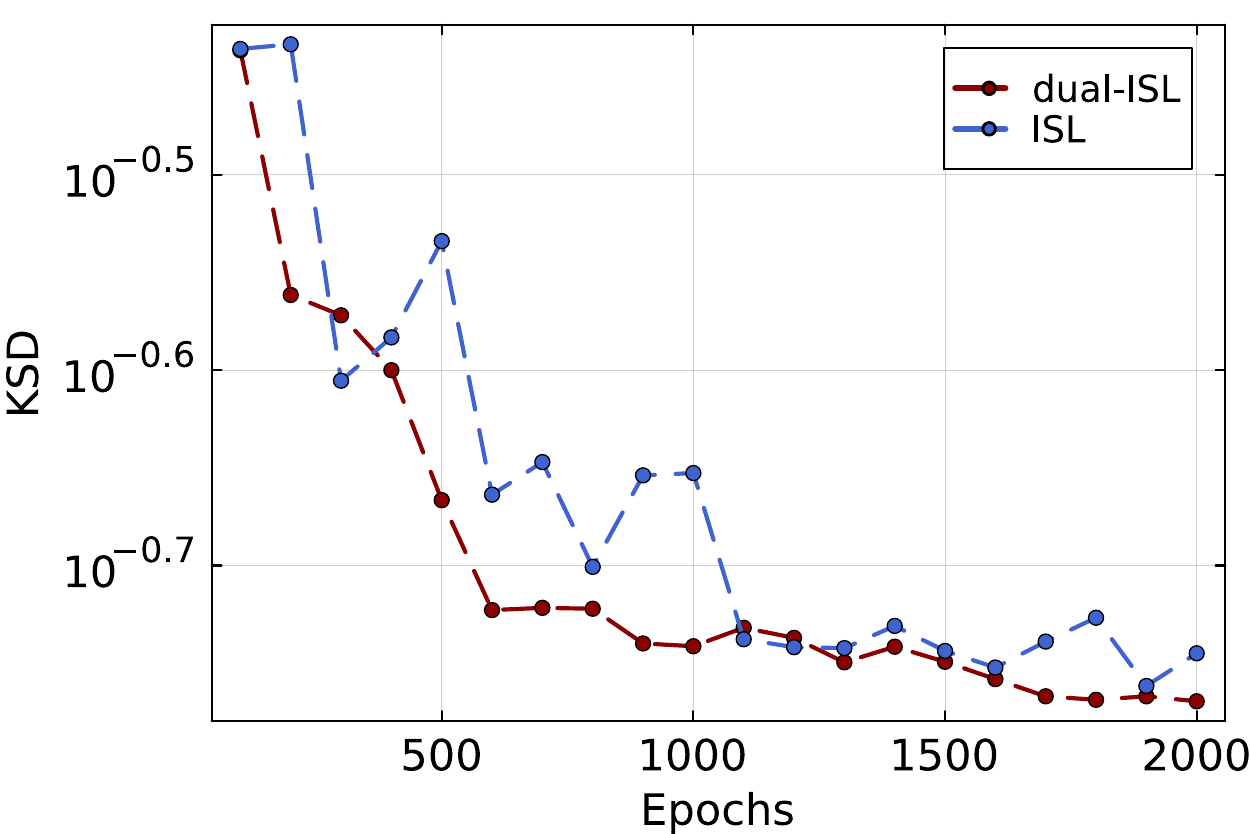}
     \caption{\small $\text{Model}_3$}
     \label{figure: convergence rate in epochs dual vs isl:mixture}
   \end{subfigure}
   \caption{\small Training curves for dual‐ISL versus classical ISL (dashed lines indicate mean over 10 runs).  Left: target \(\mathcal{N}(4,2)\). Right: target mixture Model\textsubscript{3} (mixture Pareto–Normal).}
   \label{figure: convergence rate in epochs dual vs isl}
\end{figure}

\subsection{Moment-Agnostic Optimal Transport via Monotonicity-Penalized ISL}\label{Moment-Agnostic Optimal Transport via Monotonicity-Penalized ISL}

Our ISL framework applies to \emph{any} probability law, including heavy‐tailed distributions lacking finite higher‐order moments.  Indeed, the rank statistic in Eq.~\ref{eq: rank statistic} is always well‐defined, whereas the classical \(p\)–Wasserstein distance is only finite when both distributions possess finite \(p\)-th order moments.

\paragraph{Unique transport in one dimension.}
In one dimension, any continuous map that pushes a simple base law \(p_z\) onto a target \(p\) must coincide with one of two inverses of the base cdf.  If \(F_z\) and \(F\) denote the cdfs of \(p_z\) and \(p\), then
\[
f_{+}(z) = F^{-1}\bigl(F_{z}(z)\bigr),
\quad
f_{-}(z) = F^{-1}\bigl(1 - F_{z}(z)\bigr).
\]
The monotone map \(f_{+}\) is in fact the unique optimal transport in \(\mathbb{R}\).  To recover this map, we augment our ISL loss with a \emph{monotonicity penalty}.

\paragraph{Monotonicity‐constrained training.}
Given a batch of inputs \(\{x_i\}_{i=1}^N\) sorted as
\[
x_{(1)} \le x_{(2)} \le \cdots \le x_{(N)},
\qquad
f_{\theta}(x_{(i)}) \;=\;\text{model output at }x_{(i)},
\]
we define
\[
\mathrm{Penalty}
\;=\;
\frac{1}{N}\sum_{i=1}^{N-1}
\max\bigl\{0,\;f_{\theta}(x_{(i)}) - f_{\theta}(x_{(i+1)})\bigr\},
\]
which is zero if and only if \(f_{\theta}\) is non‐decreasing.  The overall training objective becomes
\[
\mathcal{L}(p,\tilde p)
\;=\;
d_{K}(p,\tilde p)
\;+\;
\lambda\,
\frac{1}{N}\sum_{i=1}^{N-1}
\max\bigl\{0,\;f_{\theta}(x_{(i)}) - f_{\theta}(x_{(i+1)})\bigr\},
\]
where \(d_K(p,\tilde p)\) is our rank‐based discrepancy and \(\lambda>0\) weights the monotonicity constraint.  As \(\lambda\to\infty\), any violation of monotonicity incurs infinite cost, forcing \(f_{\theta}\) to converge to the unique optimal transport map \(f_{+}\).

This loss not only recovers the optimal transport in one dimension under minimal smoothness, but also extends beyond the Wasserstein framework to handle distributions with heavy tails. 

We evaluate three training objectives—Dual-ISL with monotonicity penalty, 1-Wasserstein, and 2-Wasserstein—on a suite of heavy-tailed target distributions using a five-layer MLP with ELU activations and layer widths \([16,16,32,32,16,1]\).  Each model is trained for \(10^4\) epochs via vanilla gradient descent with a fixed learning rate of \(10^{-2}\); we avoid adaptive optimizers to ensure that observed differences stem solely from the loss functions.

Performance is measured by two complementary metrics.
\begin{enumerate}
  \item \textbf{Kolmogorov--Smirnov distance (KSD):} the maximum absolute deviation between the empirical cdfs of the real pdf \(p\) and that of the generated distribution \(\tilde p\).
  \item \textbf{Tail-fit error} \(A_{\mathrm{CCDF}}\): the area between the log–log complementary cdfs of real and generated samples, defined for \(n\) data points by
  \[
    A_{\mathrm{CCDF}}
    = \sum_{i=1}^{n}
      \Bigl[
        \log\bigl(F_{p}^{-1}(i/n)\bigr)
        - \log\bigl(\tilde F_{\tilde p}^{-1}(i/n)\bigr)
      \Bigr]
      \,\log\!\bigl(\tfrac{i+1}{i}\bigr),
  \]
  where \(F_{p}^{-1}\) and \(\tilde F_{\tilde p}^{-1}\) are the inverse empirical ccdfs of \(p\) and \(\tilde p\), respectively.
\end{enumerate}

Figure~\ref{fig:loss-comparison} compares the Dual‐ISL loss (with monotonicity penalty) against the 1‐Wasserstein loss on a Pareto–Normal mixture.  Figure (a) shows that Dual‐ISL accurately recovers the true transport map despite the heavy Pareto tails.  Figure (b) demonstrates that the 1‐Wasserstein loss fails to learn the correct mapping under heavy‐tailed behavior.  Figure (c) plots both losses over 10 000 training epochs: the Wasserstein loss oscillates and does not converge, whereas the Dual‐ISL loss decreases smoothly and reliably, highlighting its stability and robustness.

\begin{figure}[htbp]
  \centering
  \captionsetup[subfigure]{justification=centering, font=small, labelfont=bf}
  \begin{subfigure}[b]{0.32\textwidth}
    \includegraphics[width=\linewidth]{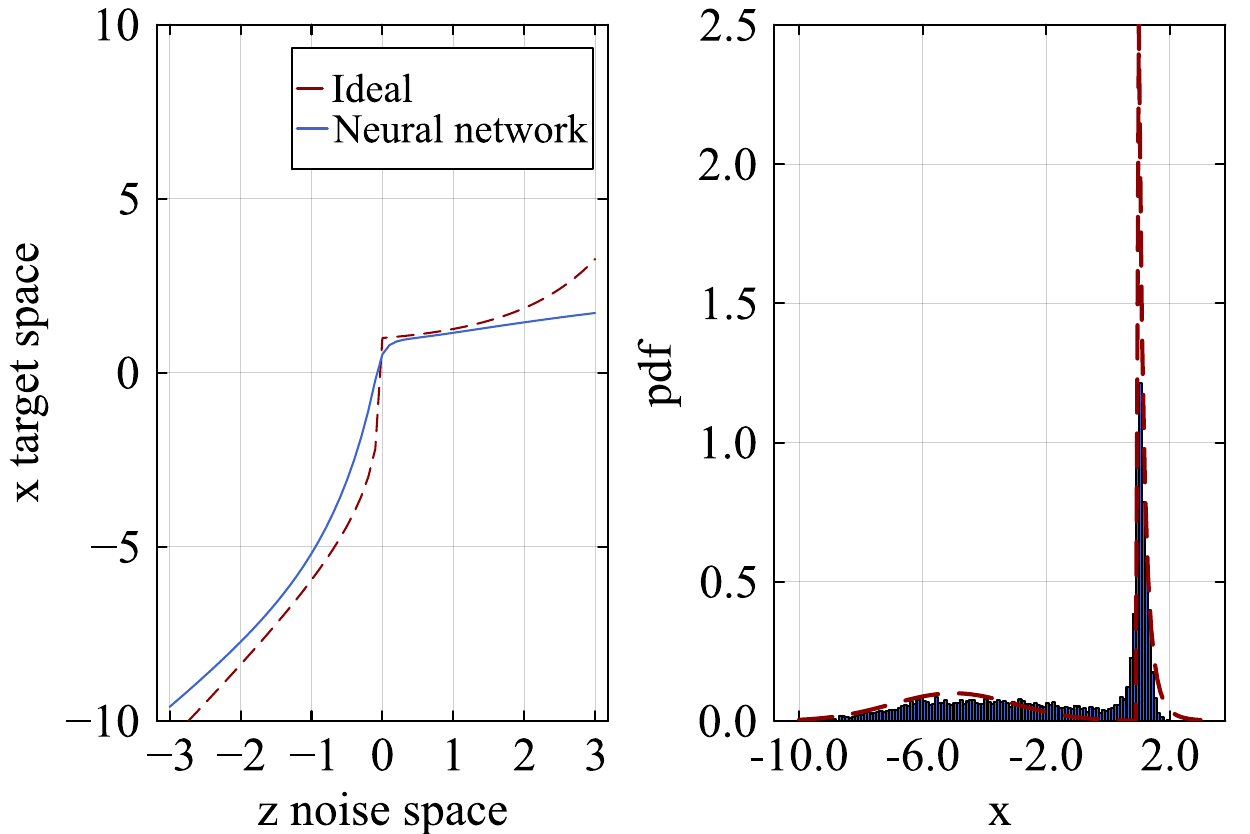}
    \caption{Dual‐ISL}
    \label{fig:dual-isl}
  \end{subfigure}
  \hfill
  \begin{subfigure}[b]{0.32\textwidth}
    \includegraphics[width=\linewidth]{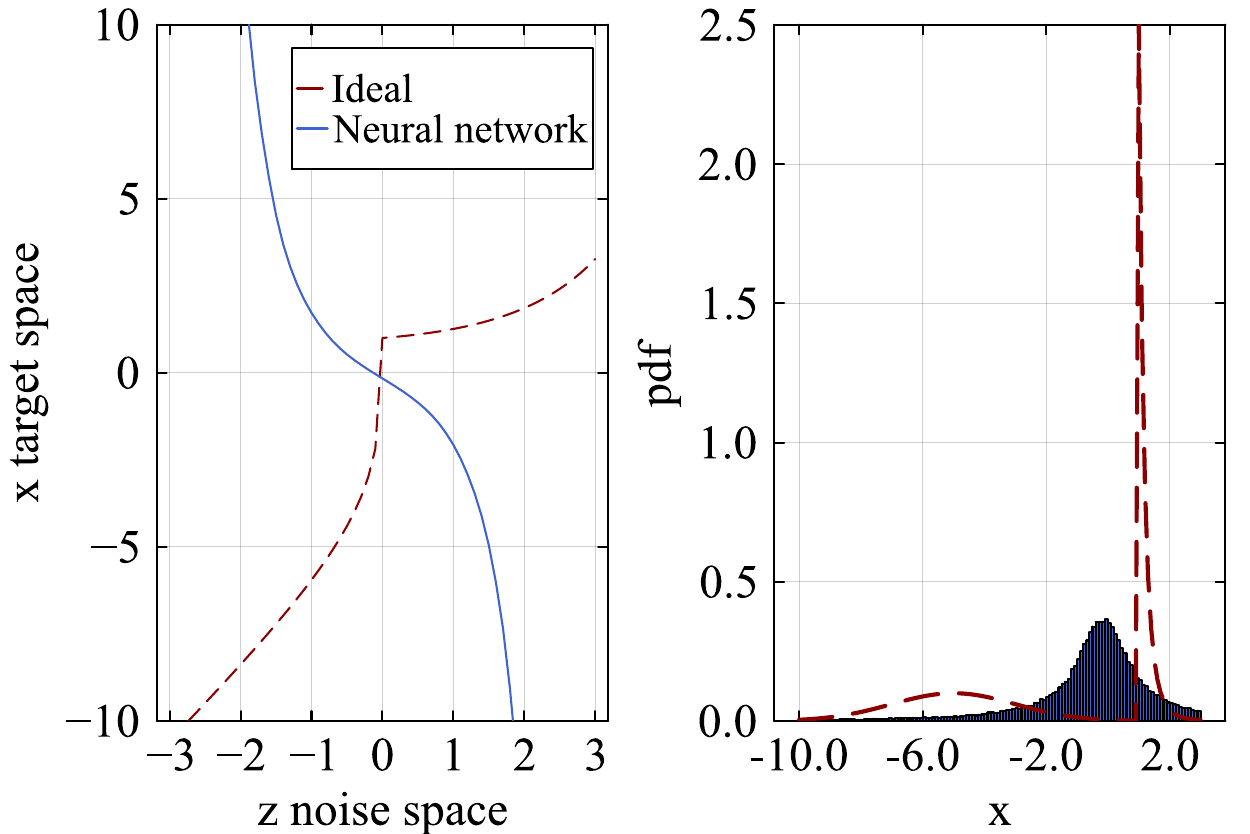}
    \caption{Wasserstein}
    \label{fig:wasserstein-loss}
  \end{subfigure}
  \hfill
  \begin{subfigure}[b]{0.32\textwidth}
    \includegraphics[width=\linewidth]{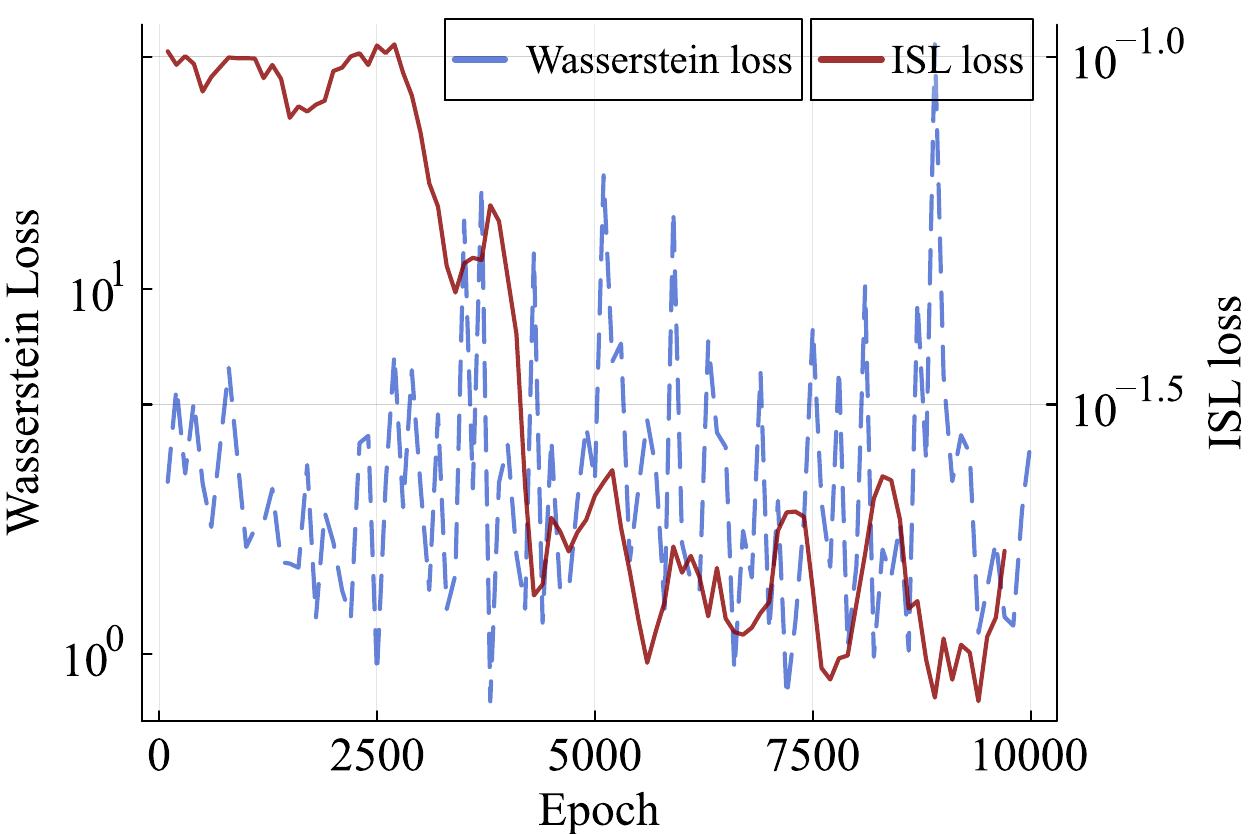}
    \caption{Optimal‐Transport}
    \label{fig:ot-loss}
  \end{subfigure}
  \vspace{3pt}
  \caption[Loss‐function comparison]{\small Comparison of transport objectives on a Pareto–Normal mixture: (a) Dual‐ISL with monotonicity penalty, (b) 1‐Wasserstein, and (c) Training dynamics over 10000 epochs, plotting Dual‐ISL loss (solid) and 1‐Wasserstein loss (dashed). The Dual‐ISL curve decreases smoothly and converges reliably, whereas the Wasserstein loss oscillates and does not settle.}
  \label{fig:loss-comparison}
\end{figure}

Table~\ref{tab:benchmarks_dual-isl_vs_wasserstein} shows that Dual-ISL (with monotonicity penalty) not only matches but frequently outperforms the classical OT baselines (1-Wasserstein and 2-Wasserstein) for \(K=10\).  For the moderately heavy-tailed \(\mathrm{Cauchy}(5,10)\), 1-Wasserstein attains the lowest KSD, yet Dual-ISL achieves a smaller \(A_{\mathrm{CCDF}}\), indicating superior tail alignment.  As the tail heaviness grows in \(\mathrm{Cauchy}(5,20)\) and \(\mathrm{Pareto}(1,1)\), Dual-ISL outperforms both OT metrics on both measures—note that 2-Wasserstein is undefined for Pareto due to its infinite second moment—underscoring Dual-ISL’s robustness where classical OT either diverges or loses precision.  Finally, on the multimodal \(\mathrm{Model}_3\), Dual-ISL yields the lowest KSD and \(A_{\mathrm{CCDF}}\), demonstrating that its rank-statistic formulation reliably recovers the unique monotone optimal-transport map in one dimension while avoiding the instability and non-differentiability of conventional OT losses.

\begin{table}[ht]
  \centering
  \makebox[\textwidth][c]{%
    \resizebox{1.0\textwidth}{!}{%
      \begin{tabular}{@{\extracolsep{\fill}}
        >{\columncolor{blue!10}}l 
        cc  cc  cc
      }
        \toprule
        \rowcolor{blue!5}
        \textbf{Target} 
          & \multicolumn{2}{c}{\textbf{Dual-ISL}}
          & \multicolumn{2}{c}{\textbf{1-Wasserstein}}
          & \multicolumn{2}{c}{\textbf{2-Wasserstein}} \\
        \cmidrule(lr){2-3}\cmidrule(lr){4-5}\cmidrule(lr){6-7}
        \rowcolor{gray!5}
          & \textbf{KSD} & \textbf{ACCDF}
          & \textbf{KSD} & \textbf{ACCDF}
          & \textbf{KSD} & \textbf{ACCDF} \\
        \midrule
        \(\mathrm{Cauchy}(5,10)\)  
          & \(0.069 \pm 0.059\) & \cellcolor{green!5}\(\mathbf{19.401 \pm 0.539}\)  
          & \cellcolor{green!5}\(\mathbf{0.037 \pm 0.024}\) & \(21.437 \pm 0.429\)  
          & \(0.504 \pm 0.267\) & \(28.678 \pm 7.018\) \\

        \(\mathrm{Cauchy}(5,20)\)  
          & \cellcolor{green!5}\(\mathbf{0.045 \pm 0.010}\) & \cellcolor{green!5}\(\mathbf{17.201 \pm 0.423}\)  
          & \(0.046 \pm 0.014\) & \(21.123 \pm 1.290\)  
          & \(0.668 \pm 0.168\) & \(42.398 \pm 11.578\) \\

        \(\mathrm{Pareto}(1,1)\)  
          & \cellcolor{green!5}\(\mathbf{0.120 \pm 0.053}\) & \cellcolor{green!5}\(\mathbf{21.769 \pm 0.072}\)  
          & \(0.240 \pm 0.064\) & \(23.676 \pm 0.072\)  
          & \(0.950 \pm 0.158\) & \cellcolor{yellow!5}{--} \\

        \(\mathrm{Model}_3\)      
          & \cellcolor{green!5}\(\mathbf{0.131 \pm 0.017}\) 
          & \cellcolor{green!5}\(\mathbf{19.406 \pm 1.303}\)  
          & \(0.148 \pm 0.097\) & \(31.561 \pm 2.769\)  
          & \(0.219 \pm 0.046\) & \(22.031 \pm 1.387\) \\
        \bottomrule
      \end{tabular}
    }%
  }
  \caption{\small KSD and ACCDF (mean \(\pm\) std) for Dual-ISL (with monotonicity penalty) vs.\ Wasserstein baselines at \(K=10\).}
  \label{tab:benchmarks_dual-isl_vs_wasserstein}
\end{table}

\subsection{Empirical proof of the convergence rate}

Our aim is to empirically validate Equation \ref{eq:convergence of qK to 1}. To do this, we train a neural network using the same architecture as in our previous experiment. The network is fed an input pdf \(\normdist{0}{1}\) and is tasked with approximating a target pdf defined as a mixture of Cauchy distributions. We estimate \(\tilde{p}\) via a kernel density estimator and compute the second derivative of \(q\) using central finite differences with a sixth-order expansion. Each experiment is repeated $10$ times, and the mean results are shown in Figures \ref{fig:empirical-convergence-rate_appendix}.


\begin{figure}[htbp]
  \centering
  \captionsetup[subfigure]{justification=centering, font=small, labelfont=bf}
  \captionsetup{font=small, labelfont=bf, skip=4pt}
  \begin{subfigure}[t]{0.48\linewidth}
    \centering
    \includegraphics[width=\linewidth,keepaspectratio]{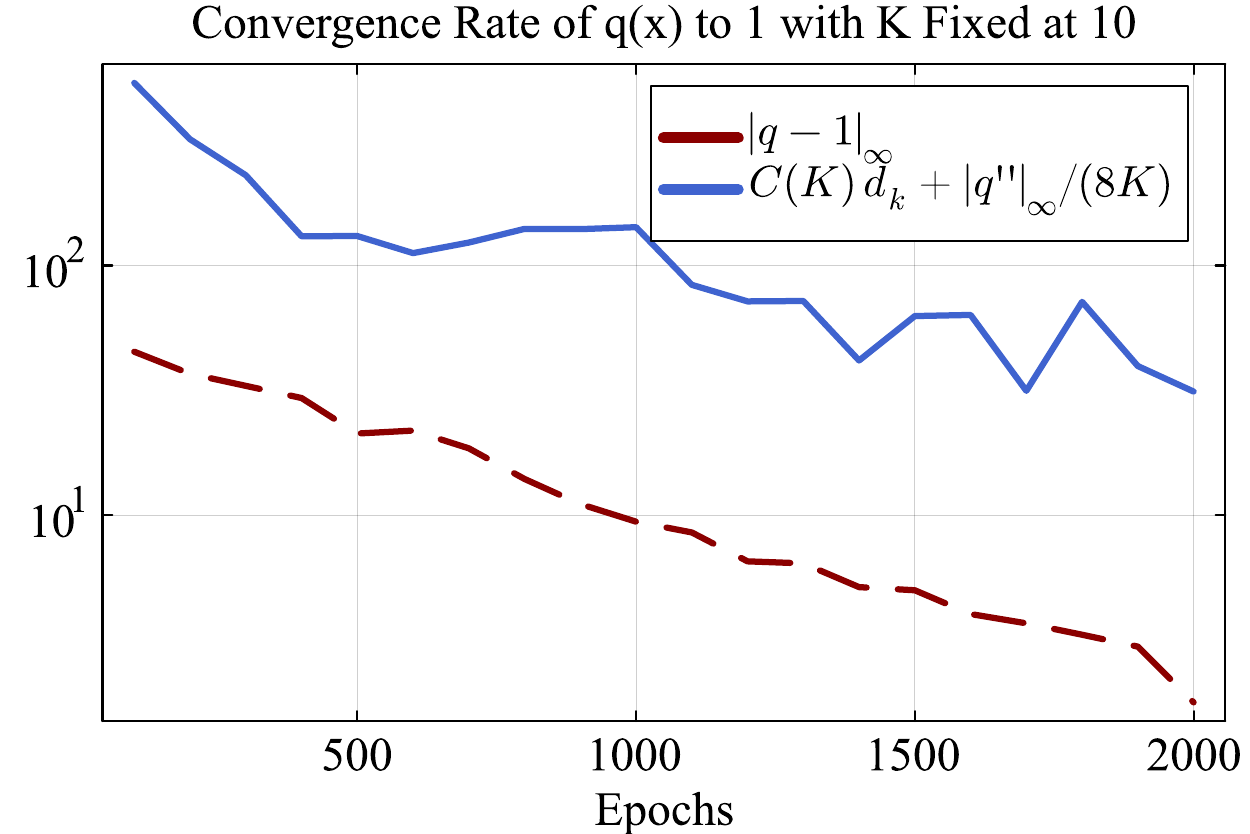}
    \caption{\small Fixed \(K=10\). NN trained with \(N=1000\) samples, lr \(10^{-3}\).}
    \label{fig:conv-rate-fixed-K}
  \end{subfigure}
  \hfill
  \begin{subfigure}[t]{0.48\linewidth}
    \centering
    \includegraphics[width=\linewidth,keepaspectratio]{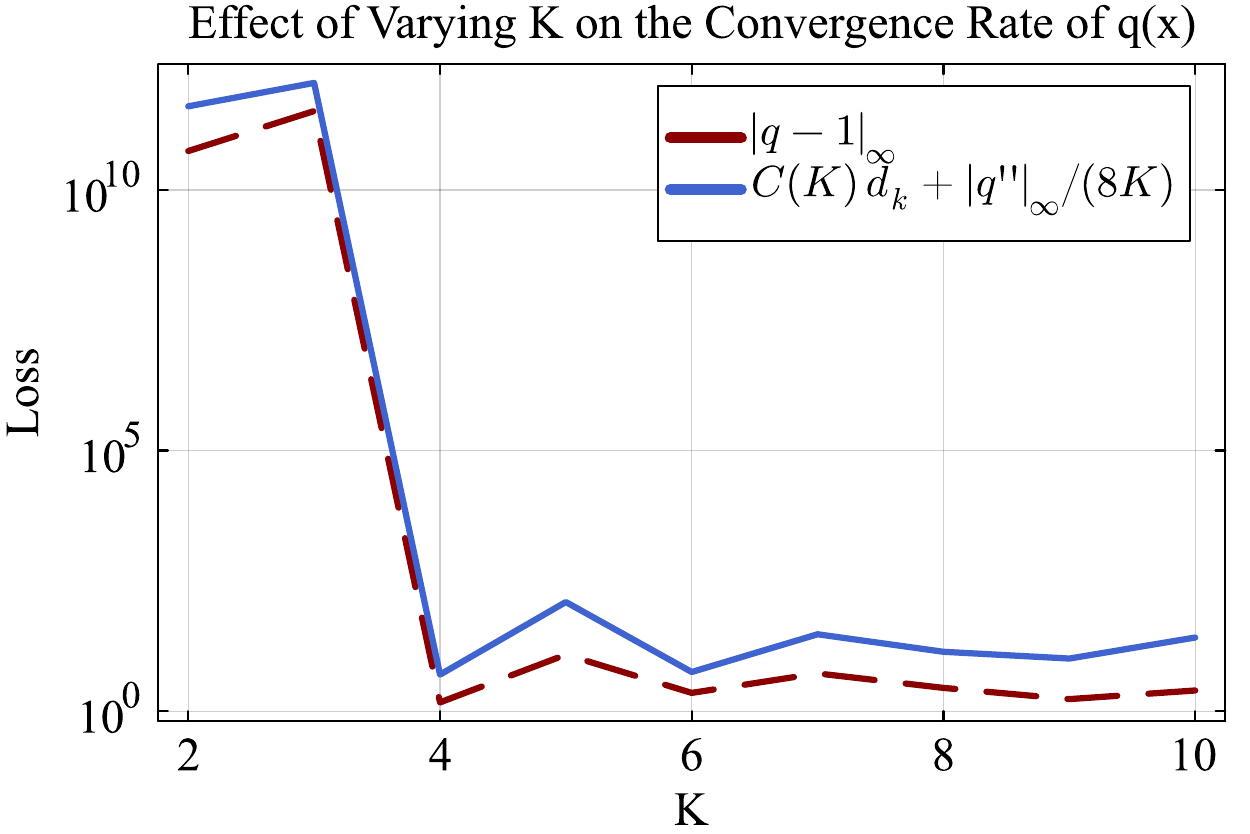}
    \caption{\small Varying \(K\). Each run uses 1000 epochs, \(N=1000\), lr \(10^{-3}\).}
    \label{fig:conv-rate-vary-K}
  \end{subfigure}

  \vspace{2pt} 
  \caption{\small Empirical convergence of dual-ISL’s Bernstein approximation (cf.\ Eq.~\ref{eq:convergence of qK to 1}).  
    The solid blue curve shows the mean theoretical upper bound \(\|q_K - 1\|_\infty\le (K+1)^3d_K\), and the dashed red curve shows the observed \(\|q - 1\|_\infty\).}
  \label{fig:empirical-convergence-rate_appendix}
\end{figure}

\subsection{Density estimation} \label{appendix:Density estimation}

\subsubsection{1D density estimation}

We employ the same fully‐connected NN and training hyperparameters as in Appendix~\ref{appendix: 1D experiments}.  Once training converges, we approximate the implicit density with Equation \ref{eq:pK estimation}. We restate the latter here for convenience
\begin{align*}
  \label{eq:pK-estimation}
  p_K(x)
  &= \widehat{\tilde p}(x)
     \sum_{m=0}^K \mathbb{Q}_K(m)\,\tilde b_{m,K}\bigl(\widehat{\tilde F}(x)\bigr)\,,
\end{align*}
where
\[
  \mathbb{Q}_K(m) \;=\; \Pr\bigl(A_K = m\bigr), 
  \quad
  \tilde b_{m,K}(u)
  = \binom{K}{m} u^m (1 - u)^{\,K-m}\,.
\]

The computation proceeds in three steps

\begin{enumerate}
  \item \textbf{Monte Carlo estimation of weights} \\
    For each \(m = 0,\dots,K\), estimate
    \[
      \mathbb{Q}_K(m)
      = \Pr\bigl(\#\{\,x_i \le x\}=m\bigr)
    \]
    by sampling \(K\) independent latent vectors \(z_i\sim\mathcal{N}(0,I)\), computing \(x_i = f_{\theta}(z_i)\), and counting how many satisfy \(x_i \le x\).  Repeat \(M\) times and take empirical frequencies:
    \[
      \widehat{\mathbb{Q}}_K(m)
      = \frac{1}{M}\sum_{j=1}^M \mathbf{1}\Bigl(\#\{x_i^{(j)} \le x\}=m\Bigr).
    \]

  \item \textbf{Empirical cdf estimation} \\
    Draw \(N\) samples \(\{x_i\}_{i=1}^N\) from the trained generator,
    sort them in ascending order, and form the empirical cdf
    \[
      \widehat{\tilde F}(x)
      = \frac{1}{N} \sum_{i=1}^N \mathbf{1}\{x_i \le x\}.
    \]

  \item \textbf{Finite‐difference density} \\
    Approximate the density of the push‐forward distribution by a first‐order finite difference:
    \[
      \widehat{\tilde p}(x)
      \approx \frac{\widehat{\tilde F}(x + \Delta) - \widehat{\tilde F}(x)}{\Delta},
      \qquad \Delta \ll 1.
    \]
\end{enumerate}

We evaluate Dual‐ISL on six univariate target distributions.  For each target, we estimate the mixture weights \(\widehat{\mathbb{Q}}_{K}\) via \(10^{5}\) Monte Carlo trials, compute the empirical cdf \(\widehat{\tilde F}(x)\) from \(10^{5}\) samples drawn from the trained generator, and form the density estimate using a first‐order finite difference with \(\Delta = 0.1\). Table \ref{Learning1D_table_densities_estimation} reports the average Kolmogorov–Smirnov distance (over 10 independent runs) for the estimated density using Equation \ref{eq:pK estimation} with $K\in\{2,5,10\}$ versus a Gaussian kernel density estimator using the same number of samples and Silverman’s rule for bandwidth selection. Figure \ref{fig:density-estimation-1d-distributions} illustrates density estimates for three representative targets—true density (solid red), ISL estimate (dashed blue), and KDE (dotted green)—all plotted on common axes to facilitate direct comparison of bias and tail behavior. As shown, Dual-ISL achieves the best results for every target except the Gaussian case, in which KDE with a Gaussian kernel performs marginally better.

\begin{table}[h]
  \centering
  \small                           
  \setlength{\tabcolsep}{3pt}      
  \rowcolors{2}{gray!10}{white}
    \begin{tabular}{lcccc}
      \toprule
      \textbf{Target} 
        & \textbf{Dual‐ISL (K=2)}
        & \textbf{Dual‐ISL (K=5)}
        & \textbf{Dual‐ISL (K=10)} 
        & \textbf{KDE} \\
      \midrule
      \(\mathcal{N}(4,2)\)  
        & \(0.0202\)
        & \(0.0178\)
        & \(0.0167\)
        & \(\mathbf{0.0110}\)\\
      Cauchy\((1,2)\)    
        & \(0.0237\) 
        & \(0.0253\)
        & \(\mathbf{0.0184}\)
        & \(0.2013\)\\
      Pareto\((1,1)\)    
        & \(0.0302\) 
        & \(\mathbf{0.0203}\)
        & \(0.0252\)
        & \(0.3872\)\\
      Mixture\textsubscript{1} 
        & \(0.0395 \)
        & \(\mathbf{0.0095}\)
        & \(0.0120 \)
        & \(0.0156\)\\
      Mixture\textsubscript{2}     
        & \(0.0171\) 
        & \(0.0167\) 
        & \(\mathbf{0.0070}\)
        & \(0.0145\)\\
      Mixture\textsubscript{3}        
        & \(0.1853\) 
        & \(0.1786\)
        & \(\mathbf{0.0741}\)
        & \(0.1644\)\\
      \bottomrule
    \end{tabular}
  \caption{\small Mean Kolmogorov–Smirnov distance (over 10 runs) for Dual‐ISL versus Gaussian KDE.  We used \(\widehat{\mathbb{Q}}_{K}\) estimated with \(10{,}000\) trials, the empirical CDF \(\widehat{\tilde F}\) from \(10{,}000\) samples, and a finite‐difference step \(\Delta=0.1\).}
  \label{Learning1D_table_densities_estimation}
\end{table}

\begin{figure}[!htbp]
  \captionsetup[subfigure]{justification=centering, font=small, labelfont=bf, skip=1pt}
  \centering
  \begin{adjustbox}{width=\textwidth}
    \begin{tabular}{@{} c c c @{}}
      \subcaptionbox{Cauchy}{%
        \includegraphics[width=0.32\textwidth]{img/density_estimation/density_estimation_kde_isl_cauchy.pdf}%
      }
      &
      \subcaptionbox{Model$_2$}{%
        \includegraphics[width=0.32\textwidth]{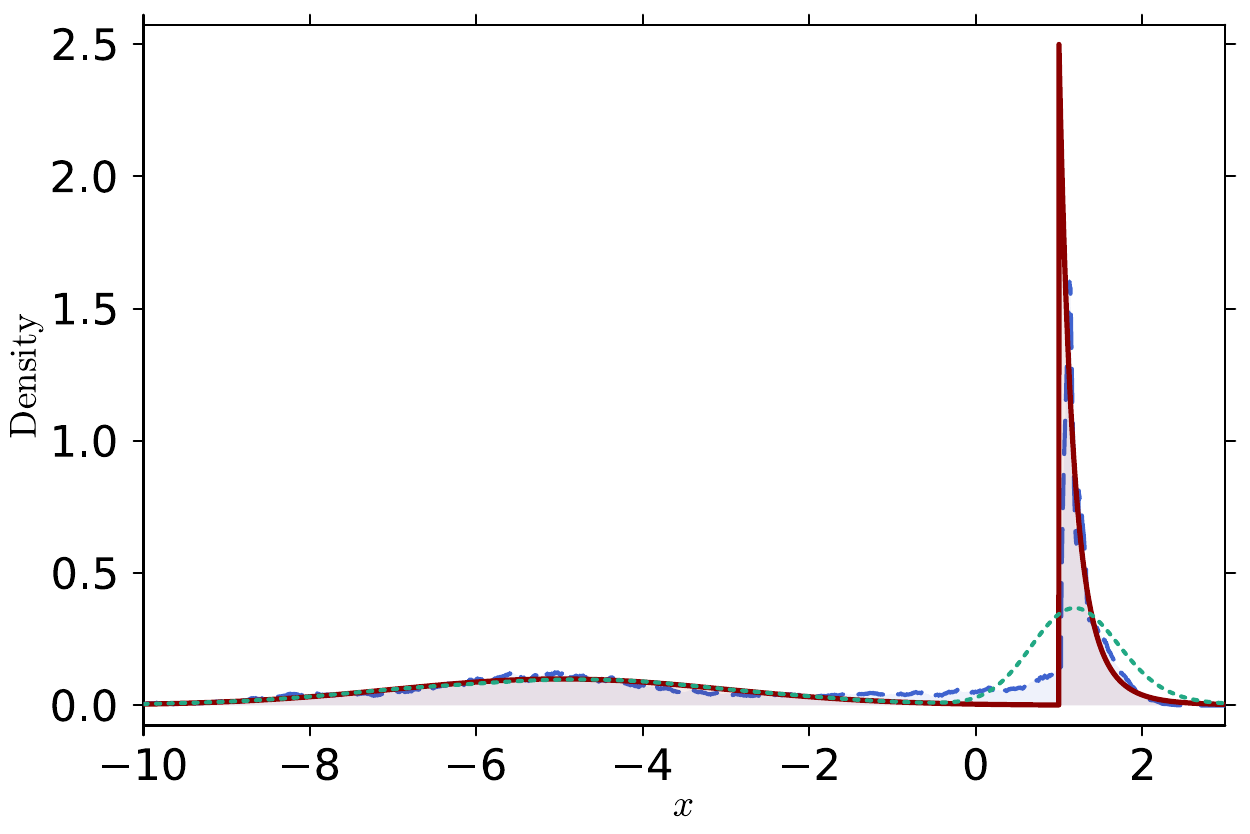}%
      }
      &
      \subcaptionbox{Model$_3$}{%
        \includegraphics[width=0.32\textwidth]{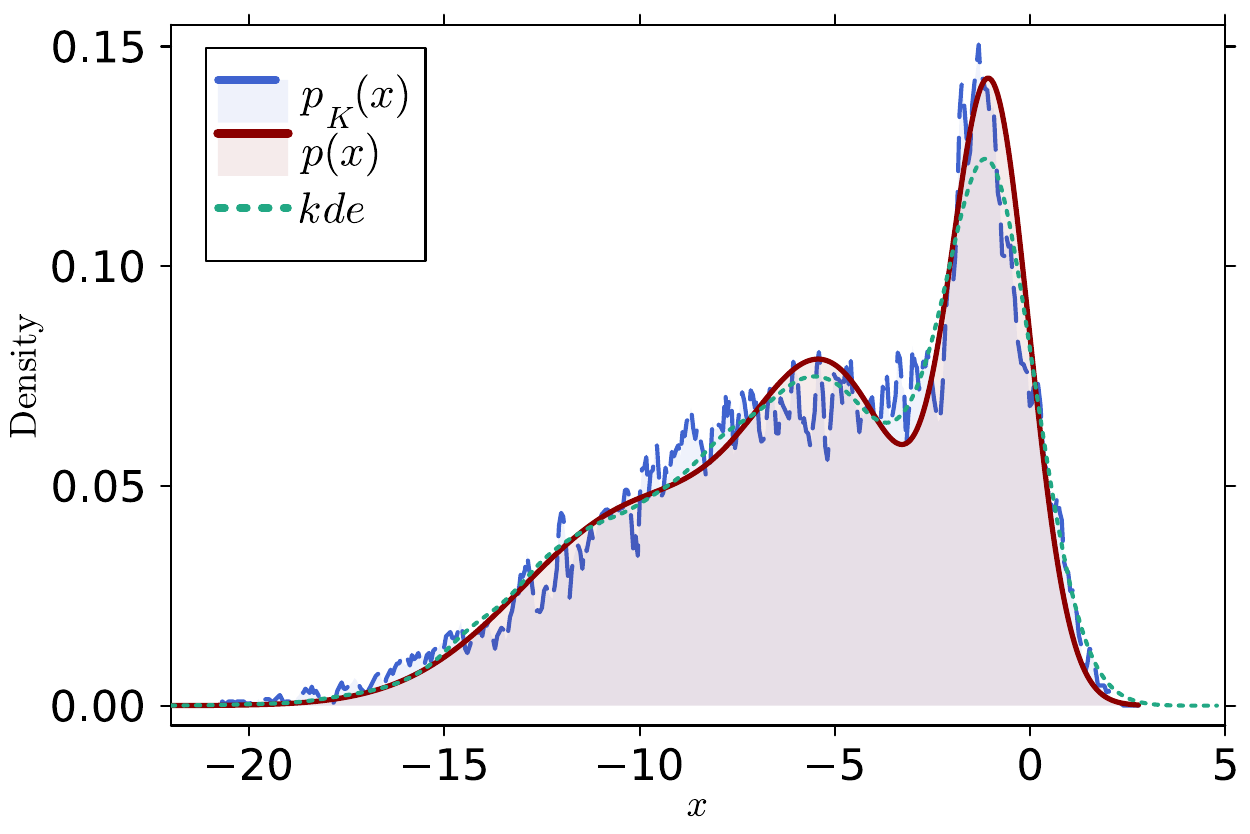}%
      }
    \end{tabular}
  \end{adjustbox}
\caption{Density estimates for three univariate distributions using the ISL estimator (blue dashed line) and Gaussian kernel density estimator (KDE; green dotted line). Figure~(a) shows the true Cauchy target density, while Figures~(b) and~(c) display the corresponding ISL and KDE estimates for Model$_2$ and Model$_3$, respectively. All panels share identical axes to facilitate direct comparison of estimator bias and tail behavior.}
\label{fig:density-estimation-1d-distributions}
\end{figure}

\subsubsection{2D density estimation}

To generalize our implicit estimator to data in \(\mathbb{R}^d\), let \(\{s_\ell\}_{\ell=1}^m\subset\mathbb S^d\) be random unit vectors.  Denote by 
\[
(s\#\hat p)(y)
\;=\;
\text{ISL-estimate of the 1D pdf at }y=s^\top x
\]
the push-forward density along \(s\).  Then for any query point \(x\in\mathbb R^d\),
\begin{align*}
\hat p(x)
&=\frac{1}{m}\sum_{\ell=1}^m
\bigl(s_\ell\#\hat p\bigr)\bigl(s_\ell^\top x\bigr).
\end{align*}
This Monte Carlo slicing—averaging one-dimensional ISL estimates—yields a consistent multivariate density approximation without ever constructing a full \(d\)-dimensional kernel.

In Figure~\ref{fig:2d-comparison density estimation}, we compare the two‐dimensional density estimates produced by our sliced Dual‐ISL method against a Gaussian kernel density estimator with Silverman’s rule for bandwidth selection.  Across all experiments, the generator is a four‐layer MLP that maps 2D standard normal noise to the data space, with each hidden layer comprising 32 \(\tanh\) units.  We train for 100 epochs using 1000 samples per epoch, and approximate the sliced discrepancy by averaging over \(m=10\) random projections with \(K=10\).  For the KDE baseline, we use 10000 samples to construct each density estimate.

\begin{figure}[!htbp]
  \centering
  \setlength{\tabcolsep}{2pt}
  \renewcommand{\arraystretch}{1.05}
  \begin{tabular}{@{}rc|ccc@{}}
    \toprule
    \rowcolor{gray!30}
    & \textbf{ISL}
    & \textbf{KDE}
    \\
    \midrule
    \rotatebox{90}{\textbf{Dual Moon}}
      & \includegraphics[width=2.8cm,height=2.8cm]{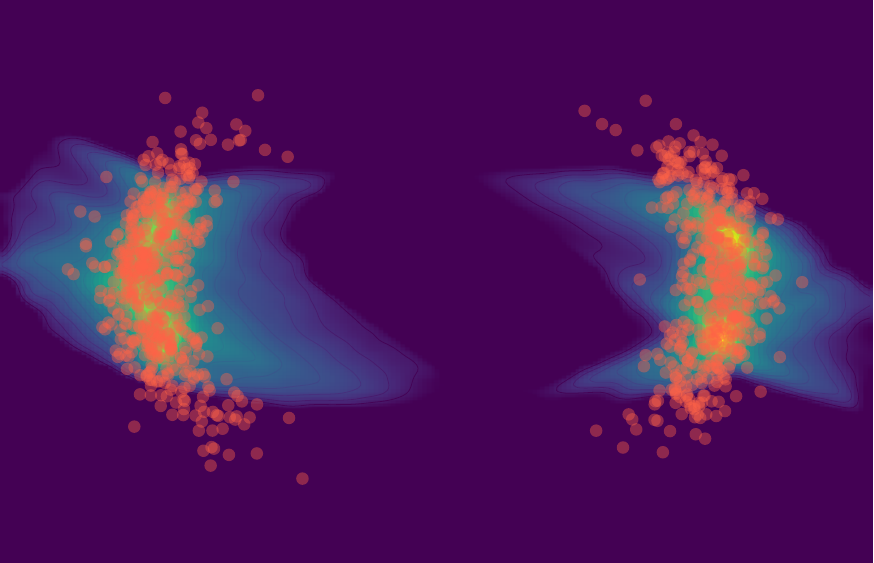}
      & \includegraphics[width=2.8cm,height=2.8cm]{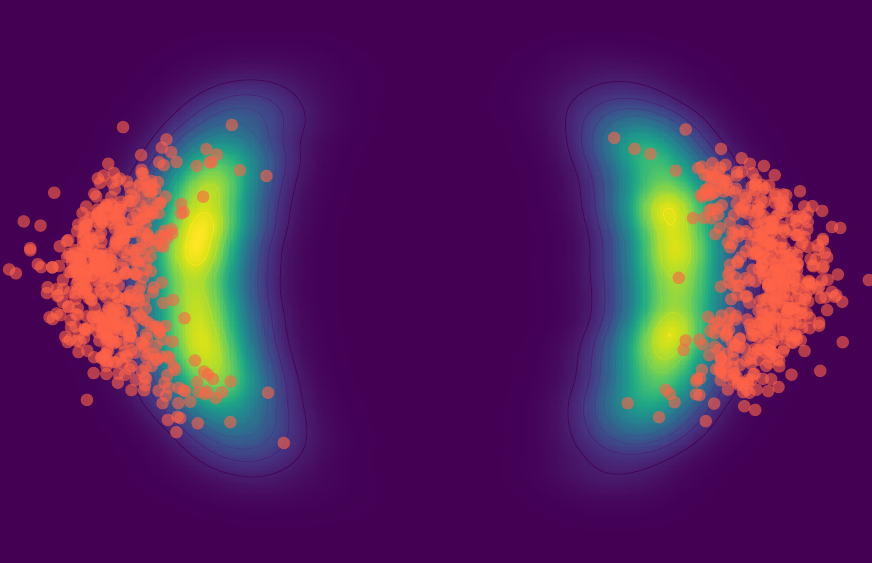}
      \\[2pt]
    \bottomrule
  \end{tabular}
  \caption{\small Two‐dimensional density estimates on synthetic Dual Moon dataset.  Left column: sliced Dual‐ISL (averaged over $m=10$ random projections with $K=10$); right column: Gaussian KDE using Silverman’s bandwidth rule.  The generator is a four‐layer MLP with 32 units per hidden layer, trained for 100 epochs with 1000 samples per epoch; the KDE baseline uses 10000 samples.}
  \label{fig:2d-comparison density estimation}
\end{figure}




\subsection{Experiments on 2D distributions}

We begin by evaluating three synthetic 2D benchmarks, each with a distinct topology:
\begin{enumerate}
  \item \textbf{Dual Moon}: A bimodal Gaussian mixture (two disconnected modes).
  \item \textbf{Circle Gaussian}: An eight-component Gaussian mixture arranged in a circle.
  \item \textbf{Two Ring}: A “double-ring” distribution consisting of two concentric circular supports.
\end{enumerate}
Our aim is to test whether the sliced dual-ISL method can fully recover these supports, including disconnected or non-convex regions.  We compare against both normalizing flows and GANs, measuring performance by
\begin{itemize}
  \item \textbf{KL divergence} between the learned and true densities, and
  \item \textbf{Visual coverage} of the true support via sample plots.
\end{itemize}

In all experiments—GAN, WGAN, and dual-ISL- the generator was a four‐layer MLP that maps two‐dimensional standard normal noise into data‐space samples, with each hidden layer comprising 32 units and tanh activations.  For the GAN and WGAN variants, the discriminator (or critic) adopted a similar four‐layer MLP but with 128‐unit hidden layers using ReLU activations and a final sigmoid output.  Every model was trained for 1 000 epochs with a batch size of 1000 under the Adam optimizer.  In the adversarial setups, we swept the critic-to-generator update ratio from 1:1 to 5:1 and chose the learning rate from $\{10^{-2},10^{-3},10^{-4},10^{-5}\}$.  For ISL, we fix $K=10$, drew $N=1000$ samples per projection, averaged over $L=10$ random projections, and used a constant learning rate of $10^{-3}$.

For the normalizing flow baseline we adopted the RealNVP architecture of \cite{dinh2016density}, consisting of four affine‐coupling layers whose scale and translation nets are two-hidden-layer MLPs with 32 ReLU units each.  RealNVP was trained under the same 1 000-epoch, batch-size-1 000 protocol, but with a fixed learning rate of $5\times10^{-5}$ as in \cite{Stimper2023}.

Figure \ref{fig:2d-comparison} exposes the fundamental limitations of popular generative models on complex, multimodal data.  GANs frequently collapse to a subset of modes, omitting entire regions of the true support.  Normalizing flows, by enforcing invertibility, preserve the topology of the base distribution but struggle to represent disconnected clusters, instead “bridging” them with thin density filaments.  In contrast, sliced Dual-ISL accurately recovers each connected component of the support—exemplified by the clean separation of the two moons—while still covering the full data manifold.  On the Circle-of-Gaussians task, however, a low projection order $K$ can allow leakage between rings; raising $K$ sharpens the estimate and eliminates this spillover.  A hybrid strategy (“Dual-ISL + GAN”), in which a model is first trained with Dual-ISL for 100 epochs and then fine-tuned adversarially, combines the stability and full-support coverage of Dual-ISL with the precision of a GAN, yielding the most faithful reconstructions.  Table \ref{table:2d_experiments} reports KL divergences between each model and the true distribution, confirming that both Dual-ISL and especially the Dual-ISL + GAN variant outperform all baselines.

\begin{figure}[!htbp]
  \centering
  \setlength{\tabcolsep}{2pt}
  \renewcommand{\arraystretch}{1.05}
  \begin{tabular}{@{}rc|ccccc@{}}
    \toprule
    \rowcolor{gray!30}
    & \textbf{Target}
    & \textbf{Real NVP}
    & \textbf{WGAN}
    & \textbf{ISL}
    & \textbf{dual-ISL}
    & \textbf{dual-ISL+GAN}
    \\
    \midrule
    \rotatebox{90}{\textbf{Dual Moon}}
      & \includegraphics[width=1.8cm,height=1.8cm]{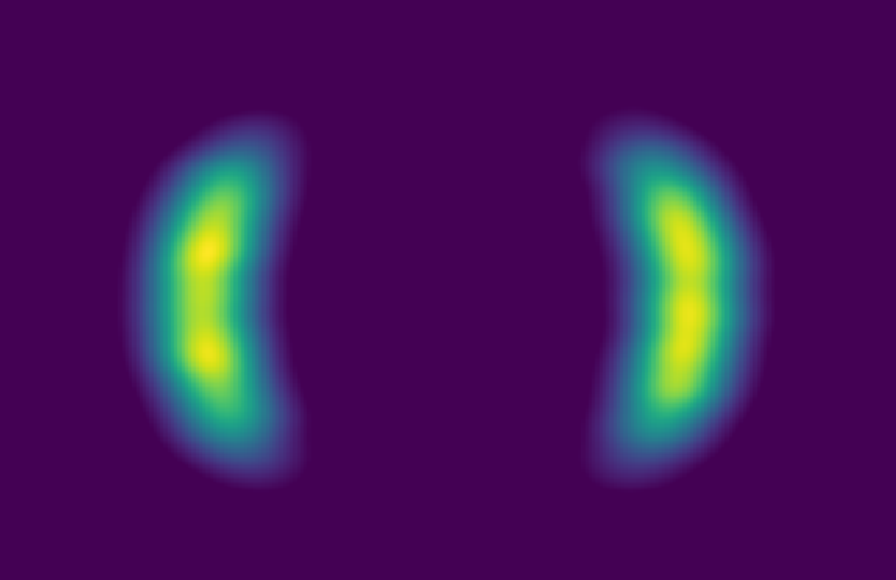}
      & \includegraphics[width=1.8cm,height=1.8cm]{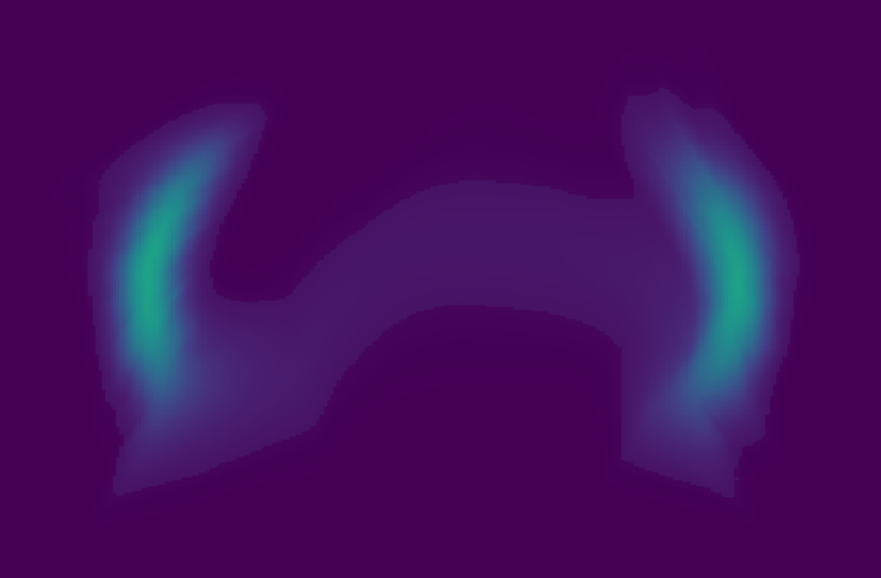}
      & \includegraphics[width=1.8cm,height=1.8cm]{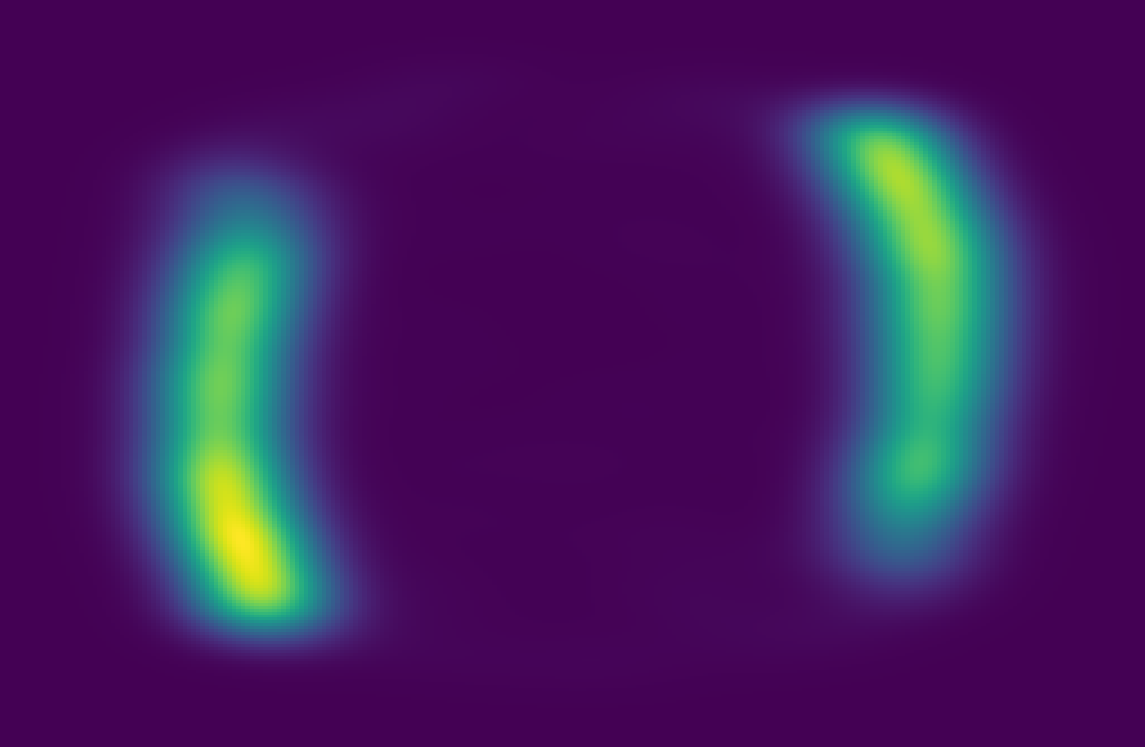}
      & \includegraphics[width=1.8cm,height=1.8cm]{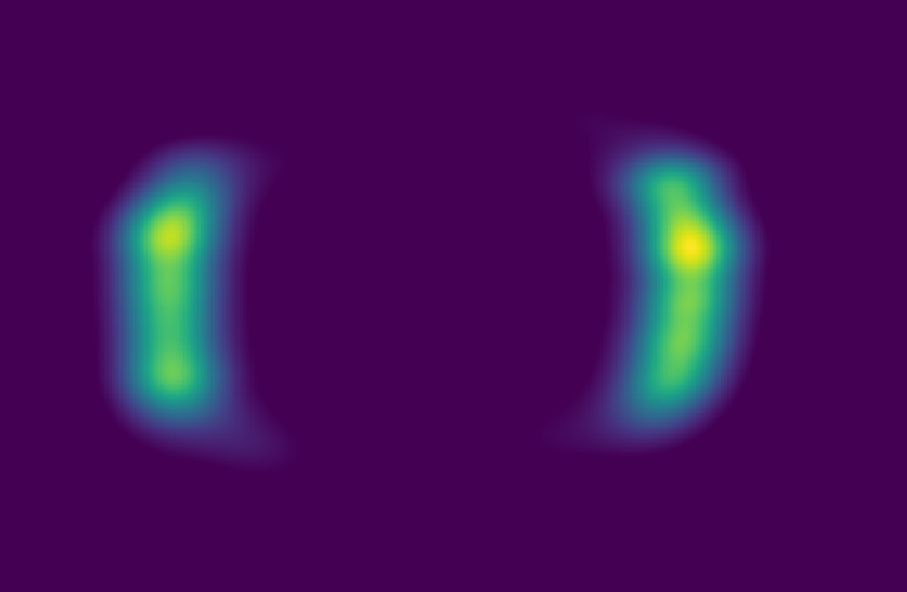}
      & \includegraphics[width=1.8cm,height=1.8cm]{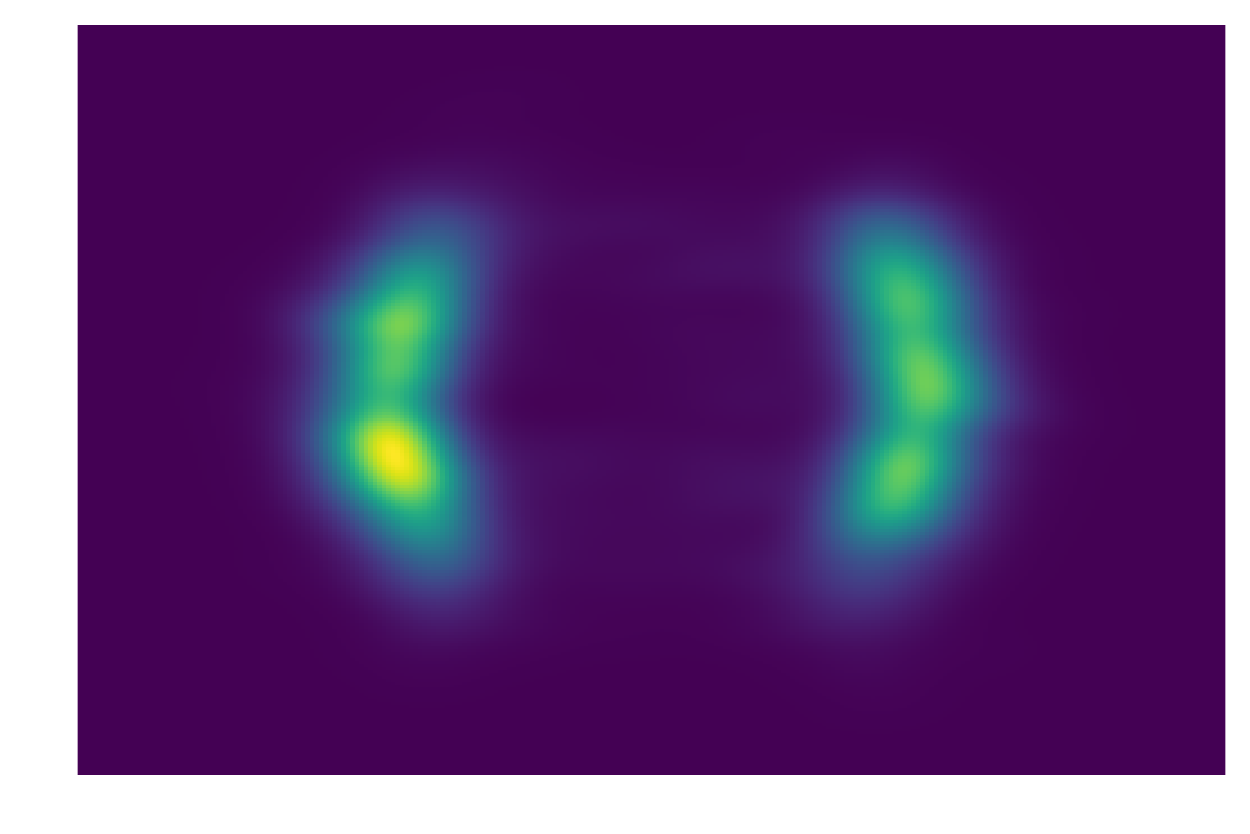}
      & \includegraphics[width=1.8cm,height=1.8cm]{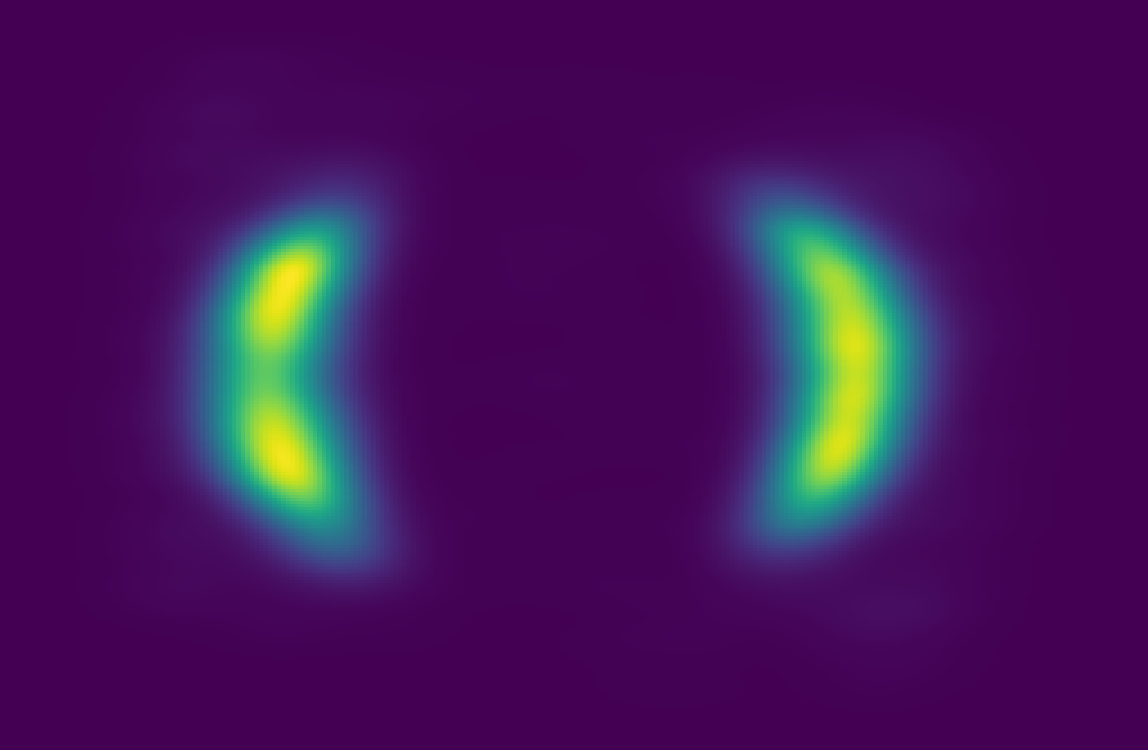}
    \\[2pt]
    \rotatebox{90}{\textbf{Circle Gauss}}
      & \includegraphics[width=1.8cm,height=1.8cm]{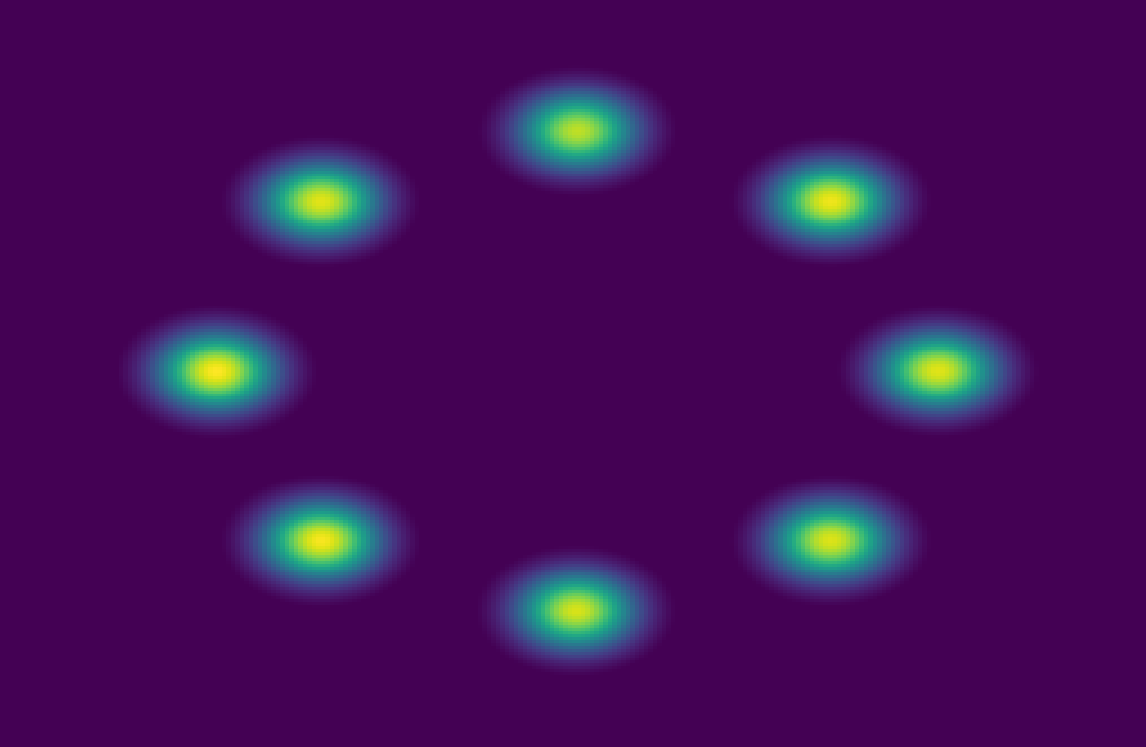}
      & \includegraphics[width=1.8cm,height=1.8cm]{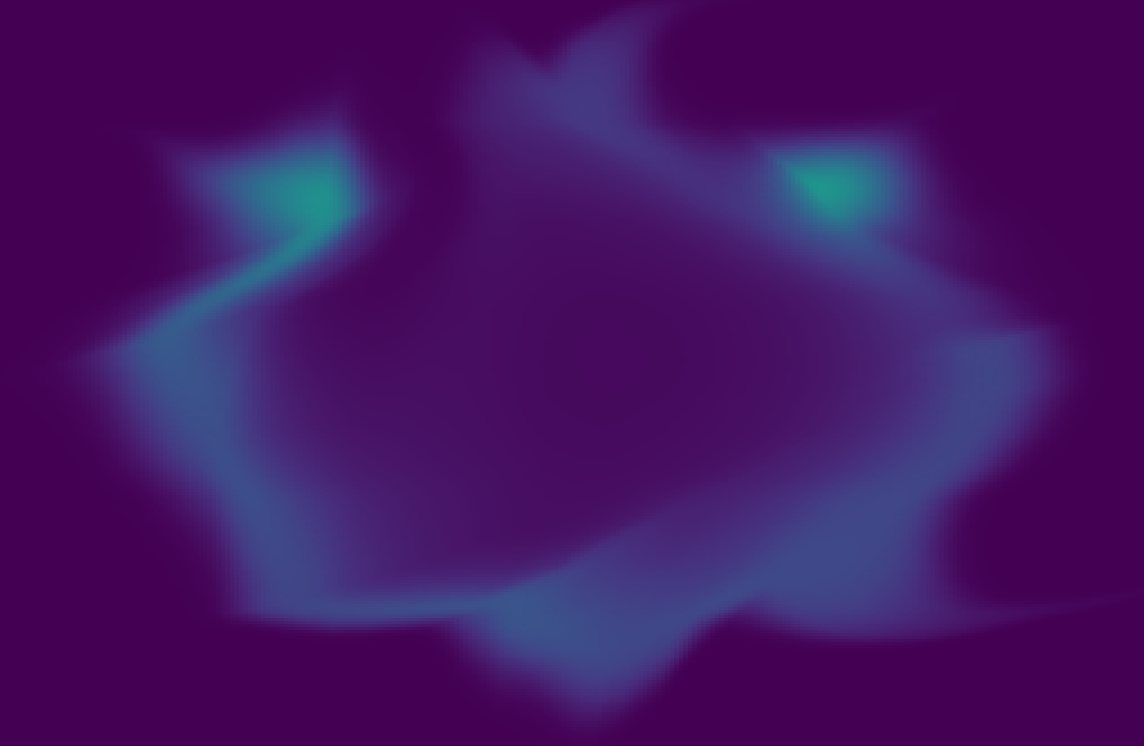}
      & \includegraphics[width=1.8cm,height=1.8cm]{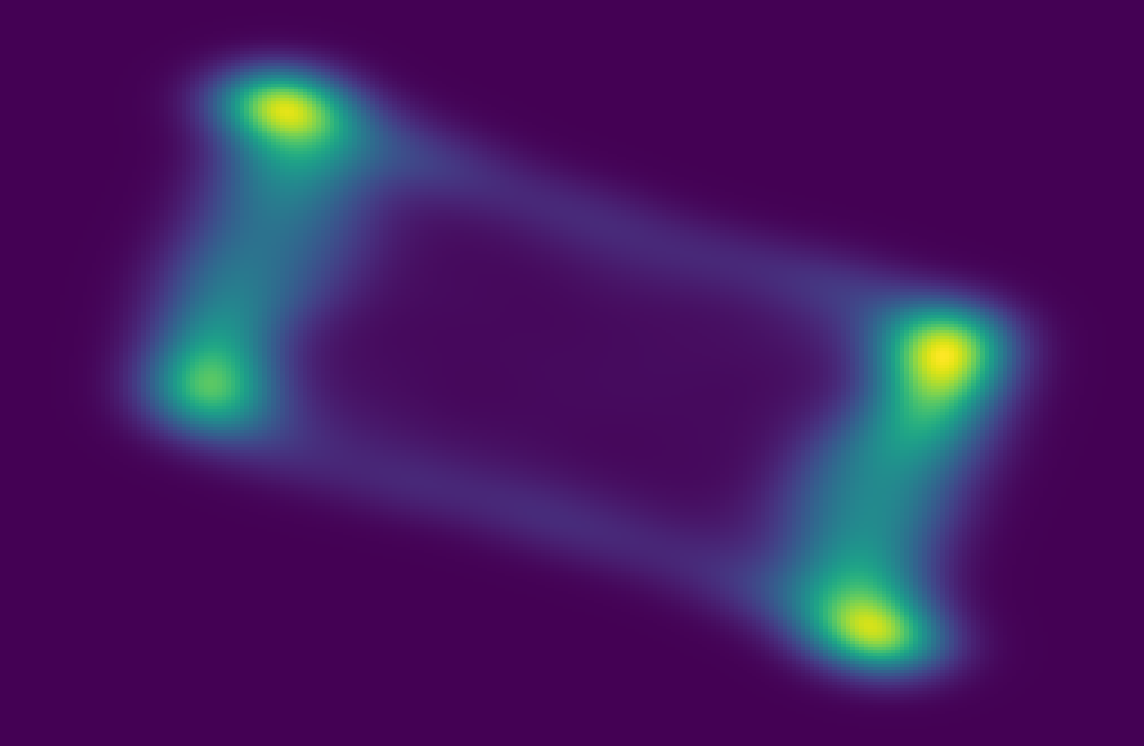}
      & \includegraphics[width=1.8cm,height=1.8cm]{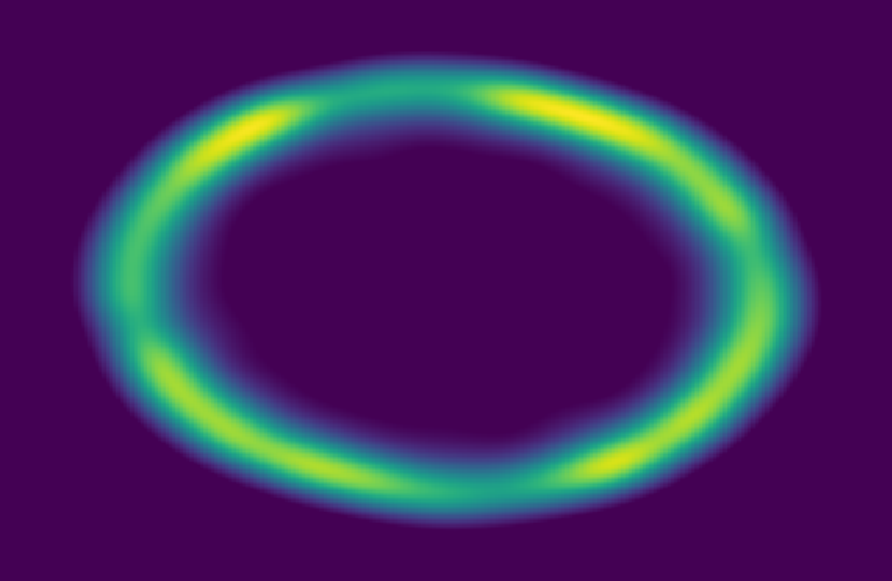}
      & \includegraphics[width=1.8cm,height=1.8cm]{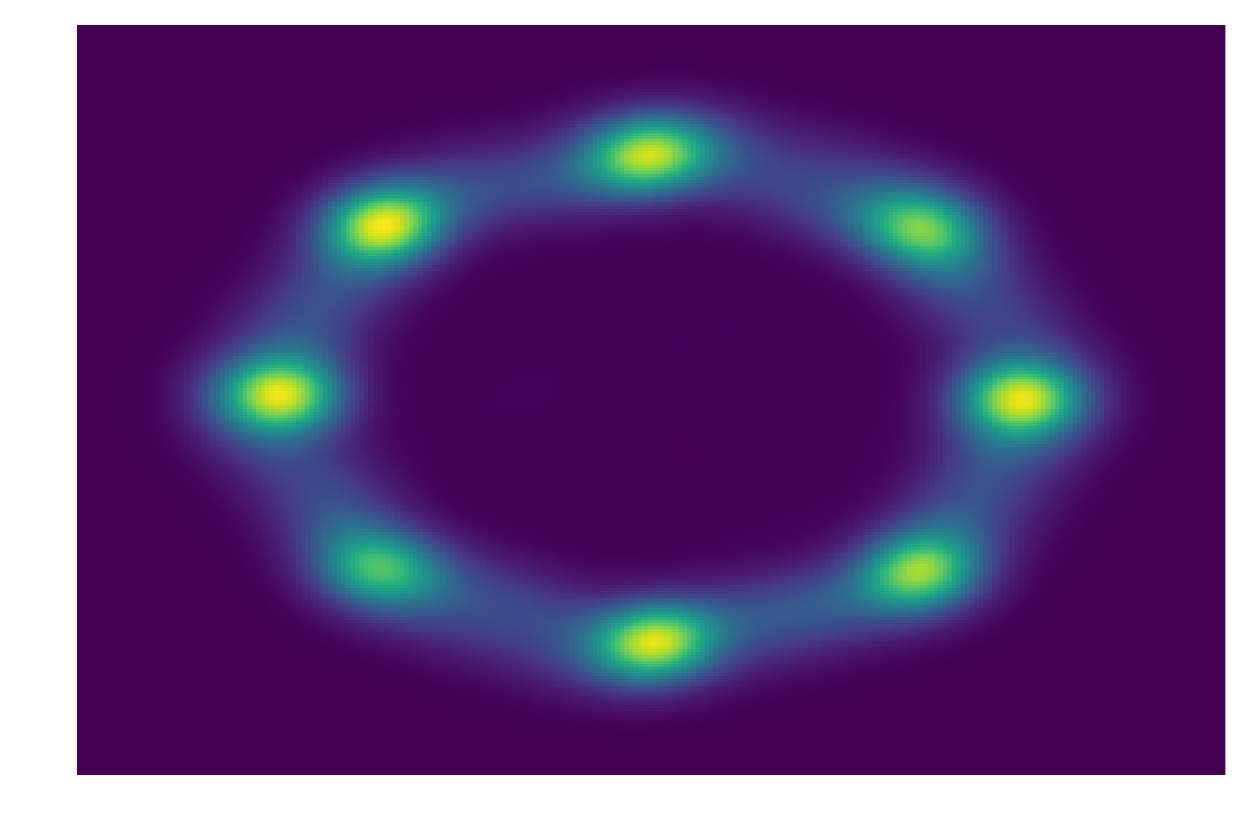}
      & \includegraphics[width=1.8cm,height=1.8cm]{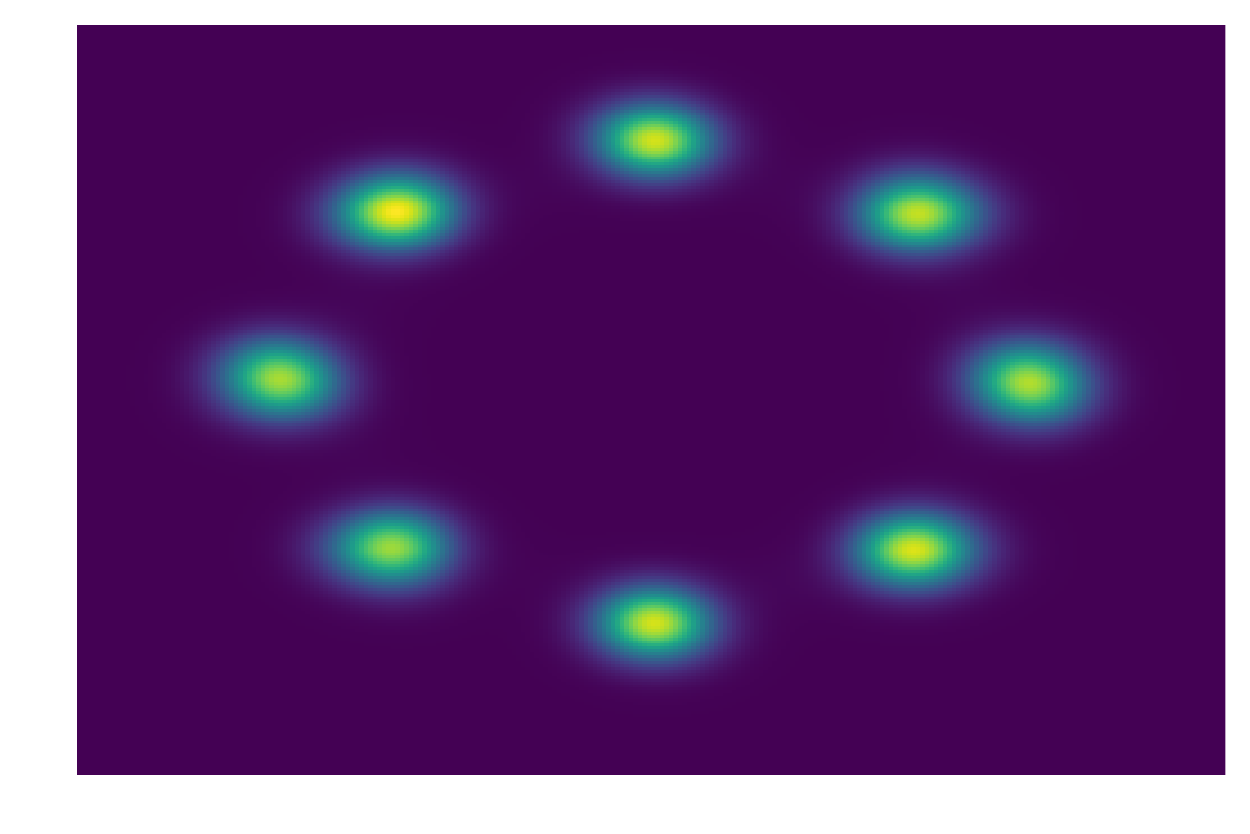}
    \\[2pt]
    \rotatebox{90}{\textbf{Two Rings}}
      & \includegraphics[width=1.8cm,height=1.8cm]{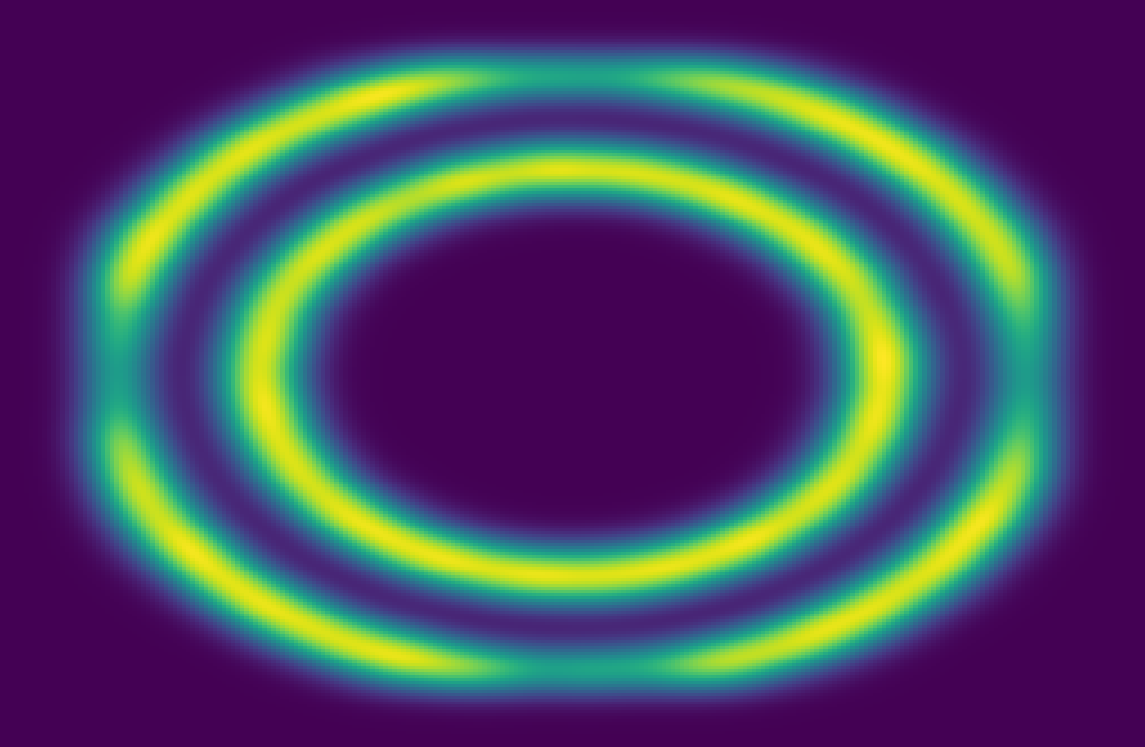}
      & \includegraphics[width=1.8cm,height=1.8cm]{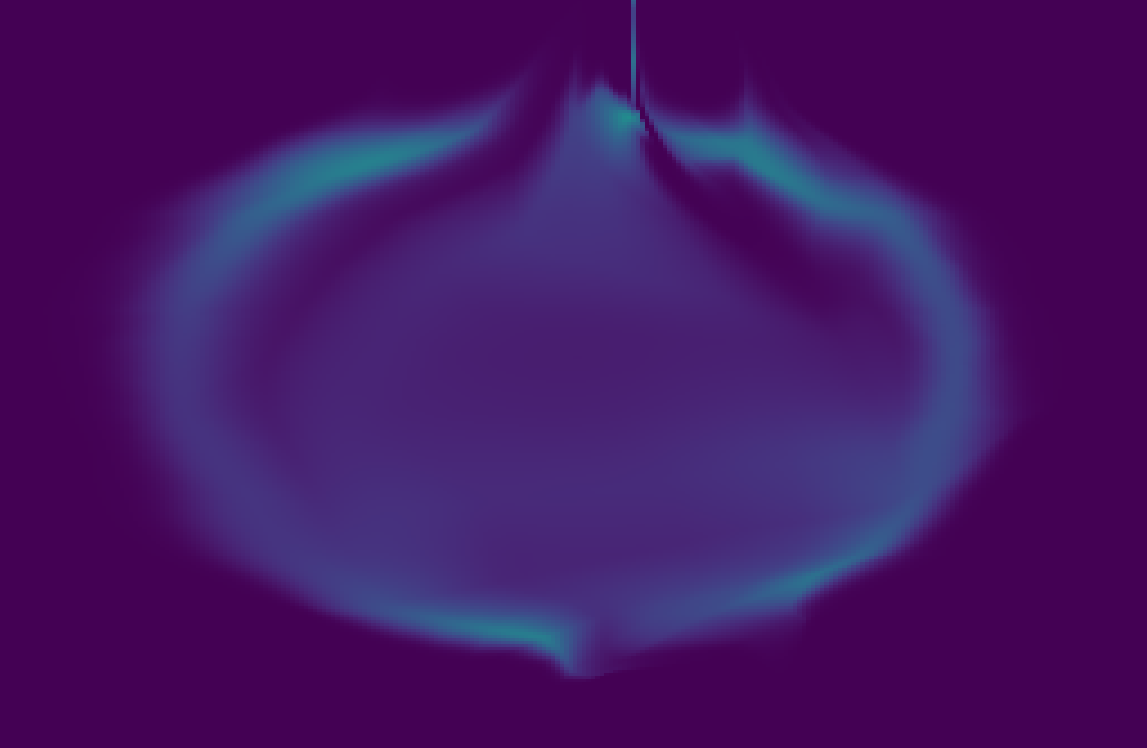}
      & \includegraphics[width=1.8cm,height=1.8cm]{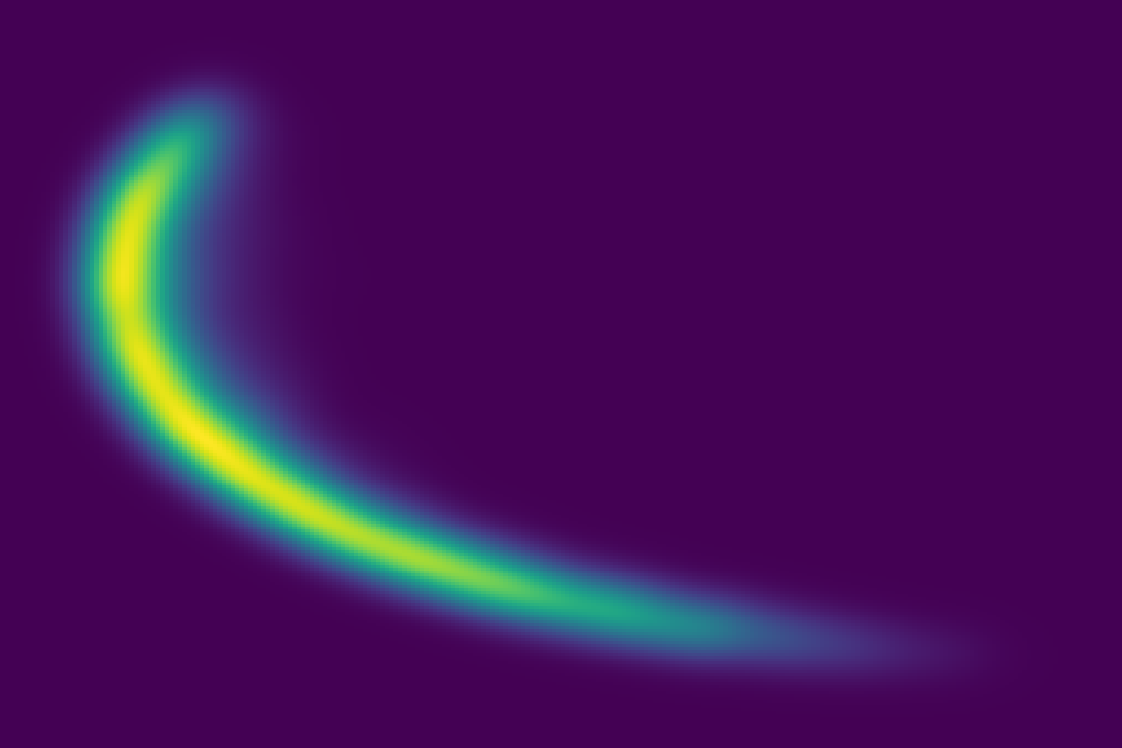}
      & \includegraphics[width=1.8cm,height=1.8cm]{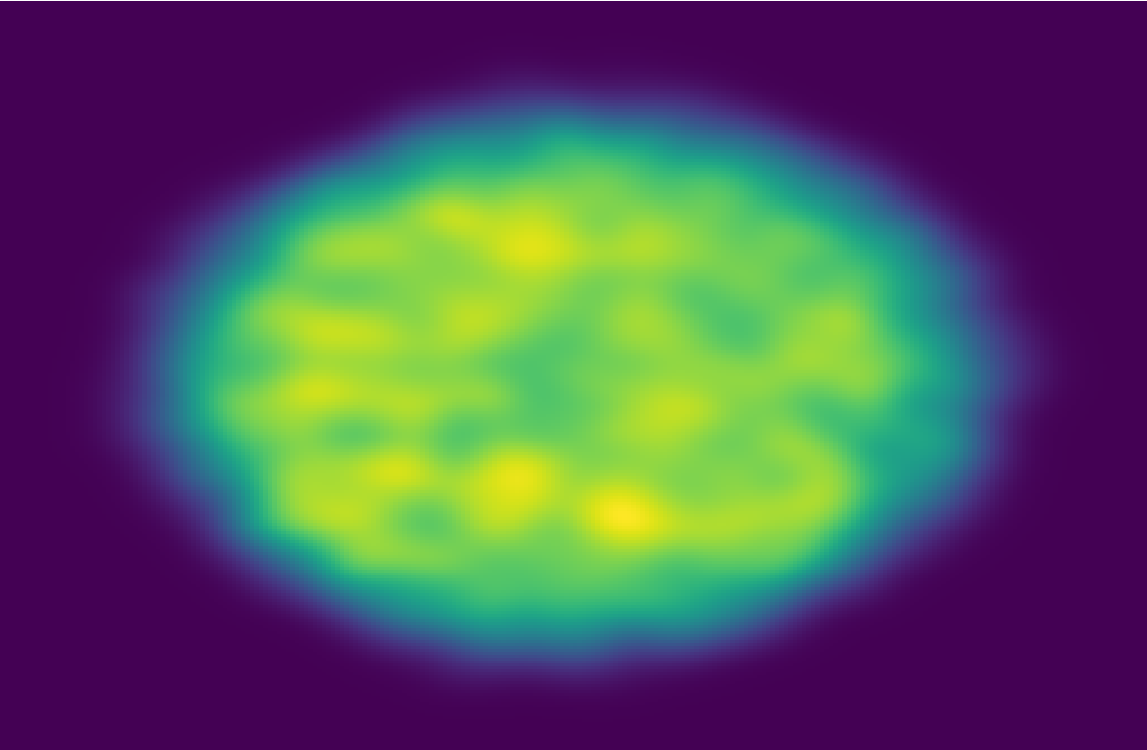}
      & \includegraphics[width=1.8cm,height=1.8cm]{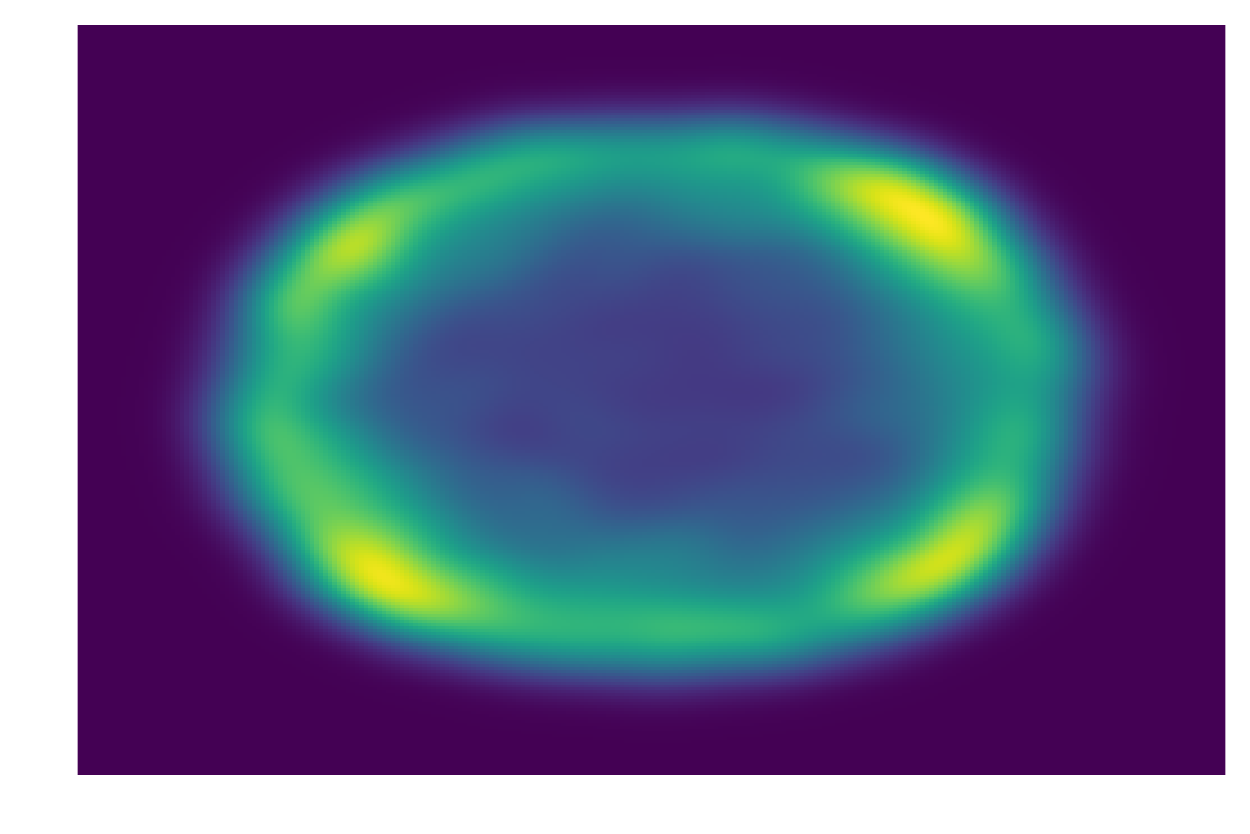}
      & \includegraphics[width=1.8cm,height=1.8cm]{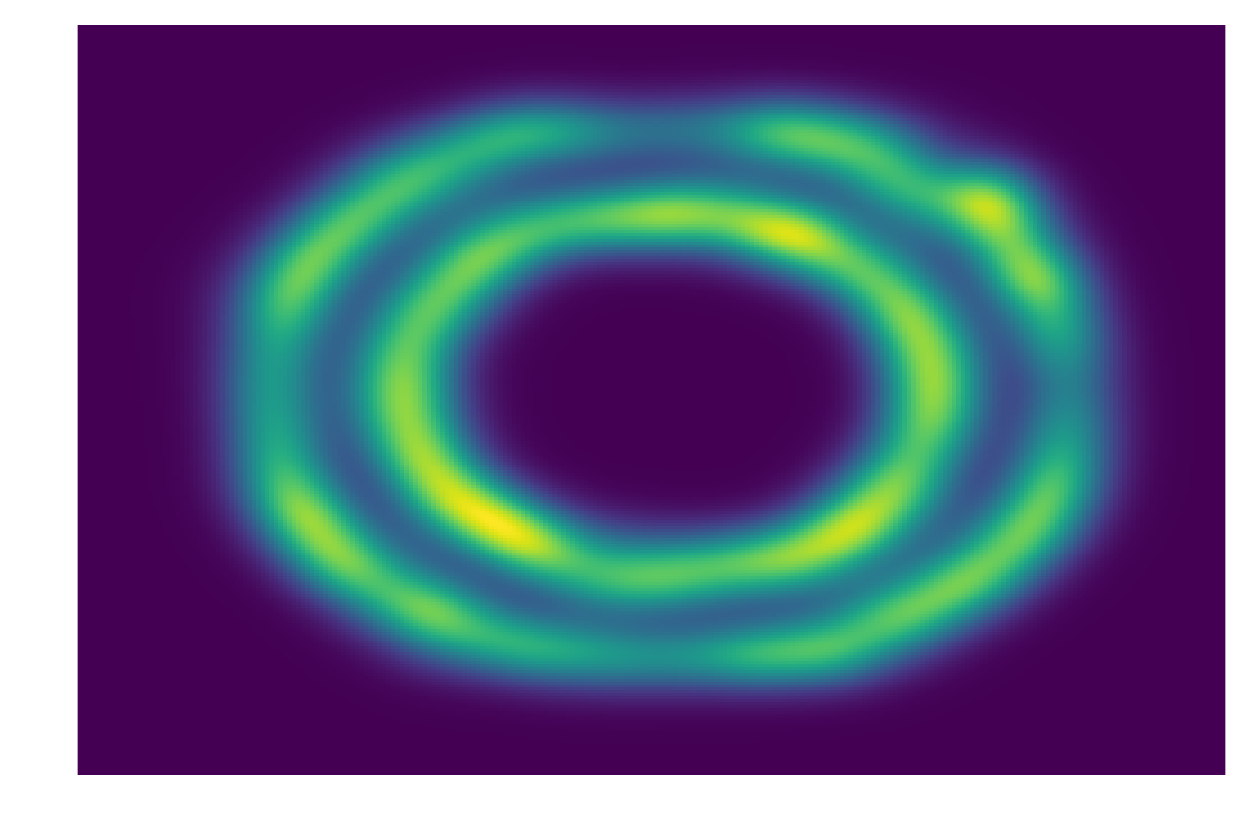}
    \\
    \bottomrule
  \end{tabular}
  \caption{Comparison of various generative methods on three 2D toy distributions. Each row (left) shows the target, and each column the outputs of Real NVP, WGAN, ISL-slicing, Dual-ISL, and dual-ISL(pretrained)+GAN, respectively.}
  \label{fig:2d-comparison}
\end{figure}

\begin{table}[htbp!]
  \centering
  \small
  \setlength{\tabcolsep}{8pt}
  \rowcolors{2}{gray!10}{white}
  \begin{tabular}{lcccccc}
    \toprule
    \rowcolor{gray!30}
    \textbf{Dataset}         & \textbf{Real NVP} & \textbf{GAN} & \textbf{WGAN} & \textbf{ISL} & \textbf{dual-ISL} & \textbf{dual-ISL+GAN} \\
    \midrule
    Dual Moon                & 1.77              & 1.23         & 1.02          & 0.43     &  0.35  & \textbf{0.21}    \\
    Circle of Gaussians      & 2.59              & 2.24         & 2.38          & 1.61   &  0.44    & \textbf{0.35}    \\
    Two Rings                & 2.69              & 1.46         & 2.74          & 0.56      & 0.43  & \textbf{0.29}    \\
    \bottomrule
  \end{tabular}
  \caption{KL‐divergence (lower is better) of different generative models on 2D toy benchmarks.}
  \label{table:2d_experiments}
\end{table}

Figure \ref{fig:2d_good_nice} overlays the true data samples (scatter points) with Gaussian‐KDE contours computed from $10^{5}$ points generated by sliced dual‐ISL.  These contours closely match the underlying support of both the Dual Moon and Two Rings datasets, accurately delineating disconnected modes and concentric rings.  This visualization confirms that sliced dual‐ISL not only recovers the global topology of complex 2D manifolds but also captures fine geometric details in the learned density.
\clearpage
\begin{figure}[htb!]
  \centering
  \begin{subfigure}[t]{0.48\textwidth}
    \centering
    \includegraphics[width=\linewidth, height=5cm]{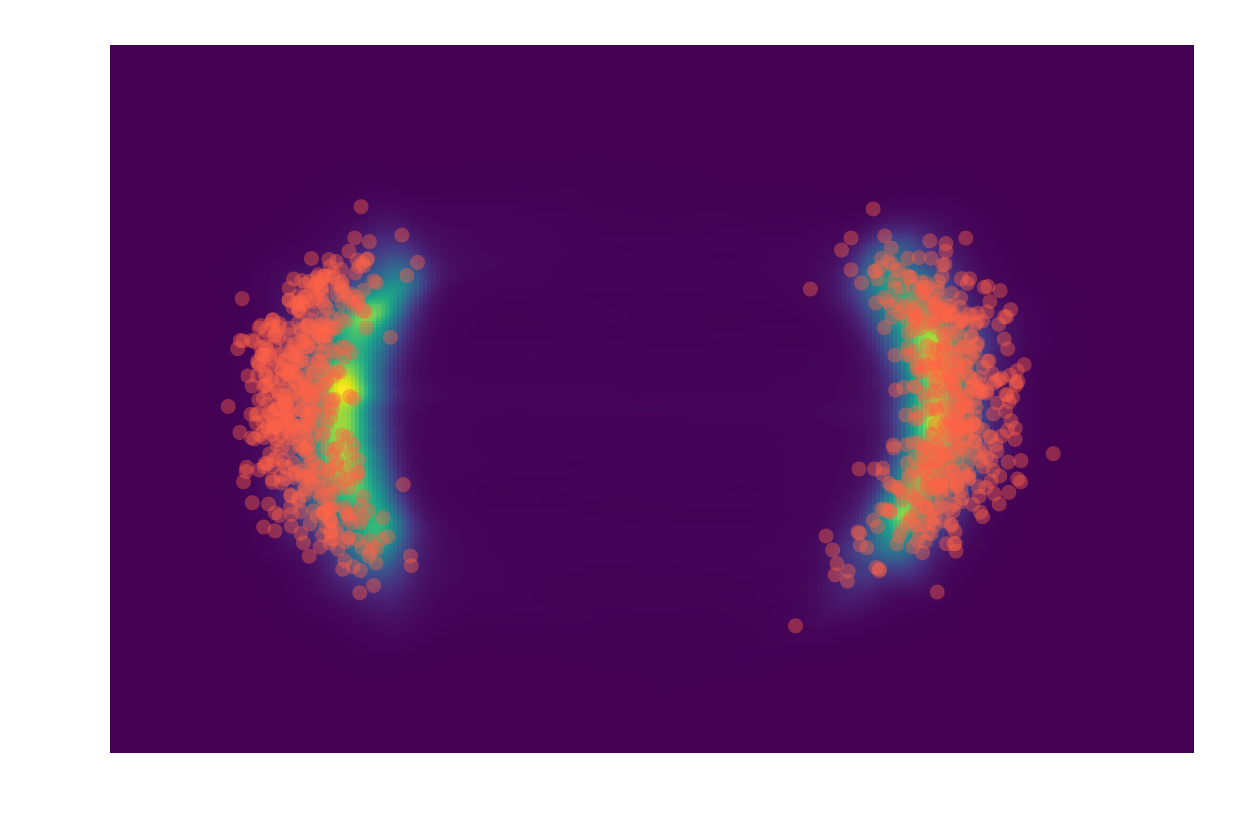}
    \caption{Dual Moon}
    \label{fig:2d_good_moon}
  \end{subfigure}
  \hfill
  \begin{subfigure}[t]{0.48\textwidth}
    \centering
    \includegraphics[width=\linewidth, height=5cm]{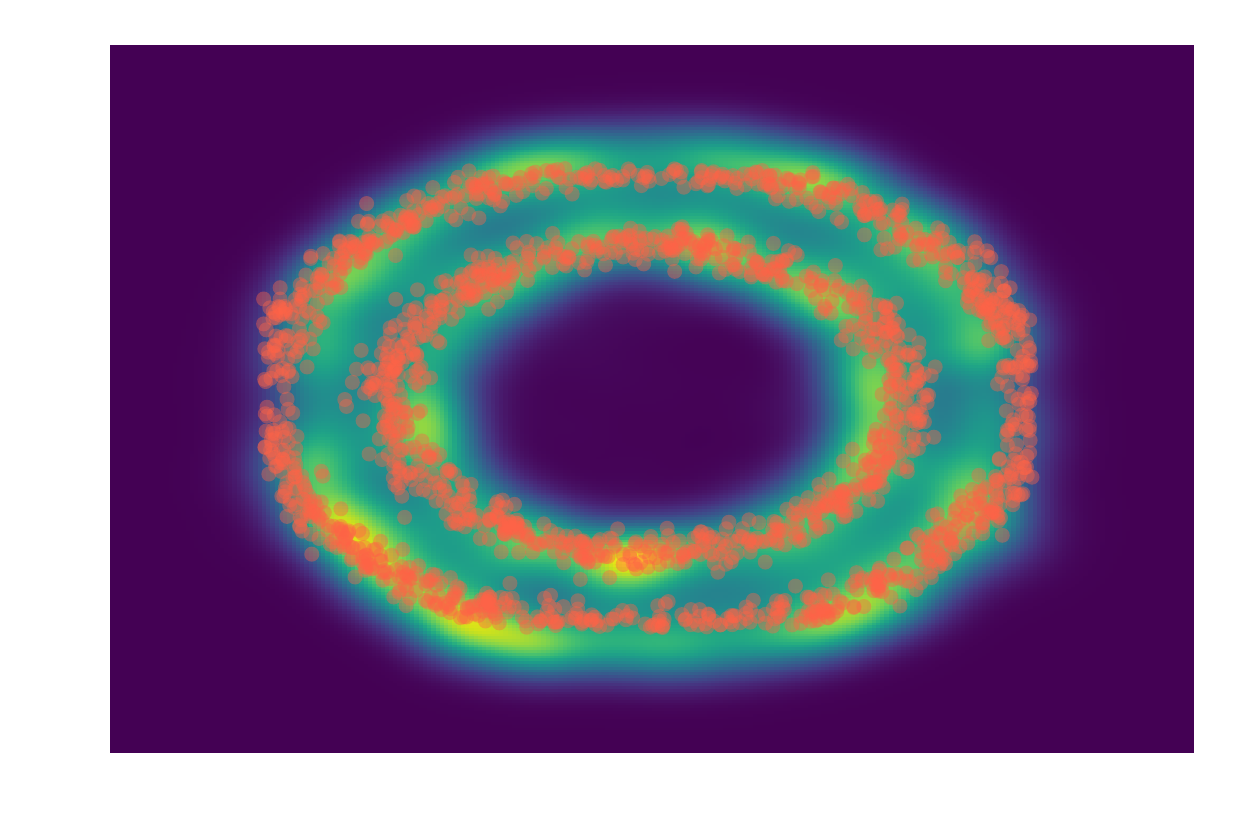}
    \caption{Two Rings}
    \label{fig:2d_good_rings}
  \end{subfigure}

  \vspace{0.5ex}
  \caption{Scatter plots of the true Dual Moon (left) and Two Rings (right) datasets, overlaid with density contours learned by Sliced dual-ISL and approximated via Gaussian KDE on $10^{5}$ generated samples. Models were trained for 1000 epochs with \(L=20\) random projections and a batch size of \(N=2000\).}
  \label{fig:2d_good_nice}
\end{figure}

\newpage \clearpage
\subsection{High dimensional experiments}

In this section, we evaluate the slicing dual-ISL method on high-dimensional image generation.  In particular, we incorporate the dual-ISL objective into the training of a DCGAN generator \cite{radford2015unsupervised} and benchmark the resulting architecture  on the MNIST and Fashion-MNIST datasets.  We report precision (a proxy for sample fidelity) and recall (a proxy for sample diversity) to assess the quality and diversity of the generated images \cite{sajjadi2018assessing}. Our models were trained for 40 epochs with a batch size of 128, except for the pretrained models, which were first pretrained with sliced dual-ISL for 20 epochs (using 20 random projections) and then trained as a DCGAN for 40 epochs.

In Table \ref{table:real word data 1}, we report our results alongside those of other implicit generative models. On MNIST, our simple ISL-based model achieves recall rates comparable to much more complex, multi-discriminator GANs \cite{durugkar2016generative, choi2022mcl} (5 discriminator each), despite using far fewer parameters. Furthermore, by pretraining the DCGAN generator with sliced dual-ISL and then fine-tuning under the standard adversarial loss, we attain state-of-the-art precision scores.

On Fashion MNIST, our model matches MCL-GAN in recall and—while we do not quite reach GMAN’s recall performance—our precision scores exceed theirs. This demonstrates that even with a simpler architecture, ISL can deliver competitive recall and precision results across diverse image-generation benchmarks.

\begin{table}[ht!]
  \centering
  \small
  \label{tab:quant-results}
  \setlength{\tabcolsep}{3pt}
  \begin{tabular}{
      l 
      l 
      S[table-format=2.2(2)] 
      S[table-format=2.2(2)] 
      S[table-format=2.2(2)] 
      S[table-format=2.2(2)] 
    }
    \toprule
    \multirow{2}{*}{\textbf{Dataset}}
      & \multirow{2}{*}{\textbf{Method}}
      & \multicolumn{2}{c}{\textbf{F-score}}
      & \multicolumn{2}{c}{\textbf{P\&R}} \\
    \cmidrule(lr){3-4} \cmidrule(lr){5-6}
      & 
      & {\(F_{1/8}\!\uparrow\)} 
      & {\(F_{8}\!\uparrow\)} 
      & {Precision\(\uparrow\)} 
      & {Recall\(\uparrow\)} \\
    \midrule
    \textbf{MNIST (28×28)}
      & dual-ISL (m=20)
      & 85.00 \pm 0.32   & 95.17 \pm 1.76
      & 84.85 \pm 1.20   & 95.35 \pm 1.39 \\
      & dual-ISL (m=50)
      & 85.69 \pm 0.29   & 95.81 \pm 1.24
      & 85.55 \pm 1.11   & 96.23 \pm 1.98 \\
      & DCGAN
      & 93.58 \pm 0.64   & 75.66 \pm 1.46
      & 93.85 \pm 1.45   & 75.43 \pm 2.56 \\
      & dual-ISL + DCGAN
      & 93.58 \pm 0.84   & 95.82 \pm 1.61
      & 94.03 \pm 1.82   & 96.68 \pm 2.42 \\
      & GMAN
      & 97.60 \pm 0.70   & 96.81 \pm 1.71
      & 97.60 \pm 1.82   & 96.80 \pm 2.42 \\
      & MCL-GAN
      & $\mathbf{97.71 \pm 0.19}$ & $\mathbf{98.49 \pm 1.57}$
      & $\mathbf{97.70 \pm 1.33}$ & $\mathbf{98.50 \pm 2.15}$ \\
    \addlinespace
    \textbf{FMNIST (28×28)}
      & dual-ISL (m=20)
      & 81.84 \pm 0.11   & 91.08 \pm 1.83
      & 81.48 \pm 1.43   & 91.49 \pm 2.15 \\
      & dual-ISL (m=50)
      & 83.90 \pm 0.09   & 91.18 \pm 1.57
      & 84.08 \pm 1.31   & 92.92 \pm 1.23 \\
      & DCGAN
      & 86.14 \pm 0.11   & 88.92 \pm 1.51
      & 86.60 \pm 1.58   & 88.97 \pm 1.33 \\
      & dual-ISL + DCGAN
      & 91.43 \pm 0.19   & 91.87 \pm 1.57
      & 91.88 \pm 1.35   & 92.42 \pm 1.47 \\
      & GMAN
      & 90.97 \pm 0.09   & $\mathbf{95.43 \pm 1.12}$
      & 90.90 \pm 1.33   & $\mathbf{95.50 \pm 2.25}$ \\
      & MCL-GAN
      & $\mathbf{97.62 \pm 0.09}$ & 92.97 \pm 1.28
      & $\mathbf{97.70 \pm 1.33}$ & 92.90 \pm 2.31 \\
    \bottomrule
  \end{tabular}\vspace{0.2cm}
    \caption{\small Quantitative comparison of generative models on MNIST and Fashion-MNIST (\(28\times28\)) using F$_{1/8}$, F$_{8}$ (weighted harmonic mean of Precision and Recall, with $\beta=1/8$ and $\beta=8$ respectively), Precision, and Recall (mean ± std) in \%. Results are shown for dual-ISL with \(m=20\) and \(m=50\) random projections, standard and Wasserstein DCGAN variants (with and without ISL pretraining), GMAN, and MCL-GAN. Boldface highlights the best score in each column per dataset. Higher is better.} \label{table:real word data 1}
\end{table}

Figure \ref{fig:digits frequency} shows the class‐frequency distributions obtained by our sliced dual‐ISL model (40 epochs, $m=50$ random projections) versus a standard DCGAN. Our model produces all ten digit classes in nearly uniform proportions—closely matching the true uniform distribution—while the DCGAN exhibits pronounced class imbalance. To compute these frequencies, we generated 10 000 samples from each model and classified them with a pretrained digit recognizer. In a Kolmogorov–Smirnov test for uniformity (on 10 000 samples), the sliced dual‐ISL model achieved $p = 0.070 $, compared to $p = 0.642$ for DCGAN, indicating a significantly better match to the ideal uniform distribution.

Finally, Figure \ref{fig:ablation studies on p and r metrics vs m} presents an ablation study of precision and recall in dual-ISL on MNIST, comparing the effects of using 20 versus 100 random projections. Even with only 20 projections, the model already achieves strong performance, and increasing the number of projections yields only marginal gains in these quality metrics.

\begin{figure}[htbp]
  \centering
\begin{tikzpicture}
  \begin{axis}[
      name=plotA,
      ybar,
      bar width=6pt,
      width=6.5cm, height=4cm,
      enlarge x limits=0.1,
      ymin=0, ymax=0.25,
      ylabel={Frecuency},
      xtick={0,...,9},
      xticklabel style={font=\footnotesize},
      axis lines=box,
      xtick pos=bottom,
      ytick pos=left,
    ]
    \draw[black, dashed] 
      (axis cs:-0.5,0.1) -- (axis cs:9.5,0.1);
    \addplot[fill=blue!60] coordinates {
      (0,0.086) (1,0.0996) (2,0.1223) (3,0.1694) (4,0.0904)
      (5,0.0979) (6,0.0762) (7,0.1034) (8,0.0806) (9,0.0742)
    };
  \end{axis}

  \begin{axis}[
      at={(plotA.south east)},    
      anchor=south west,          
      xshift=1.5cm,               
      ybar,
      bar width=6pt,
      width=6.5cm, height=4cm,
      enlarge x limits=0.1,
      ymin=0, ymax=0.25,
      xtick={0,...,9},
      xticklabel style={font=\footnotesize},
      axis lines=box,
      xtick pos=bottom,
      ytick pos=left,
    ]
    \draw[black, dashed] 
      (axis cs:-0.5,0.1) -- (axis cs:9.5,0.1);
    \addplot[fill=blue!60] coordinates {
      (0,0.1372) (1,0.0864) (2,0.2413) (3,0.036) (4,0.1659)
      (5,0.0155) (6,0.1171) (7,0.0087) (8,0.1895) (9,0.0078)
    };
  \end{axis}
\end{tikzpicture}
  \caption{Class‐frequency distributions of MNIST samples generated by two different models. Left: dual-ISL; Right: standard DCGAN. Each bar represents the proportion of generated images assigned to each digit class (0–9), illustrating that dual-ISL produces a more uniform coverage across all classes, whereas the DCGAN exhibits notable biases toward certain digits. The dashed line indicates the ideal uniform distribution across all classes.}
  \label{fig:digits frequency}
\end{figure}
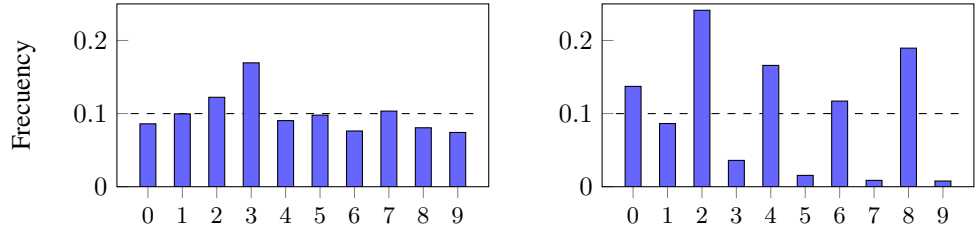

\begin{figure}[htb!]
  \centering
  \begin{subfigure}[t]{0.48\textwidth}
    \centering
    \includegraphics[width=\linewidth, height=5cm]{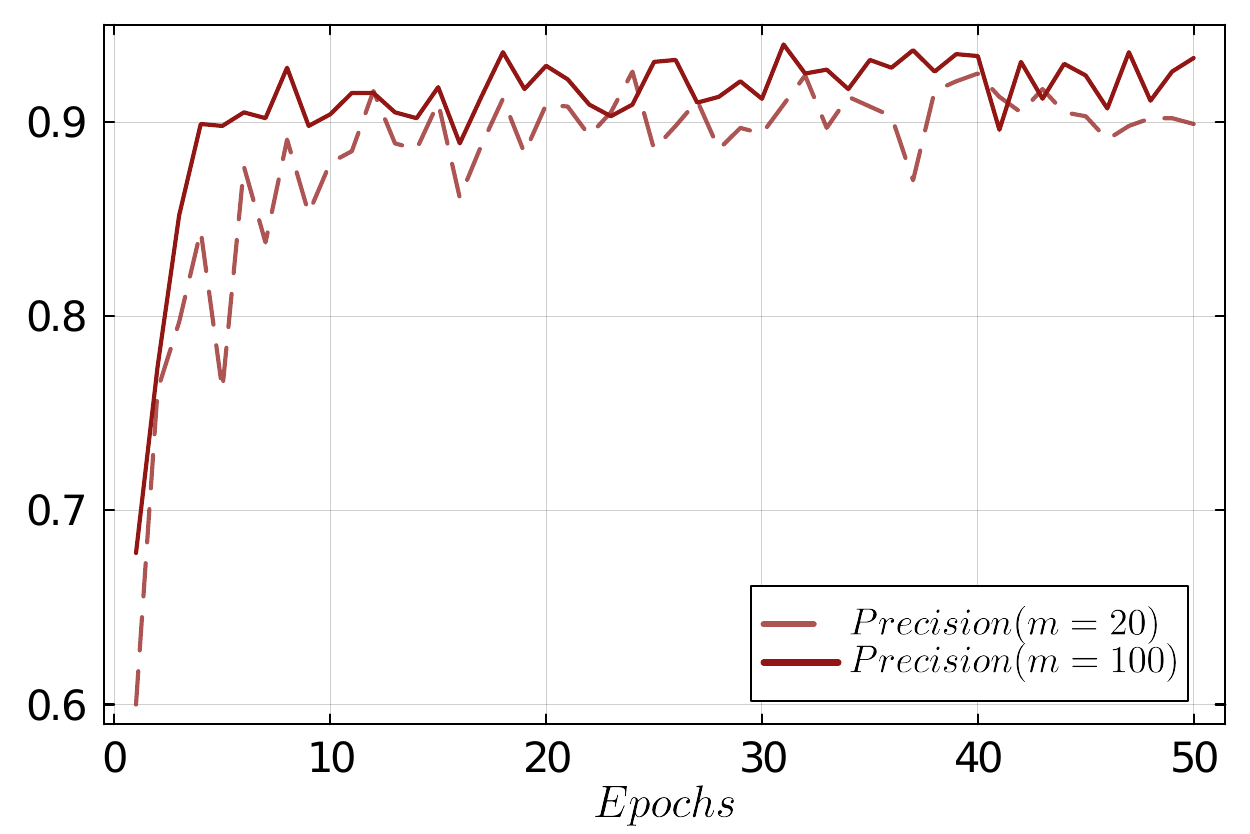}
    \label{fig:precision_m20_m100_MNIST}
  \end{subfigure}
  \hfill
  \begin{subfigure}[t]{0.48\textwidth}
    \centering
    \includegraphics[width=\linewidth, height=5cm]{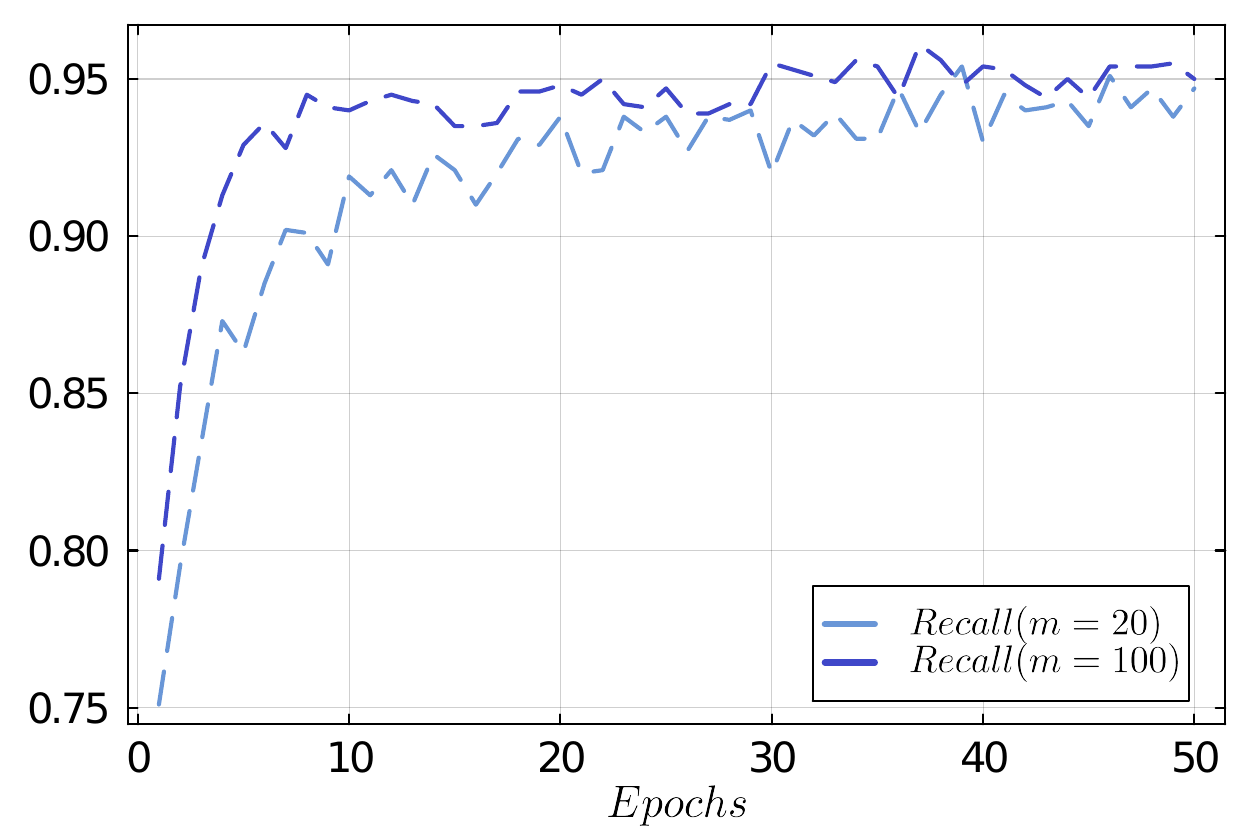}
    \label{fig:recall_m20_m100_MNIST}
  \end{subfigure}

  \vspace{0.5ex}
  \caption{Precision (left) and recall (right) on the MNIST test set after 50 training epochs, using $m=20$ and $m=100$ random projections.}
  \label{fig:ablation studies on p and r metrics vs m}
\end{figure}

\begin{figure}[htb!]
  \centering
  \begin{subfigure}[t]{0.48\textwidth}
    \centering
    \includegraphics[width=\linewidth, height=5cm]{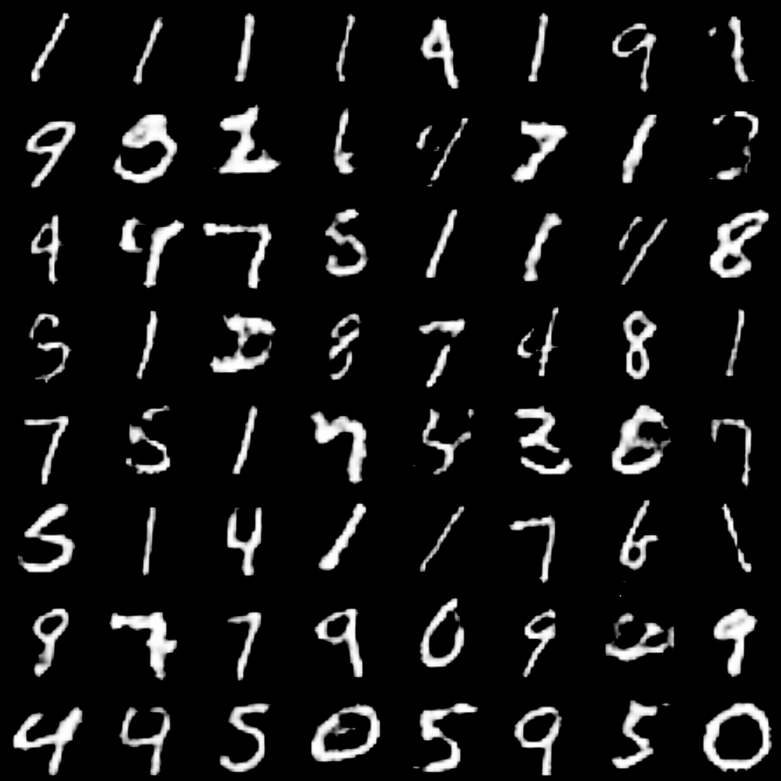}
    \caption{DCGAN}
    \label{fig:GAN}
  \end{subfigure}
  \hfill
  \begin{subfigure}[t]{0.48\textwidth}
    \centering
    \includegraphics[width=\linewidth, height=5cm]{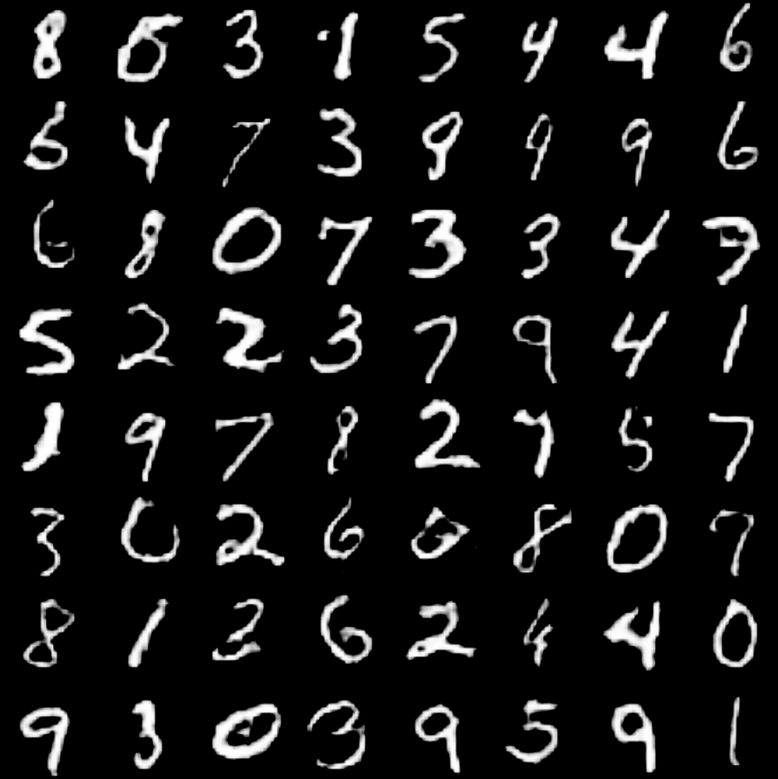}
    \caption{DCGAN pretrained with ISL}
    \label{fig:ISL_GAN}
  \end{subfigure}

  \vspace{0.5ex}
  \caption{Figure \ref{fig:GAN} shows digits generated by a standard DCGAN, while Figure \ref{fig:ISL_GAN} shows samples from a DCGAN pretrained with ISL. The ISL-pretrained model exhibits significantly greater diversity across all classes, whereas the vanilla DCGAN produces an overabundance of ‘1’s.
}
  \label{fig:MNIST_good}
\end{figure}

\begin{figure}[htb!]
  \centering
  \begin{subfigure}[t]{0.48\textwidth}
    \centering
    \includegraphics[width=\linewidth, height=5cm]{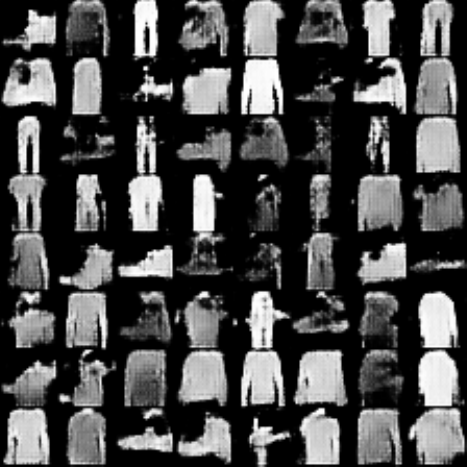}
    \caption{DCGAN}
    \label{fig:FMNIST_ISL}
  \end{subfigure}
  \hfill
  \begin{subfigure}[t]{0.48\textwidth}
    \centering
    \includegraphics[width=\linewidth, height=5cm]{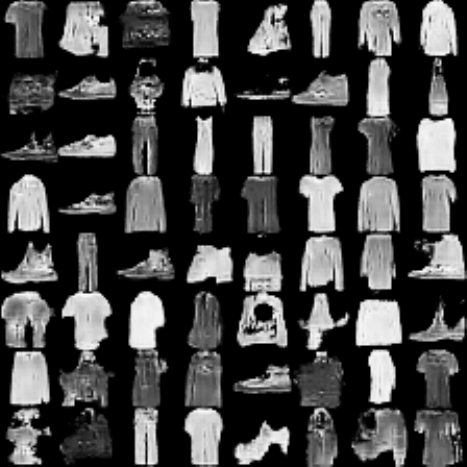}
    \caption{DCGAN pretrained with ISL}
    \label{fig:FMNIST_ISLDCGAN}
  \end{subfigure}

  \vspace{0.5ex}
  \caption{Comparison of Fashion-MNIST samples: \ref{fig:FMNIST_ISL} generated by DCGAN trained 40 epochs, and \ref{fig:FMNIST_ISLDCGAN} generated by a DCGAN pretrained with dual-ISL. The ISL-pretrained model demonstrates greater class diversity and improved precision.
}
  \label{fig:FMNIST_good}
\end{figure}

\newpage\clearpage

\newpage \clearpage

\section{Pseudocodes} \label{pseudocodes}
Below we summarize two variants of the dual‐ISL training procedure.  Algorithm~\ref{alg:dual-isl} presents the basic likelihood‐free, rank‐based update in the one‐dimensional case.  Algorithm~\ref{alg:dual-isl-multidim} extends this to multi‐dimensional data by drawing random one‐dimensional projections and averaging the resulting rank losses.

\begin{algorithm}[ht!]
\setstretch{1.2}
\caption{dual-ISL Training}\label{alg:dual-isl}
\begin{algorithmic}[1]
  \REQUIRE Generator network $f_{\theta}$, real samples $\{y_i\}_{i=1}^N$, batch size $M$, rank draws $K$, epochs $T$, learning rate $\eta$
  \ENSURE  Trained parameters $\theta$
  \FOR{epoch $=1$ to $T$}
    \FOR{each minibatch $B\subset \{y_i\}_{i=1}^{N}$ of size $M$}
      \STATE Sample noise $z \sim p_z$
      \STATE Generate fictious sample $\tilde y = f_\theta(z)$
      \STATE Initialize histogram $\mathbf{q} \leftarrow \mathbf{0}$
      \FOR{$t = 1,\dots,\lfloor M/K\rfloor$}
        \STATE $\{y_i\}_{i=1}^{K} \gets \text{RandomSubset}(B,\,K)$  \COMMENT{draw $K$ real samples from the minibatch}
        \STATE $a_K \leftarrow \sum_{i=1}^K \mathrm{SoftIndicator}[\tilde y \le y_{i}]$ \COMMENT{differentiable count of the $A_{K}$ statistic}
        \STATE $\mathbf{q}\leftarrow \mathbf{q} + \mathrm{SoftHotEncoding}(a_K,\ \text{length}=K+1)$ \COMMENT{accumulate a differentiable one-hot into the histogram}
      \ENDFOR
      \STATE $\mathbf{q} \leftarrow \mathrm{normalize}(\mathbf{q})$
      \STATE Compute loss: $\mathit{loss} \leftarrow \left|\left|\frac{1}{K+1}\mathbf{1}_{K+1}-\mathbf{q}\right|\right|_{\ell^1}$
      \STATE Update: $\theta \leftarrow \theta \;-\;\eta\,\nabla_{\theta}\,\mathit{loss}$
    \ENDFOR
  \ENDFOR 
  \RETURN $\theta$
\end{algorithmic}
\end{algorithm}

\begin{algorithm}[ht!]
\setstretch{1.2}
\caption{Dual‐ISL with Random Projections}\label{alg:dual-isl-multidim}
\begin{algorithmic}[1]
  \REQUIRE Generator $f_{\theta}$, real data $\{y_i\}_{i=1}^N\subset\mathbb{R}^d$, batch size $M$, rank draws $K$, projection draws $L$, epochs $T$, learning rate $\eta$
  \ENSURE  Learned parameters $\theta$
  \FOR{epoch = 1 \TO $T$}
    \FOR{each minibatch $B\subset\{y_i\}$ of size $M$}
      \STATE Sample noise $z \sim p_z$
      \STATE Generate fictious sample $\tilde y = f_\theta(z)$
      \STATE Initialize histogram $\mathbf{q} \leftarrow \mathbf{0}$
      \FOR{$\ell = 1$ \TO $L$}
        \STATE Sample random unit vector $v_\ell\sim \mathrm{Uniform}(S^{d-1})$
        \FOR{$t = 1,\dots,\lfloor M/K\rfloor$}
            \STATE $\{y_i\}_{i=1}^{K} \gets \text{RandomSubset}(B,\,K)$  \COMMENT{draw $K$ real samples from the minibatch}
            \STATE Compute projections $\tilde u = v_\ell^\top \tilde y$ and $u_i = v_\ell^\top y_i$ for all $i$
            \STATE $a_K \leftarrow \sum_{i=1}^K \mathrm{SoftIndicator}[\tilde u \le u_{i}]$ \COMMENT{differentiable count of the $A_{K}$ statistic}
            \STATE $\mathbf{q}\leftarrow \mathbf{q} + \mathrm{SoftHotEncoding}(a_K,\ \text{length}=K+1)$ \COMMENT{accumulate a differentiable one-hot into the histogram}
        \ENDFOR
        \STATE $\mathbf{q} \leftarrow \mathrm{normalize}(\mathbf{q})$
        \STATE $\mathrm{loss} \leftarrow \|\,\mathbf{q} - \tfrac{1}{K+1}\mathbf1_{K+1}\|_1$
      \ENDFOR
      \STATE Compute loss: $\mathrm{projection\_loss = \mathrm{mean}(loss)}$
      \STATE Update: $\theta \leftarrow \theta - \eta\,\nabla_\theta\,\mathrm{projection\_loss}$
    \ENDFOR
  \ENDFOR
  \RETURN $\theta$
\end{algorithmic}
\end{algorithm}

\newpage

\section{Experimental Setup}\label{Experimental Setup}
All experiments were performed on a MacBook Pro running macOS 13.2.1, equipped with an Apple M1 Pro CPU and 16 GB of RAM. When GPU acceleration was required, we used a single NVIDIA TITAN Xp with 12 GB of VRAM. Detailed hyperparameter settings for each experiment are provided in the corresponding sections. An anonymous repository containing all code and data is available at \url{https://anonymous.4open.science/r/dual-isl-6633}. The code will also be included in the supplementary materials in a folder.

\section{Limitations}

While our invariant statistical loss (ISL) framework eliminates the need for adversarial critics and guarantees strong convergence, it faces a critical trade-off when extended to high-dimensional data via ISL-slicing. Specifically, it requires $m$ random projections: a large $m$ enhances fidelity but incurs steep computational costs, whereas a small $m$ may overlook key anisotropic features. To address this, future research should design adaptive strategies for choosing or weighting projections—potentially drawing on recent advances in slicing-Wasserstein theory—to maximize information gain per projection. Moreover, exploring alternative projection methods (such as data-dependent or learned mappings) and establishing rigorous convergence bounds that link $m$  to both convergence rate and approximation error will be essential for fully automating and optimizing ISL’s performance.

By viewing ISL as the “cost” of rearranging generated samples to match real data points, we uncover its direct relationship with optimal transport. In essence, the permutation that sorts one sample set against another defines an explicit coupling—much like the Monge map—between the model and data distributions. Future work should formalize this correspondence, harnessing optimal transport tools to both analyze ISL’s theoretical properties and develop faster, more principled algorithms.

\section{Potential Societal Impact}

Implicit generative models unlock powerful data‐synthesis capabilities but also pose dual-use risks. Although our ISL framework could produce highly realistic outputs—such as photorealistic faces or authentic-sounding voices—it could likewise be misused to generate deep fakes. At the same time, we have shown that ISL excels at modeling heavy-tailed distributions, which helps ensure rare or minority subpopulations are neither over- nor under-represented. Furthermore, the closed-form density estimation inherent in ISL offers a transparent window into what is otherwise a black-box process, improving explainability and enabling practitioners to audit and adjust model behavior for fairness. By combining high-fidelity synthesis with built-in safeguards and interpretability, ISL paves the way for more responsible, equitable deployment of generative technologies.

\end{document}